\documentclass[a4paper, twocolumn]{article}
\usepackage{graphicx} %
\usepackage[utf8]{inputenc}
\usepackage{amsmath, amssymb, amsthm, amsfonts}
\usepackage{bm}
\usepackage{thmtools}

\newcommand{\E}[1]{\mathbb{E}\left[ #1 \right]}
\DeclareMathOperator{\Ex}{\mathbb{E}}
\DeclareMathOperator{\Bin}{\text{Bin}}
\DeclareMathOperator{\Var}{\text{Var}}
\newcommand{\prob}[1]{\mathbb{P}\left(#1\right)}
\newcommand{\eps}{\epsilon}

\usepackage{mathtools}
\newcommand{\defeq}{\vcentcolon=}

\newcommand{\pihat}{\hat{\pi}}
\newcommand{\vpi}{\bm{\pi}}
\newcommand{\pitwiddle}{\bm{\widetilde{\pi}}}
\newcommand{\Fhat}{\hat{F}}
\newcommand{\F}{\bm{F}}

\newcommand{\p}{\bm{p}}

\renewcommand{\textit}{\emph}
\newcommand{\norm}[1]{\left\lVert#1\right\rVert}
\let\originalmiddle=\middle
\def\middle#1{\mathrel{}\originalmiddle#1\mathrel{}}
\newcommand{\Do}{\text{do}}
\DeclarePairedDelimiter\floor{\lfloor}{\rfloor}

\newcommand\simiid{\mathrel{\overset{\makebox[0pt]{\mbox{\normalfont\tiny iid}}}{\sim}}}

\theoremstyle{plain}
\newtheorem{thm}{Theorem}
\newtheorem{lemma}{Lemma}

\newtheorem{cor}{Corollary}
\newtheorem{defn}{Definition}
\theoremstyle{definition}
\newtheorem{notn}{Notation}
\theoremstyle{remark}
\newtheorem{rk}{Remark}
\newtheorem{eg}{Example}
\newtheorem{asm}{Assumption}

\usepackage{microtype}
\usepackage[compact,small]{titlesec}
\usepackage[right=1.5cm, left=1.5cm, bottom = 3cm, top=3cm]{geometry}
\usepackage{xfrac}

\usepackage{xcolor}
\usepackage{hyperref}

\usepackage[capitalise]{cleveref}
\Crefname{thm}{Theorem}{Theorems}
\Crefname{prop}{Proposition}{Propositions}
\Crefname{rk}{Remark}{Remarks}
\Crefname{eg}{Example}{Examples}
\Crefname{asm}{Assumption}{Assumptions}
\Crefname{defn}{Definition}{Definitions}
\Crefname{cor}{Corollary}{Corollaries}

\hypersetup{
    colorlinks,
    linkcolor={red!60!black},
    citecolor={blue!60!black},
    urlcolor={blue!60!black}
}

\usepackage{ifthen}
\newboolean{commentsactivated}
\setboolean{commentsactivated}{true}
\newcommand{\co}[1]{\ifthenelse{\boolean{commentsactivated}}{{\color{red} {\em CO: #1 }}}{}}
\newcommand{\ec}[1]{\ifthenelse{\boolean{commentsactivated}}{{\color{teal} {\em EC: #1 }}}{}}
\newcommand{\vc}[1]{\ifthenelse{\boolean{commentsactivated}}{{\color{olive} {\em VC: #1 }}}{}}
\usepackage{multirow,array}

\newboolean{LOFTversion}
\setboolean{LOFTversion}{true}
\newcommand{\excludeforloft}[1]{\ifthenelse{\boolean{LOFTversion}}{}{#1}}
\usepackage[inline]{enumitem}

\newboolean{journalversion}
\setboolean{journalversion}{false}
\newcommand{\excludeforjournal}[1]{\ifthenelse{\boolean{journalversion}}{}{#1}}

\usepackage{placeins}

\usepackage{thm-restate}

\usepackage{pgfplots}
\pgfplotsset{compat = 1.18}

\usepackage{tikz}
\usetikzlibrary{automata,arrows,positioning,calc}
\usetikzlibrary{shapes.multipart}
\usepackage{xcolor}

\usepackage[round]{natbib} %
\title{Can CDT rationalise the \textit{ex ante} optimal policy via modified anthropics?}
\author{Emery Cooper, Caspar Oesterheld, Vincent Conitzer}
\date{\today}

\begin{document}

\maketitle

\begin{abstract}
In Newcomb's problem, causal decision theory (CDT) recommends two-boxing and thus comes apart from evidential decision theory (EDT) and \textit{ex ante} policy optimisation (which prescribe one-boxing). However, in Newcomb's problem, you should perhaps believe that with some probability you are in a simulation run by the predictor to determine whether to put a million dollars into the opaque box. If so, then causal decision theory might recommend one-boxing in order to cause the predictor to fill the opaque box. In this paper, we study generalisations of this approach. That is, we consider general Newcomblike problems and try to form reasonable self-locating beliefs under which CDT's recommendations align with an EDT-like notion of \textit{ex ante} policy optimisation. We consider approaches in which we model the world as running simulations of the agent, and an approach not based on such models (which we call `Generalised Generalised Thirding', or GGT). For each approach, we characterise the resulting CDT policies, and prove that under certain conditions, these include the \textit{ex ante} optimal policies. %
\end{abstract}


\section{Introduction}

Consider Newcomb's problem \citep{Nozick}:
\begin{eg}\label{eg:Newcombs}
An agent sees two boxes, one of them transparent, containing a sure \$1000, and the other opaque, containing either \$0 or \$1,000,000. The agent must choose whether to take just the opaque box (to `one-box') or to take both boxes (to `two-box'). Prior to this, a powerful predictor, `Newcomb's Demon', predicted the agent's decision. If Newcomb's Demon predicted that the agent would one-box, the Demon put \$1,000,000 in the opaque box. Otherwise, it left the opaque box empty.
\end{eg}

Two\footnote{Other theories of decision making have also been proposed \cite[e.g.][]{Price2012,Wedgwood2013,Dohrn2015,Levinstein2020}, but are less discussed in the literature.%
}
competing normative theories of decision making have been proposed for dealing with such problems: Evidential Decision Theory (EDT) \cite[e.g.,][]{Horgan1981,Price1986,Ahmed2014} and Causal Decision Theory (CDT) \cite[e.g.,][]{Gibbard1981,Lewis1981,Skyrms1982,Joyce1999,Weirich2016}.
EDT recommends one-boxing in Newcomb's problem, reasoning that conditional on one-boxing the agent expects to earn \$1,000,000, while conditional on two-boxing, it expects to earn only \$1,000. CDT instead recommends two-boxing, on the basis that the agent cannot causally affect the contents of the boxes, and whatever their contents, two-boxing renders the agent \$1,000 better off. Problems in which CDT and EDT come apart are called \emph{Newcomblike}.

We can also evaluate actions from an \textit{ex ante} perspective: what would the agent have wanted to commit to doing at some (hypothetical) earlier point in time, before finding itself in the decision problem \cite[][]{Gauthier1989,Meacham2010}?
The \textit{ex ante} optimal policy in Newcomb's problem is to one-box.\footnote{This comes with some caveats. In general, the CDT \textit{ex ante} optimal policy depends on the \textit{ex ante} perspective taken, and the manner in which the Demon makes its prediction. For instance, if the Demon's prediction is based on observing the agent's DNA at birth, then for any \textit{ex ante} perspective after the agent is born, CDT still recommends two-boxing. Cf.\ \Cref{sec:ex_ante_EU_rk}. Here, we consider an EDT notion of \emph{ex ante} optimality.} Thus, CDT's recommendation in Newcomb's Problem is \textit{ex ante} suboptimal.

Some have argued that the CDT agent ought to have anthropic uncertainty over whether it is the real agent or some kind of simulation of itself that Newcomb's Demon runs in order to predict it (\citealt[p.~12f.]{Neal}; \citealt{Aaronson2005}; \citealt{Taylor2016}; \citealt[p.~37f.]{dese}; \citealt{EaswaranUnpublished}%
).\footnote{The simulation interpretation is one way to rationalise one-boxing on causalist, or perhaps pseudo-causalist, grounds, but other arguments toward the same end have been given \citep{EaswaranUnpublished}. That is, some authors have argued that in some relevant, not merely evidential sense, the decision to one-box influences the content of the opaque box. As a rough example of such an argument, you might imagine that an agent chooses not the physical act, but the logical/Platonic fact of what choice the agent makes, and this choice affects both whether the agent one- or two-boxes, and whether box B contains a million dollars. (Often the goal is to develop a theory that one-boxes, but avoids other (alleged) problems with evidential decision theory, such as misbehaviour in medical Newcomb-like problems.) See, for instance, \citet[Ch.\ 5, 6]{Drescher2006}, \citet{Yudkowsky2010}, \citet{Spohn2012} (cf.\ \citet{Poellinger2013}), \citet{Yudkowsky2018}, \citet{Levinstein2020}. These approaches are similar to the simulation approach in that they use a CDT-like framework and put the agent causally upstream of the prediction. However, the mechanism and the resulting dynamics are quite different. In the simulation approach, the agent is either in a position to influence its physical choice, or to influence the prediction, but not both. In contrast, the approaches discussed in this paragraph have the agent (pseudo-)causally influence both at once.

Outside of the decision theory of Newcomb-like problems, there have been some discussions of the relation between predicting someone and, in some sense, simulating them, see \citet{sep-folkpsych-simulation} for an overview. Since the objective of this literature is so different, it is not clear how it relates to the thesis underlying that present paper: that prediction induces self-location uncertainty.
} If the agent were to assign equal credence to each of these two possibilities, CDT would then recommend one-boxing: If the agent is in Newcomb's Demon's simulation, its action only causally influences whether the opaque box is filled, making one-boxing  \$1,000,000 more valuable than two-boxing in this case. Otherwise, it is the real agent, and one-boxing earns \$1,000 less than two-boxing. Thus, with credence $0.5$ in each, the expected utility of one-boxing is $0.5\times \$1,000,000 - 0.5\times\$1,000 = \$499,500$ more than two-boxing.

Even if Newcomb's Demon does not literally simulate the agent, perhaps the agent should anthropically identify with the prediction anyway. That is, the agent should still think it might be an entity in Newcomb's Demon's mind. One justification for this might be that whether an agent identifies with a particular entity ought to depend only on the behavioural or functional properties of that entity. %
In particular, if the agent believes that something will act exactly as it does (as with Newcomb's Demon's prediction, since the prediction will be equal to its action), perhaps it should treat that thing \emph{as if} it were an exact copy of the agent.

\begin{eg}\label{eg:75Newcomb}
    Consider now a variant of Newcomb's problem in which the predictor is only 75\% accurate -- i.e., Newcomb's Demon has probability 0.75 of filling the opaque box if the agent one-boxes, and probability 0.25 otherwise. We then have that from an \textit{ex ante} perspective, the optimal policy is still to one-box, for an expected utility of \$750,000 compared to \$251,000 for two-boxing. We can again model this as occurring via simulations, for example as follows: Newcomb's Demon flips a coin to determine whether to simulate the agent. If the coin comes up Heads, the Demon simply randomises uniformly to determine whether to fill the opaque box. If the coin comes up Tails, the Demon simulates the agent, and fills the opaque box if and only if the agent one-boxes. How then should the agent divide its credence between the possible worlds, once it finds itself in Newcomb's Problem? In this model, the problem has the same structure as Sleeping Beauty \cite[][]{Elga2000}. There are two possible worlds -- in the Heads world the agent chooses only once (in the real Newcomb's problem), whereas in the Tails world the agent chooses twice (once in the real Newcomb's problem, once in the simulation). In problems such as Sleeping Beauty, probabilities become controversial. On the one hand, the agent faces Newcomb's problem at least once regardless of how the Demon makes its prediction, so its observations would seem to provide no update about which world it is in. According to this line of reasoning -- `Halfing' -- the agent should then assign credence $\frac{1}{2}$ to being in the world where the Demon simply flips a coin (giving probability $\frac{1}{4}$ of being in a simulation). On the other hand, only one third of expected copies of the agent are in the world where the Demon flips the coin. Thus, `Thirding' says that the agent should assign credence $\frac{1}{3}$ to being in that world (giving probability $\frac{1}{3} = \frac{1}{2}\times\frac{2}{3}$ of being in a simulation\footnote{Both approaches divide their credence evenly between copies within a world.}).
\end{eg}

The above example then fits into the setting of \cite{dese} \cite[cf.][]{piccione,Briggs2010}, %
in which all Newcomblike dependence is via exact copies. \citeauthor{piccione} show that, in such a setting, the \textit{ex ante} optimal policy is compatible with Generalised Thirding (GT) combined with CDT (or `CDT+GT').

In this paper, we ask: Can this approach work in general? That is, can we turn any Newcomb-like scenario into an equivalent scenario with simulations (and thus one in which CDT+GT is compatible with the \emph{ex ante} optimal policy)? To illustrate why this question is tricky, consider the following variant on the twin Prisoner's Dilemma (PD) \cite[][]{Brams1975,Lewis1979}:

\begin{eg}\label{eg:SB_PD}
 Dorothea and her identical twin Theodora, who reason about decision problems very similarly, have been caught robbing a bank. They each face a choice of whether to Cooperate (C), and not testify against their twin, or Defect (D), and testify against their twin. There's a catch: Theodora will first be given powerful sleeping medication that prevents her forming new memories, and will twice be awoken to face this decision. If on either occasion Theodora defects, it will incriminate Dorothea -- it is akin to her defecting in the original PD.
\begin{table}
\centering
\begin{tabular}{cc|c|c|c|c|}
  & \multicolumn{1}{c}{} & \multicolumn{4}{c}{Theodora's actions}\\
  & \multicolumn{1}{c}{} & \multicolumn{1}{c}{$C, C$}  & \multicolumn{1}{c}{$C, D$} & \multicolumn{1}{c}{$D, C$} & \multicolumn{1}{c}{$D, D$}\\
  \cline{3-6}
  \multirow{2}*{\shortstack{Dorothea's \\action}}  & $C$ & $3,3$ & $0,5$ & $0,5$ & $0,5$ \\
  \cline{3-6}
  & $D$ & $5,0$ & $2,2$ & $2,2$ & $2,2$ \\
  \cline{3-6}
\end{tabular}
\caption{The utilities for (Dorothea,Theodora) in \Cref{eg:SB_PD}.}\label{tab:utils}
\end{table}
Otherwise, if Theodora cooperates on both awakenings, it is treated as cooperation in the original PD. The rules, including the sleeping medication, are common knowledge. Each twin is self-interested and would like to maximise her own utility (which is decreasing in the length of her prison sentence).
 
 We display the utilities for the twins in \Cref{tab:utils}. We consider mixed strategies for Dorothea -- i.e., we allow Dorothea to randomise between cooperating and defecting. We consider such strategies both because access to randomisation devices is often plausible, and because the (ratificationist) version of CDT that we consider in this paper does not work well when randomisation is not possible, even in much simpler settings (cf.\ \Cref{sec:indep_rand}, and \citet[Section 6.1]{dese}).
Since Theodora reasons very similarly to Dorothea, Dorothea believes that her policy gives her evidence for Theodora's policy. In particular, let's say she thinks that if she cooperates with probability $p$, Theodora will cooperate with probability $F(p)$, for some \textit{dependence function} $F$. We consider two versions of the problem:
    \begin{enumerate*}[leftmargin=*,label=(\roman*)]
    \item Dorothea thinks that Theodora will act similarly to her for middling choices of $p$, but is more indecisive than Dorothea, and as Dorothea's choice tends towards deterministic, Theodora will continue to randomise with some probability. Specifically, $F(p) = \tfrac{1}{6}+2p^2 - \tfrac{4}{3}p^3$, as shown in \Cref{fig:graph_dep}.\label{case:theodorosimple}
    
    \item While Dorothea is making her decision, she will be given medication that with probability 0.2 confuses her and causes her to choose at random, and otherwise leaves her reasoning unaffected. Thus, given that an unmedicated Dorothea would cooperate with probability $p$, the medicated version of Dorothea will instead cooperate with probability $0.8p+0.1$. Then, Dorothea thinks that Theodora will act so as to cooperate overall (i.e., for both her awakenings to cooperate) with the same (\textit{ex ante}) probability as the medicated version of Dorothea, i.e., $F(p) = \sqrt{0.8p+0.1}$.\label{case:theodorocomplex}
\end{enumerate*}
Dorothea can model the medicated version of herself as simulating her with probability 0.8, and otherwise randomising evenly. But with Theodora, the dependency is much more complicated than the linear dependence of examples \Cref{eg:Newcombs,eg:75Newcomb}. Can Dorothea have credence in being somehow instantiated by Theodora such that the CDT recommendation is compatible with the \textit{ex ante} optimal policy?%

\begin{figure}[h]
\centering
\begin{tikzpicture}
\begin{axis}[
title={$F(p)=\frac{1}{6}+2p^2- \frac{4}{3}p^3$},
width = 2.5cm,
height = 2.5cm,
 xlabel={$p$},
ylabel={$F(p)$},
xmin=0, xmax=1,
ymin=0, ymax=1,
ytick={1/6,5/6,1},
yticklabels ={$\sfrac{1}{6}$,$\sfrac{5}{6}$,$1$},
xtick={0,1},
enlargelimits=false,
scale only axis=true
]
\addplot[color=red,domain=0:1]{1/6+2*x^2 - 4*x^3/3};

\end{axis}
\end{tikzpicture}
\caption{The dependence function in the first version of \Cref{eg:SB_PD}.}\label{fig:graph_dep}
\end{figure}
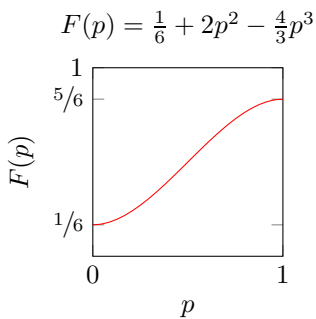
\end{eg}

\subsection*{Contributions}\label{sec:contribs}

In this paper, we analyse formally whether and how self-locating beliefs can be combined with CDT to obtain (EDT-style) \textit{ex ante} optimal policies. We consider a setting in which an agent faces a potentially Newcomblike problem (such as \Cref{eg:Newcombs,eg:75Newcomb,eg:SB_PD}), in which a number of parts of the environment, or `dependants', may depend upon the agent's policy via arbitrary functions (`dependence functions'). We formally introduce this setting in Section \ref{sec:setting}.

\paragraph{Simulation-based approach}\label{para:intro_simulation_based} First, we ask whether we can model Newcomblike scenarios via simulations of the agent, in such a way that CDT applied to the model with simulations gives good results.
Since such a model would then reduce to the (simpler) setting of \citeauthor{piccione}, we could then assign self-locating beliefs in accordance with GT over the simulations, and their result on CDT+GT policies would translate over to our setting. In Section \ref{sec:samples_model}, we consider three different ways to do this:
\begin{enumerate}[leftmargin=*,wide]
    \item The simplest approach is that for a given Newcomblike environment, we give a simulation-based environment that fully models it, and then apply CDT + GT to that, as we did for \Cref{eg:Newcombs,eg:75Newcomb}. More precisely, we model the dependence functions as arising from functions of some number of actions, sampled independently from the agent's policy. This allows us to apply the results on CDT+GT, which gives us that any \textit{ex ante} optimal policy is a CDT+GT policy in the environment with the simulations. We also obtain that the resulting %
    set of policies (of which there may be multiple) does not depend on the details of the model. We give necessary and sufficient conditions for such a simulation model to be possible -- essentially, that the dependence function be expressible as a specific kind of polynomial (\Cref{prop:global_sim,cor:poly_factor_char,cor:A=2,prop:glob_sims_necessary}). This approach, which we discuss in Section \ref{sec:glob_sims_model}, then solves case \ref{case:theodorosimple} of Example \ref{eg:SB_PD}. Consider this case now, recalling that Theodora's policy depends on Dorothea's policy $p$ via $F(p) = \tfrac{1}{6} + 2p^2 - \tfrac{4}{3}p^3$. We may equivalently write this as $\tfrac{1}{6}(1-p)^3 + 3\times\tfrac{1}{6}p(1-p)^2 + 3\times\tfrac{5}{6}p^2(1-p) + \tfrac{5}{6}p^3$. In turn, if $A_1,A_2,A_3\simiid pC + (1-p)D$ (i.e., are random variables taking value $C$ with probability $p$ and value $D$ otherwise), we have that $F(p) = \tfrac{1}{6}\prob{\text{At most 1 of } A_1, A_2, A_3 = C} + \tfrac{5}{6}\prob{\text{At least 2 of } A_1, A_2, A_3 = C}$. Thus, Dorothea can in this case model Theodora as running 3 simulations of her, looking at what action the majority take, and taking the same action as the majority of simulations with probability $\tfrac{5}{6}$ and the other action with probability $\tfrac{1}{6}$. Applying Generalised Thirding, Dorothea thinks she is then herself with probability $\tfrac{1}{2\times3+1}=\tfrac{1}{7}$ and one of Theodora's simulations with probability $\tfrac{6}{7}$ (recall that Theodora chooses an action twice).%
    \item To model a wider range of dependence functions, we look at what happens if we consider a sequence of simulation models that converge (in some sense, which we will specify) to the true environment. We show that this is possible whenever the dependence functions are continuous, and that the \textit{ex ante} optimal policies of the approximate models then converge to \textit{ex ante} optimal policies of the original problem (\Cref{prop:unif_convergence,prop:ex_ante_limit}). As a consequence, there is always a sequence of CDT+GT policies of the approximate problems that converges to the set of \textit{ex ante} optimal policies of the original problem (Corollary \ref{cor:CDTGT_limits}). We discuss this approach in Section \ref{sec:limit_approx}. This approach is applicable to both cases of Example \ref{eg:SB_PD}. %
    \item In our third approach, we again consider modelling the dependence via simulations, but where the nature of these depends non-causally on the agent's policy. We do so in such a way that the simulations capture the \textit{local} dependence on the policy. We introduce and motivate this approach in Section \ref{sec:local_sims}. The resulting set of CDT policies includes the \textit{ex ante} optimal policies and does not depend on the details of the simulation model (Theorem \ref{thm:local_sims_result}). Moreover, when the first simulation-based approach is also applicable, the CDT policies resulting from that approach are identical to those of this approach. %
    We show that we can apply this method exactly when the dependence function is differentiable (Propositions \ref{prop:local_sample_equiv}), and so this solves both cases of Example \ref{eg:SB_PD}.
    \end{enumerate}

\paragraph{General approach} In Section \ref{sec:GGT}, we look at a more parsimonious way to assign self-locating beliefs in our setting, which we call `Generalised Generalised Thirding' (GGT). This approach is applicable whenever the dependence functions are differentiable.
We characterise the GGT+CDT policies (Theorem \ref{thm:main}) and obtain that all \textit{ex ante} optimal policies are GGT+CDT policies (Corollary \ref{cor:main}).
We find that this approach is closely related to the earlier simulation-based approaches: %
It leads to the same set of CDT policies, and the parameters of GGT admit an interpretation in terms of the simulation models (\Cref{prop:local_sample_model}). We discuss this in Section \ref{sec:relating_to_sims}.

\paragraph{Limitations}
Finally, we turn to the limits of our approaches, and to the fundamental limits of trying to obtain the \textit{ex ante} optimal policy via modifying the self-locating beliefs of CDT.

All our approaches require that the dependence function be at least continuous. There exist scenarios with discontinuous dependence functions in which no reasonable theory of self-locating beliefs combined with CDT is always compatible with the \emph{ex ante} optimal policy (cf.\ \Cref{sec:limitations}).

Another assumption we make is that, holding fixed the policies of the different dependants, their actions at different time points are sampled independently at random. In Section \ref{sec:indep_rand}, we show that this assumption also cannot be relaxed.

\section{Setting and background}\label{sec:setting}
\subsection{Our setting}
We are interested in Newcomblike decision problems in which an agent chooses between some set of actions $A$, according to a policy $\pi_0$, which determines (probabilistically) the policies of some number of \textit{dependants}. These are parts of the environment (possibly other agents) that depend (subjectively) on the agent's policy. For instance, this dependence could be via prediction (a la Newcomb's problem or Death in Damascus \cite[][Sect.\ 11]{Gibbard1981}) or via similarity (a la the twin Prisoner's Dilemma, cf.\ \citet{Brams1975}, \citet{Lewis1979}, \citet{Hofstadter1983}). We define \textit{decision problems} for our purposes as follows:

\begin{defn}[Decision problem]
    A \emph{decision problem} is a tuple $(S, S_T, P_0, n, i, A, T, u, \F)$, comprising:
    \begin{itemize}[leftmargin=*, topsep=0pt,itemsep=-1ex,partopsep=1ex,parsep=1ex]
        \item a set of \textit{states} $S$, with a subset $S_T$ of \textit{terminal} states;
        \item an \textit{initial distribution} $P_0 \in \Delta(S-S_T)$ over the non-terminal states;
        \item a number $n \in \mathbb{N}$ of \emph{dependants};
        \item an \textit{index function} $i:S-S_T\rightarrow \{1, \ldots n\}$ mapping non-terminal states to dependants $1, \ldots n$;
        \item a finite set $A$ of \textit{actions} that the agent and dependants choose from;
        \item a \textit{transition function}, a probabilistic mapping $T: (S-S_T)\times A \rightarrow \Delta(S)$, giving, for each non-terminal state $s$ and action $a$, a distribution $T(\cdot \mid s,a)$ over successor states;
        \item a \textit{utility function} $u:S_T\rightarrow \mathbb{R}$ for the agent;
        \item a \textit{dependence function} $\F= (F_1,\ldots,F_n)$, with $\F:\Delta(A)\rightarrow\Delta(A)^n$, mapping the agent's policy onto policies for the dependants.
    \end{itemize}
\end{defn}

Write $\pi_0\in\Delta(A)$ for the agent's policy, which cannot depend on the state (cf.\ \Cref{rk:obs}). Write $\pi_j \in \Delta(A)$ for dependant $j$'s policy, and $\bm{\pi}$ for the joint policy $(\pi_1,\ldots,\pi_n)$ of the dependants. Dependant $j$'s policy depends on the agent's policy via the dependence function: $\pi_j=F_j(\pi_0)$.

    We will write $F_j(a\mid \pi_0) \defeq F_j(\pi_0)(a)$ for the probability of $a$ in the policy $F_j(\pi_0)$.

Decision problems then work as follows:
\begin{enumerate}[nolistsep]
    \item The agent chooses a policy $\pi_0 \in \Delta(A)$.
    \item The dependants' policies are $\bm{\pi} = \F(\pi_0) \in \Delta(A)^n$.\footnote{We assume that dependants all have the same action space as each other and as the agent for convenience only. This assumption is in fact WLOG -- cf.\ \Cref{sec:act_space_WLOG}.}
    \item An initial state $s_0$ is sampled, according to the initial distribution $P_0$.
    \item Dependant $i(s_0)$ is associated with state $s_0$. An action, $a_0$ is drawn from its policy $\pi_{i(s_0)}$.
    \item A successor state $s_1$ is drawn from the distribution given by the transition function $T(\cdot\mid s_0, a_0)$.
    \item The previous two steps are repeated until a terminal state is reached, each time drawing actions and successor states independently of any previous events.
    \item When terminal state $s_t$ is reached, the agent obtains utility $u(s_t)$, and the decision problem ends.
\end{enumerate}

Given a policy $\pi_0$ this induces a stochastic process: a Markov chain (with initial distribution $P_0$ and transition matrix $T_{s,s'} = \E{T(s'\mid s, A_t)\mid A_t\sim \F_{i(s)}(\pi_0)}$).\footnote{Strictly, for this to be a Markov chain, the transition matrix should say what happens at terminal states. We may simply treat these as absorbing states (i.e., set $T_{s,s} =1$ for $s\in S_T$), although we will not be concerned with anything that might happen after a terminal state is reached.}

To be able to apply theories of self-locating beliefs, and for CDT to be well-defined, we need this process to terminate with probability 1. We make the following assumption throughout:

\begin{asm}\label{asm:finite_hist}
For all non-terminal states $s\in S$, and all policies $\bm{\pi} = (\pi_1,\ldots,\pi_n)$, the decision problem almost surely terminates if we start in state $s$ and dependants follow policy $\bm{\pi}$. %
\end{asm}

\begin{rk}[Observations and copies]\label{rk:copies}
    Note that while the agent itself does not directly appear in the decision problem, if dependant $j$'s dependence function $F_j$ is the identity, it is effectively a copy of the agent. That said, our setting does not necessarily have enough information to say what naive CDT (without modified anthropics) would do -- in particular, we do not specify observations associated with each state, which would usually be used by the agent to determine which state it might be at. Thus, our setting does not distinguish exact copies of the agent (who have all the same observations as the agent) with decision theoretic copies of the agent (who act the same as the agent in all relevant situations). We discuss the reasons for not including different observations, and give results for a possible generalisation of our setting to multiple observations in \Cref{rk:obs}.
\end{rk}

Note that our setting is similar to that of a `single-player game of imperfect recall' from \cite{dese}. The main difference is the introduction of dependants whose actions depend probabilistically on the action of the agent, rather than only including exact copies of the agent.

\begin{eg}[continues= eg:SB_PD]\label{eg:SB_PD_formal}
We now formalise \ref{eg:SB_PD} as a decision problem in our setting, illustrated in \Cref{fig:SB_PD}. We have non-terminal states $S-S_T = \{\mathit{Th}_0,\mathit{Th}_C,\mathit{Th}_D,\mathit{Do}_C,\mathit{Do}_D\}$ and terminal states $S_T = \{0,2,3,5\}$. We may interpret $\mathit{Th}_0$ (resp. $\mathit{Th}_C$ or $\mathit{Th}_D$) as the state at which Theodora makes her first (resp. second) decision, and $\mathit{Do}_C$ or $\mathit{Do}_D$ as the state at which Dorothea makes her decision. We model Theodora's decisions as occurring first, followed by Dorothea's decision. The initial distribution  is $P_0(\mathit{Th}_0)=1$. There are two dependants: dependant 1 is Dorothea, and dependant 2 is Theodora. The index function $i$ is 2 on $\{\mathit{Th}_0, \mathit{Th}_C, \mathit{Th}_D\}$ and 1 on $\{\mathit{Do}_C,\mathit{Do}_D\}$. The actions are $A=\{C,D\}$. The transition function is $T(Th_a\mid \mathit{Th}_0,a) = 1$ (for $a=C,D$), $T(\mathit{Do}_C\mid \mathit{Th}_C, C) = 1$, $T(\mathit{Do}_D\mid \mathit{Th}_C,D) = 1$, $T(\mathit{Do}_D\mid \mathit{Th}_D, a) = 1$ for all $a$. $T(3\mid \mathit{Do}_C,C)=T(5\mid \mathit{Do}_C,D)= T(0 \mid \mathit{Do}_D,C)=T(2\mid \mathit{Do}_D,D) = 1$. The utility function is $u(j) = j$ for $j=0,2,3,5$.

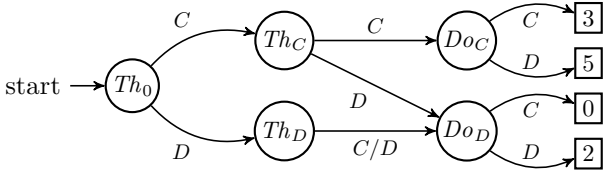
\begin{figure}
\centering
\begin{tikzpicture}[->, >=stealth', auto, semithick, node distance= 0.75cm and 2.75cm, on grid,
	accepting/.style = {shape = rectangle, minimum size=0.5pt, inner sep = 3pt},
	every state/.style = {fill=none,draw=black,thick,text=black, minimum size = 0.5pt, scale=0.9, inner sep = 1pt},
	every text node part/.style={align=center}]
	\node[state,initial]    (Th0) {$\mathit{Th}_0$};
	\node[state]    (ThC)[above right = 0.6cm and 2cm of Th0] {$\mathit{Th}_C$};
	\node[state]    (ThD)[below right = 0.6cm and 2cm of Th0]   {$\mathit{Th}_D$};
	\node[state]    (DoC)[right = 2.4cm of ThC] {$\mathit{Do}_C$};
	\node[state]    (DoD)[right = 2.4cm of ThD] {$\mathit{Do}_D$};
    \node[state,accepting]    (3)[above right = 0.3cm and 1.6cm of DoC]{3};
    \node[state,accepting]    (5)[below right = 0.3cm and 1.6cm of DoC]{5};
    \node[state,accepting]    (0)[above right = 0.3cm and 1.6cm of DoD]{0};
    \node[state,accepting]    (2)[below right = 0.3cm and 1.6cm of DoD]{2};
	\path
	(Th0) edge[bend left, above left]	node[scale = 0.8]{$C$}	(ThC)
	(Th0) edge[bend right, below left]		node[scale = 0.8]{$D$}	(ThD)
	
	(ThC) edge node[scale = 0.8]{$C$} (DoC)
	(ThC) edge node[scale = 0.8, below left, distance = 0.5cm]{$D$} (DoD)
	
	(ThD) edge node[below, scale = 0.8]{$C/D$} (DoD)

	(DoC) edge[bend left, below] node[scale = 0.8]{$C$} (3)
	(DoC) edge[bend right, above] node[scale = 0.8]{$D$} (5)
	(DoD) edge[bend left, below] node[scale = 0.8]{$C$} (0)
	(DoD) edge[bend right, above] node[scale = 0.8]{$D$} (2);
	\end{tikzpicture}
 \caption{A graph of \Cref{eg:SB_PD}. The nodes show the states and the edges show the transitions (all deterministic). The utilities of the terminal states are the same as the states' labels.}\label{fig:SB_PD}
\end{figure}
We considered two versions of the problem, with different dependence functions. We now express these as a function from $\Delta(A)$ to itself, writing probabilities of $C$ before $D$. For the first version, we have $F_1(\pi_0) = \pi_0$, and $F_2(\pi_0) = (\tfrac{1}{6}+2\pi_0(C)^2 - \tfrac{4}{3}\pi_0(C)^3,\tfrac{5}{6}-2\pi_0(C)^2+ \tfrac{4}{3}\pi_0(C)^3)$. For the second, we have $F_1(\pi_0) = 0.2\times(0.5,0.5)+0.8\pi_0$ and $F_2(\pi_0)= (\sqrt{0.1+0.8\pi_0(C)}, 1- \sqrt{0.1+0.8\pi_0(C)})$.

As an example, suppose that in the first version of this problem, Dorothea cooperates deterministically. Then Theodora will cooperate with probability $\tfrac{5}{6}$ at each point when she acts. The initial state will be $s_0 = \mathit{Th}_0$, followed by $s_1 = \mathit{Th}_C$ with probability $\tfrac{5}{6}$ and $s_1 = \mathit{Th}_D$ otherwise. Then, we arrive in state $s_2=\mathit{Do}_C$ with probability $\tfrac{5}{6}\prob{s_1 = \mathit{Th}_C} = \tfrac{25}{36}$, and $s_2=\mathit{Do}_D$ otherwise. Since Dorothea will then cooperate, her \textit{ex ante} expected utility (cf.\ \Cref{defn:ex_ante_EU}, p.\ \pageref{defn:ex_ante_EU}) is then $3\prob{s_2 = \mathit{Do}_C} = \tfrac{25}{12}$.
\end{eg}

\begin{notn}
    For any policy $\tau \in \Delta(A)$, denote by $T(\cdot\mid s,\tau)$ the distribution of successor states induced by choosing an action according to $\tau$. That is, 
    $
        T(s'\mid s, \tau) \defeq \sum_{a\in A}\tau(a) T(s'\mid s,a) = \E{T(s'\mid s, A_t) \mid A_t \sim \tau}.
    $
\end{notn}

\begin{defn}[History]
    A \textit{history} is a finite sequence of states $s_0 s_2 \cdots s_t$, where $s_0,\ldots,s_{t-1}\in S-S_T$ are non-terminal, and $s_t\in S_T$ is terminal. Given dependants' policies $\bm{\pi}= \F(\pi_0)$, our setting induces a distribution over histories:\footnote{
    This is the probability, given $\bm{\pi}$, of the induced Markov chain having this sequence of states (up to and including the time step in which it hits a terminal state).%
    }
    $
        \prob{s_0\cdots s_t\mid \bm{\pi}} = P_0(s_0)T(s_1\mid s_0, \pi_{i(s_0)})\cdots T(s_t\mid s_{t-1}, \pi_{i(s_{t-1})}).
    $
    Write $\mathcal{H}$ for the set of histories.
\end{defn}

\begin{notn}
    By $\mathbb{P}^s$ denote the probability distribution over histories induced by starting the decision problem deterministically at state $s$. Similarly, write  $\Ex^s$ for the expectation operator when %
    starting the decision problem at state $s$ (instead of from the initial state distribution $P_0$).%
\end{notn}

\begin{restatable}{lemma}{exphistlengthfinite}\label{cor:exp_hist_length_finite}
    Given Assumption \ref{asm:finite_hist}, the expected history length is finite.
\end{restatable}
\excludeforjournal{
\begin{proof}
    See \Cref{sec:hist_length_exp}.
\end{proof}
}

\subsection{Definitions and prior work}
We shall be concerned with the following question:
How should the agent choose its policy so as to maximise its expected utility? There are different ways of interpreting this question. These vary along the dimensions of self-locating beliefs and decision theory.

First, we turn to \textit{ex ante} expected utility:
\begin{defn}[\textit{Ex ante} expected utility]\label{defn:ex_ante_EU}
    Given a decision problem $(S,S_T,P_0,n,i,A,T,u,\F)$, the \textit{ex ante expected utility} of dependants following policy $\bm{\pi}$ is the expected utility of the terminal state in the history: $\E{u\mid \bm{\pi}} \defeq \sum_{s_0\cdots s_t \in \mathcal{H}}u(s_t)\prob{s_0\cdots s_t\mid\bm{\pi}}$.
    Similarly, the \textit{ex ante expected utility} of the agent following policy $\pi_0$ is $\E{u\mid \pi_0} \defeq \E{u\mid \bm{\pi}=\F(\pi_0)}$.
\end{defn}
\begin{rk}
    Note that our definition is the EDT \textit{ex ante} expected utility. That is, it is obtained by simply conditioning on our policy $\pi_0$. For discussion of how this differs from the CDT \textit{ex ante} expected utility (which we do not consider here), see \Cref{sec:ex_ante_EU_rk}.%
\end{rk}

We introduce the following notation. We sometimes use $a$ to refer to the pure policy $\pi\in \
\Delta(A)$ which chooses action $a$ with probability 1. We write $x_{1:n}$ for the tuple $(x_1,\ldots,x_n)$, and  $x_{-j}$ for the tuple $(x_1,\ldots,x_{j-1},x_{j+1},\ldots,x_n)$.

We will now consider \textit{de se} approaches to decision making. In \textit{de se} approaches, the agent forms beliefs about where it is and evaluates the consequences of taking an action at this particular decision point (as opposed to evaluating a policy by looking at the scenario as a whole).

\begin{defn}[(Theory of) self-locating beliefs]
Given a decision problem, \emph{self-locating beliefs} are probability distributions $P$ over non-terminal states $s$.
A \textit{theory of self-locating beliefs} $X$ is then a map from (some subset of) decision problems $D$ and policies $\pi_0$ to self-locating beliefs $P_X(\pi_0,D)$ for $D$. We assume that the self-locating beliefs do not depend on the utility function in $D$.
\end{defn}

We now define \textit{Generalised Thirding}, a theory of self-locating beliefs proposed by \citet{piccione} as \textit{consistency}. Generalised Thirding is also sometimes called the \textit{Self Indication Assumption} \cite[][]{Bostrom2010}.
Roughly, according to Generalised Thirding (GT), one should assign credence to being in a particular state equal to the proportion of expected agents who are in that state. For instance, in \Cref{eg:75Newcomb}, there are three possible states: being the real agent when the coin came up heads, being the real agent when the coin came up tails, and being the simulated agent when the coin came up tails. The number of expected agents in each of these states is $\tfrac{1}{2}$ (note that in the tails world, the agent appears at both states), and so GT assigns an equal credence of $\tfrac{1}{3}$ to each state (and hence $\frac{2}{3}$ to the coin having come up tails). %

For our setting, we define Generalised Thirding such that we treat all dependants as copies of the agent, since we will only apply it in the case where all dependence functions are the identity:

\begin{notn}
Let $\#(s)$ (resp.\ $\#(j)$) be the number of instances of state $s$ (resp.\ dependant $j$) in the history.
\end{notn}

\begin{defn}[Generalised Thirding]
    Generalised Thirding (GT) is a theory of self-locating beliefs defined by, for $s\in S-S_T$, when the agent has policy $\pi_0$, 

    \begin{equation}
    P_{GT}(s \mid \pi_0) \defeq \frac{\E{\#(s)\mid \pi_0}}{\sum\limits_{s'\in S-S_T} \E{\#(s')\mid \pi_0}}.
    \end{equation}
    
\end{defn}

\begin{rk}
    By  \Cref{cor:exp_hist_length_finite} the expected history length is finite, ensuring that the GT credences are well-defined.
\end{rk}

We will now define CDT, to which we will take a \textit{ratificationist} approach \cite[cf.][Sect.\ 3.5]{piccione,GranCanaria,dese,Weirich2016}. %
In our setting, CDT works by considering what would happen if the agent counterfactually took a particular action right now at its current state, without changing anything about what happens at other time steps (in particular, without taking into account any evidence it might gain about its, or its dependants', actions at other time steps). For instance, in \Cref{eg:75Newcomb}, the CDT expected utility of one-boxing minus that of two-boxing in the simulation is $\$1,000,000$, and the CDT expected utility of one-boxing minus that of two-boxing as the real agent is $-\$1,000$. Based on its theory of self-locating beliefs, the CDT agent then aggregates over states to compute the overall expected causal value of taking each action.

Now, in general, the causal impact of an agent taking a particular action depends on the agent's policy -- whether it is good to take a particular action at a particular state can depend on what will happen in the future, and in particular on how the agent will act at future states. Thus, to evaluate the CDT expected utility of the agent taking action $a$ at state $s$, we require beliefs as to how the agent behaves at different points in time. A policy is \textit{ratifiable} if it is causally optimal assuming that the agent follows that same policy. This is an equilibrium-like concept: $\pi_0$ is ratifiable if given that a CDT agent expects to follow $\pi_0$, the agent does not wish to deviate from $\pi_0$. In general there may be multiple ratifiable policies (cf.\ e.g.\ \citet[Section 6.3]{dese}, \citet{aumannhartperry}), or even no ratifiable policies (cf. \citet[Section IV.4]{CDTDutchbook}). We now formally define CDT expected utility (using `$\Do$' notation for causal counterfactuals \citep{Pearl_2009}). %

\begin{defn}[CDT expected utility]
    The \emph{CDT expected utility} of an agent taking action $a$, given a theory of self-locating beliefs $X$ and assuming that the agent follows policy $\pi_0$ at other time points is:
    \begin{align*}
        &\Ex_X[u\mid do(a),\pi_0]\\
        &\defeq \E{\Ex^{s'}[u\mid \pi_0]\middle| s\sim P_X(\cdot\mid \pi_0), s' \sim T(\cdot\mid s, a)}\\
        &\defeq \sum_{s} P_X(s\mid \pi_0)\sum_{s'}T(s'\mid s,a)\Ex^{s'}[u\mid\pi_0]
    \end{align*}
\end{defn}

\begin{notn}
    Given $\gamma \in \Delta(A)$, write $\Ex_X[u\mid \Do(\gamma),\pi_0]$ for $\E{\Ex_X[u\mid \Do(A_1),\pi_0]\mid A_1\sim\gamma} = \sum_a \gamma(a) \Ex_X[u\mid \Do(a),\pi_0]$.
\end{notn}
We are now ready to define CDT policies:
\begin{defn}[CDT+X policy]\label{def:CDT_policies}
    Suppose $X$ is a theory of self-locating beliefs. Then $\pi_0$ is a \emph{CDT+X (compatible) policy} for the agent, or is \emph{ratifiable} if $\Ex_X[u\mid \Do(\pi_0),\pi_0] = \max_a \Ex_X[u\mid \Do(a),\pi_0]$.
    Equivalently, $\pi_0$ is a CDT+X policy if for all $a$ with $\pi_0(a)>0$, we have $\Ex_X[u\mid \Do(a),\pi_0] = \max_{a'} \Ex_X[u\mid \Do(a'),\pi_0]$.
\end{defn}

Mathematically, the CDT+GT expected utility of an action $a$ is closely related to the derivative of the ex ante expected utility in the direction of $a$. We define derivatives on the simplex as follows:

\begin{defn}[Differentiability/derivatives on the simplex]\label{def:differentiability}
    Write $\frac{\partial}{\partial \pi_0(a)}$ for the directional derivative of a function on $\Delta(A)$, in the direction $(a - \pi_0)$, viewed as a vector. That is:
    \begin{equation*}
        \frac{\partial f(\pi_0)}{\partial \pi_0(a)} \defeq \nabla_{a-\pi_0}f(\pi_0) =   \lim_{\eps\rightarrow 0} \frac{f(\pi_0+\eps(a - \pi_0)) - f(\pi_0)}{\epsilon}.
    \end{equation*}
    We say $f$ is \textit{differentiable} at $\pi_0$ if $\frac{\partial f}{\partial \pi_0(a)}$ exists for all $a$, and
    \begin{equation*}
        f(\pi_0+\eps(\pi_0'-\pi_0)) = f(\pi_0) + \eps\sum_a \pi_0'(a) \frac{\partial f}{\partial \pi_0(a)} + o(\eps).
    \end{equation*}

\end{defn}

\cite{piccione} show that, in their setting (without dependence functions), any \textit{ex ante} optimal policy is a CDT+GT policy \cite[cf.][]{Briggs2010}. \cite{dese} give an explicit characterisation of CDT+GT policies, based on \citeauthor{piccione}'s proof. We restate these results here in terms of our setting.

\begin{thm}[\citeauthor{dese}, \citeyear{dese}, Theorem 4]\label{thm:CDT+GT}
Suppose we have a decision problem in which
the dependence functions are all the identity. Then $\pi_0$ is a CDT+GT policy if and only if for all $a \in A$ we have $\frac{\partial}{\partial \pi_0(a)}\E{u\mid \pi_0} \leq 0$.
\end{thm}
\begin{cor}[\citeauthor{piccione}, \citeyear{piccione}, Proposition 3]\label{cor:CDT+GT}
Suppose we have a decision problem in which
the dependence functions are all the identity. Then all ex ante optimal policies are CDT+GT policies.
\end{cor}

\subsection{Generalised theories of self-locating beliefs}
We will sometimes want a more complex model of what happens given that we are at a particular state. For instance, we might think that a dependant's action depends on our action in some way other than just doing what we do (as in \Cref{eg:SB_PD}). Alternatively, the dependant might be a decision-theoretic copy who faces a decision isomorphic to our own, but where the analogous actions have different labels. We thus introduce the notion of generalised theories of self-locating beliefs that allow for such models:

\begin{defn}[Transformation functions]
    Given a decision problem with $n$ dependants, action set $A$, and a policy $\pi_0$, \emph{transformation functions} are functions $\tau_1,\ldots,\tau_n$ from actions to policies: $\tau_i(\cdot, \pi_0):A\rightarrow \Delta(A)$.
\end{defn}

\begin{defn}[Generalised theory of self-locating beliefs]
    A \emph{(generalised) theory of self-locating beliefs} is a map from (some subset of) decision problems $D$ and policies $\pi_0$ to self-locating beliefs and sets of transformation functions, independent of the utility function in $D$.
\end{defn}

In a generalised theory of self-locating beliefs, given that the agent thinks it is at a state associated with dependant $j$, it believes that when it takes a particular action, its action will be transformed according to the relevant transformation function to obtain the action of the dependant at the state. The special case where all transformation functions are always just the identity is then equivalent to a non-generalised theory of self-locating beliefs.

We then define the CDT expected utility given a generalised theory of self-locating beliefs as follows:

\begin{defn}[CDT expected utility under generalised self-locating beliefs]
        If $X$ is a generalised theory of self-locating beliefs with transformation functions $\tau_{1:n}$, the \emph{CDT+$X$ expected utility} is:
    \begin{align*}
        &\Ex_X[u\mid \Do(a),\pi_0] \defeq\\ %
        &\sum_{s} P_X(s\mid\pi_0)\sum_{s'}T(s'\mid s,\tau_{i(s)}(a,\pi_0))\Ex^{s'}[u\mid\pi_0].
    \end{align*}
\end{defn}

For a generalised theory of self-locating beliefs $X$, the CDT+$X$ policies are defined as in \ref{def:CDT_policies} (with respect to the CDT+$X$ expected utility as defined above).

\section{Self-locating beliefs based on simulation models}\label{sec:samples_model}
\subsection{Global simulation models}\label{sec:glob_sims_model}
Consider a Newcomblike decision problem. Suppose that we can write each dependence function $F_i(\pi_0)$ as sampling from $\pi_0$ and then applying a function to the samples. That is, for some $N\in\mathbb{N}$ and $g:A^N\rightarrow \Delta(A)$, we have that $F_i(\pi_0)$ gives the strategy resulting from applying $g$ to $N$ samples of actions from $\pi_0$:
\begin{equation}\label{eq:global_samples}
\begin{split}
    F_i(\pi_0) &= \sum_{a_{1:N}} \prob{A_{1:N}=a_{1:N}\mid A_{1:N}\simiid \pi_0} g(a_{1:N})\\
    &= \E{g(A_{1:N})\mid A_{1:N}\simiid \pi_0}.
\end{split}
\end{equation}

Then we may consider a modified version of the decision problem, where we replace each state associated with a dependant with a series of states associated with each of the simulations, and modify the transition function accordingly. In the resulting problem, our dependence functions are then all the identity, and we may apply the results for CDT+GT.

\begin{eg}[continues= eg:SB_PD_formal]\label{eg:glob_sim_early}
We now apply this approach to \Cref{eg:SB_PD}, which we previously formalised in our setting (p.\ \pageref{eg:SB_PD_formal}). We'll restrict our attention to the first version of this problem for now. We have dependence functions: $F_1(\pi_0) = \pi_0$ and $F_2(\pi_0) = (\tfrac{1}{6}+2\pi_0(C)^2 - \tfrac{4}{3}\pi_0(C)^3,\tfrac{5}{6}-2\pi_0(C)^2+ \tfrac{4}{3}\pi_0(C)^3)$. Clearly, $F_1$ can be viewed as the distribution arising from taking a single sample of an action from $\pi_0$, and and outputting that same action. Can we find $g$ such that $F_2(\pi_0)$ is the distribution arising from applying $g$ to some number of samples from $\pi_0$? Yes -- set $g(a_1,a_2,a_3) = (\tfrac{5}{6},\tfrac{1}{6})$ when at least two of $a_{1:3} = C$, and $(\tfrac{1}{6}, \tfrac{5}{6})$ otherwise. We then have $F_2(\pi_0) = \E{g(A_1,A_2,A_3) \mid A_{1:3}\simiid \pi_0}$. (Cf.\ \Cref{sec:contribs}, \nameref{para:intro_simulation_based}. See also \Cref{sec:ThDo_1_details} for further algebraic details.) Thus, we can model Dorothea as simulating herself once, and Theodora as simulating Dorothea three times on each awakening.

Now, we sketch how to apply CDT+GT to the resulting simulation model (see \Cref{sec:glob_sims_eg_construction} for a formal construction of the model). We first compute the GT probabilities. Since each awakening of Theodora runs 3 copies of Dorothea, and there are two such awakenings, there are 6 copies of Dorothea simulated by Theodora, and only 1 copy of Dorothea as herself. Thus, GT gives credence $\tfrac{6}{7}$ of Dorothea being in one of Theodora's simulations.

We now compute the causal expected utility of Theodora choosing $C$ over $D$ in each possible state.
First, suppose Dorothea is in one of Theodora's simulations, and expects to act according to policy $\pi_0$. If she plays action $a$, she will expect Theodora to choose according to strategy $\E{g(a,A_2,A_3)\mid A_{2:3}\simiid \pi_0}$ (using symmetry of $g$) that awakening. The difference in cooperation probability for Theodora this awakening between $a=C$ and $a=D$ is then is equal to $(\tfrac{5}{6}-\tfrac{1}{6})\prob{A_{2:3}= (C,D) \text{ or } (D,C)\mid \pi_0} = \tfrac{4}{3}\pi_0(C)\pi_0(D)$. Then, the increase in probability of \emph{both} copies of Theodora cooperating from Dorothea cooperating rather than defecting is given by $F_2(\pi_0)_1\times \tfrac{4}{3}\pi_0(C)\pi_0(D)$. When this happens, Dorothea's utility increases by 3 compared to defecting. If Dorothea is instead herself, cooperating reduces her utility by 2 compared to defecting.

We now find the CDT+GT policies in the simulation model. Recall that these are the policies $\pi_0$ such that each action in the support of $\pi_0$ has maximal CDT+GT expected utility given $\pi_0$. %
We have that the overall difference in expected utility of cooperating rather than defecting given policy $\pi_0$ is $\tfrac{6}{7}\times4F_2(\pi_0)_1\pi_0(C)\pi_0(D)-\tfrac{1}{7}\times 2$. We can see from this that if $\pi_0$ is a pure strategy, cooperating has expected CDT+GSGT utility $-\tfrac{2}{7}$ compared to defecting. Thus, cooperating is not a CDT+GT strategy in the simulation model, but defecting is. The CDT+GT mixed strategies have full support, and so must have $\Ex_{GT}[u\mid do(C), \pi_0] - \Ex_{GT}[u\mid do(D),\pi_0] = 0$. These are then the roots of a quintic, and are not analytic, but can be computed as $\pi_0(C) \approx 0.36$ and $\pi_0(C) \approx 0.88$. We verify that the CDT+GT policies for the simulation-based decision problem include the \textit{ex ante} optimal policy. We have \textit{ex ante} expected utility $\E{u\mid \pi_0} = 2(1-\pi_0(C)) + 3 (\tfrac{1}{6}+2\pi_0(C)^2-\tfrac{4}{3}\pi_0(C)^3)^2$,
which is maximised at $\pi_0(C) \approx 0.88$. $\pi_0(C)\approx0.36$ is instead the global minimum.
\end{eg}

\begin{rk}
    One immediate question is why we ought to have to write the dependence function as depending only on the \emph{actions} of the simulated agent, rather than directly on its policy. In settings where a predictor can run an agent's source code, for instance, it is natural to imagine that when the agent randomises, the predictor might be able to simply run a single simulation, and directly compute the probabilities with which the agent takes each action, rather than only being able to sample from it. One possible approach might then be to use the original dependence function in place of transformation functions. %
    We discuss this approach in \Cref{sec:N_sample_AO}, but find it unpromising -- its recommendations are neither compatible with the \textit{ex ante} optimal policy, nor are they intuitive. %
\end{rk}

We do not directly give a construction of a simulation-based model from a decision problem in the main text. Instead, we describe self-locating beliefs for the original problem which are equivalent to applying GT to a modified problem based purely on action simulations. We give an explicit construction of the relevant simulation model, and a proof of this equivalence, in the proof of \Cref{thm:global_sims_result} (in \Cref{pf:globalsimsresult}), and for \Cref{eg:SB_PD} in \Cref{sec:glob_sims_eg_construction}.

\begin{defn}[Global simulation-based self-locating beliefs]
    Let $(S,S_T,P_0,n,i,A,T,u,\F)$ be a decision problem. Suppose that for $j= 1,\ldots, n$, there exists $N_j\in \mathbb{N}$ and $g_j:A^{N_j}\rightarrow \Delta(A)$ such that $F_j(\pi_0) = \E{g_j(A_{1:N_j})\mid A_{1:N_j}\simiid \pi_0}$ for all $\pi_0$.
    Then a set of generalised self-locating beliefs comprising self-locating beliefs $P$ and transformation functions $\tau_{j}$ defined by:
    \begin{align*}
        P(s \mid \pi_0) &= \frac{N_{i(s)}\E{\#(s)\mid \pi_0}}{\sum_{s'\in S-S_T}N_{i(s')}\E{\#(s')\mid \pi_0}}\\
        \tau_j(a, \pi_0) &= \sum_{k=1}^{N_j}\frac{1}{N_j}\E{g_j(A_{1:N})\mid A_k = a, A_{-k} \simiid \pi_0}.
    \end{align*}
    is a set of \emph{global-simulation based Generalised Thirding (GSGT)} beliefs. These are the beliefs corresponding to applying GT to the model in which dependant $j$ runs $N_j$ simulations of the agent.
\end{defn}
\begin{rk}
    Note that GSGT beliefs are not unique. If a function can be expressed as arising from $N$ simulations, it can also be expressed as arising from $N+1$ simulations (e.g., by ignoring the first simulation), and so any decision problem that admits GSGT beliefs admits infinitely many such beliefs. 
\end{rk}

\begin{restatable}{thm}{globalsimsresult}\label{thm:global_sims_result}
    Let $(S,S_T,P_0,n,i,A,T,u,\F)$ be a decision problem admitting GSGT beliefs. Then, for any set of GSGT beliefs, the CDT+GSGT policies are precisely those $\pi_0$ for which, for all $a$, we have $\frac{\partial}{\partial \pi_0(a)}\E{u\mid \pi_0} \leq 0$.
    In particular, they include the ex ante optimal policies.
\end{restatable}
\excludeforjournal{
\begin{proof}
    See \Cref{pf:globalsimsresult}, p.\ \pageref{pf:globalsimsresult}.
\end{proof}
}

So if the dependence can be written as a function of action samples (as in \Cref{eq:global_samples}), we can justify the \textit{ex ante} optimal policy from a \textit{de se} perspective. In particular, it is a CDT+GSGT policy, or a CDT+GT policy in a simulation-based model of the original decision problem. Now we are left with the question of when the dependence functions can be written in this fashion. First, we give a necessary and sufficient condition: that the dependence functions can be written as polynomials with non-negative coefficients (where we say a vector is non-negative if all of its entries are non-negative).

\begin{restatable}{prop}{globalsim}\label{prop:global_sim}
    Let $F:\Delta(A) \rightarrow \Delta(A)$, and $N\in \mathbb{N}$. The following are equivalent:
    \begin{enumerate}[label=(\roman*),ref=(\roman*)]
   \item There exists $g:A^N \rightarrow \Delta(A)$ such that
    $
        F(\pi_0) = \E{g(A_1,\ldots,A_N)\mid A_1,\ldots,A_N \sim \pi_0}.
    $
    \label{global_sims}
   \item $F$ is a polynomial and may be written with only non-negative coefficients and degree at most $N$. \label{poly}
  \end{enumerate}
    
\end{restatable}
\excludeforjournal{
\begin{proof}
    See \Cref{pf:globalsim}, p.\ \pageref{pf:globalsim}.
\end{proof}
}

We may further characterise the conditions under which \ref{global_sims} holds. We give our main such result below, and the rest in \Cref{sec:more_poly_stuff}.

\begin{restatable}{cor}{polyfactorchar}\label{cor:poly_factor_char}
   Suppose that $F:\Delta(A) \rightarrow \Delta(A)$ is a polynomial with no zeros on the simplex (i.e., none of the entries of $F(\pi_0)$ are zero for $\pi_0 \in \Delta(A))$. Then $F$ may be written as a distribution arising from samples of actions, as in equation (\ref{eq:global_samples}).
\end{restatable}
\excludeforjournal{
\begin{proof}
See \Cref{pf:polyfactorchar}, p.\ \pageref{pf:polyfactorchar}.
\end{proof}
}

\begin{eg}[continues= eg:SB_PD_formal]\label{eg:GSGT}
Consider now the second version of \Cref{eg:SB_PD}. The dependence function for this version involves a square root, which is not a polynomial. Thus \Cref{prop:global_sim} implies that this dependence function cannot be expressed as a function of action simulations, and so GSGT cannot be applied here.
\end{eg}

\subsection{Global approximation in the limit}\label{sec:limit_approx}

Perhaps we can \emph{approximate} the true dependence functions by functions of action simulations, even if we cannot express them precisely in this way. We could then look at the CDT+GSGT policies for a sequence of increasingly good approximations to the decision problem. Do these converge in some sense to some set of policies, and if so, does this include the \textit{ex ante} optimal policy? In this section, we explore these questions.

We approximate dependence functions $F$ by the distribution obtained by applying $F$ to the empirical distribution of some number of independent samples of the agent's policy. %
For instance, we might sample from the policy four times. If we sample actions $C$, $D$, $C$, $C$, we then apply $F$ to the empirical distribution $0.75C + 0.25D$.
We get that whenever $F$ is continuous, the approximation converges uniformly to $F$ as the number of samples goes to infinity.

\begin{restatable}{prop}{unifconvergence}\label{prop:unif_convergence}
    Suppose that $F:\Delta(A)\rightarrow\Delta(A)$ is continuous. Then there exists a sequence of functions $(g_N)_{N\in\mathbb{N}}$ with $g_N:A^N\rightarrow\Delta(A)$ such that setting $\Fhat_N(\pi_0) \defeq \E{g_N(A_1,\ldots,A_N)\mid A_1,\ldots,A_N \sim \pi_0}$, $\Fhat_N$ converges uniformly to $F$. That is, $\sup_{\pi_0} \norm{\hat{F}_N(\pi_0) - F(\pi_0)} \rightarrow 0 $ as $N\rightarrow\infty$. In particular, this holds for $g_N(a_1,\ldots,a_N) \defeq F(\frac{1}{N}\sum_{k=1}^N a_k)$, where the sum denotes a mixture.
\end{restatable}
\excludeforjournal{
\begin{proof}
See \Cref{pf:unifconvergence}, p.\ \pageref{pf:unifconvergence}.
\end{proof}
}

\begin{rk}
    Note that if $F$ is not continuous, we can't possibly write $F$ as the uniform limit of polynomials, by the uniform limit theorem. Thus, the condition in \Cref{prop:unif_convergence} is also necessary.
\end{rk}

In turn, if $F$ may be expressed as such a uniform limit, we have that the \textit{ex ante} optimal policies of the decision problems where $F$ is replaced with its approximation, converge to \textit{ex ante} optimal policies of the original problem:

\begin{restatable}{prop}{exantelimit}\label{prop:ex_ante_limit}
    Let $\F = (F_1, \ldots, F_n)$. Suppose that we have a sequence of functions $(\Fhat_N)$, with $\Fhat_N:\Delta(A)\rightarrow\Delta(A)^{n}$ continuous, converging uniformly to $\F$. Let $\Theta \subseteq \Delta(A)$ be the set of ex ante optimal policies in the original problem, and $\Theta_N$ the set of ex ante optimal policies in the problem with $\F$ replaced with $\Fhat_N$.
    
    Then $\Theta$ is non-empty and $\Theta_N$ is non-empty for all $N$, and the maximum distance of policies %
    in $\Theta_N$ to the nearest policy in $\Theta$ goes to 0 as $N$ goes to $\infty$. Formally, $
        \lim_{N\rightarrow \infty}\sup_{\pi \in \Theta_N} d(\pi, \Theta) =
        0$, where $d(x,Y) \defeq \inf_{y\in Y} \norm{x-y}_2$. Equivalently, any sequence $(\pi_{0,N})_{N\in \mathbb{N}}$ with $\pi_{0,N} \in \Theta_N$ has $d(\pi_{0,N},\Theta) \rightarrow 0$.
\end{restatable}
\excludeforjournal{
\begin{proof}
See \Cref{pf:exantelimit}, p.\ \pageref{pf:exantelimit}.
\end{proof}
}

Putting these two results together, we have that whenever the dependence functions are continuous, we may construct a sequence of increasingly good approximations to the decision problem in which we may apply GSGT. There is then a sequence of CDT+GSGT policies for these problems that converges to the set of %
\textit{ex ante} optimal policies of the original problem.

\begin{restatable}{cor}{CDTGTlimits}\label{cor:CDTGT_limits}
    Suppose that dependence function $\F:\Delta(A)\rightarrow\Delta(A)^n$ is continuous. Then there exists a sequence of GSGT compatible dependence functions $(\hat{\F}_N)$ uniformly converging to $\bm{F}$. For any such $(\hat{\F}_N)$, there then exists a sequence of policies $(\pi_0^N)$ such that: (i) for all $N$, we have that $\pi_0^N$ is a CDT+GSGT policy for the decision problem where $\F$ is replaced by $\hat{\F}_N$ and (ii) $(\pi_0^N)$ converges to the set of ex ante optimal policies for the original problem as $N\rightarrow\infty$.
\end{restatable}
\excludeforjournal{
\begin{proof}
See \Cref{pf:CDTGTlimits}, p.\ \pageref{pf:CDTGTlimits}.
\end{proof}
}

\begin{rk}
    Many policies $\pi_0$ are the limit of a sequence $(\pi_0^N)$ of CDT+GSGT policies for \emph{some} (potentially very odd) sequence of approximating functions $(\hat{\F}_N)$ (cf. \Cref{app:some_approx_weak}).
    By contrast, the \textit{ex ante} optimal policy is the limit of some sequence $(\pi_0^N)$ of CDT+GSGT policies for \emph{every} sequence of approximating dependence functions. While we do not give a characterisation of policies for which this latter condition holds, it appears likely to be far more restricted.
\end{rk}

\begin{eg}[continues = eg:SB_PD]\label{eg:glob_approx}
Let us now apply this approach to the second case of our example. Recall that we have $F_1(\pi_0) = 0.2\times(0.5,0.5)+0.8\pi_0$ and $F_2(\pi_0)= (\sqrt{0.1+0.8\pi_0(C)}, 1- \sqrt{0.1+0.8\pi_0(C)})$. First, let's compute the \textit{ex ante} optimal policy. We have that \textit{ex ante}, if Dorothea cooperates with probability $p$, the medicated version of her cooperates with probability $0.8p+0.1$, and the two copies of Theodora overall cooperate with probability $0.8p+0.1$. The \textit{ex ante} expected utility is then $(3-2)(0.8p+0.1)+2$, since Dorothea cooperating costs her 2, but Theodora cooperating gains Dorothea 3. %
 Thus, the unique \textit{ex ante} optimal policy is to cooperate deterministically.

Now, let's find a sequence of GSGT compatible dependence functions converging to the original dependence functions. $F_1$ may be directly expressed as arising from a single simulation, so we have no need to approximate it:
$
    F_1(\pi_0) = \E{(0.1,0.1)+0.8A_1 \mid A_1\sim \pi_0}.
$

But this is not possible for $F_2$. Instead, we may approximate $F_2$ via 
$
\Fhat^N_2(C\mid\pi_0) \defeq \E{F_2(C\mid(\hat{p}_N,1-\hat{p}_N))} = \E{\sqrt{0.1+0.8\hat{p}_N}},
$
where $\hat{p}_m \sim \frac{1}{m}\Bin(m,\pi_0(C))$ (where $\Bin$ refers to the binomial distribution, so $(\hat{p}_N,1-\hat{p}_N)$ is distributed as the empirical distribution of $N$ samples from $\pi_0$). \Cref{prop:unif_convergence} gives that $\Fhat^N_2 $ converges uniformly to $F_2$.

Now, we compute the CDT+GSGT policies for the problem where $F_2$ is replaced by $\Fhat_2^N$. When Dorothea is one of Theodora's $N$ simulations, the increase in probability of that copy of Theodora cooperating from Dorothea cooperating rather than defecting is given by $\E{\sqrt{0.1 + 0.8(\tfrac{N-1}{N}\hat{p}_{N-1}+\tfrac{1}{N})}-\sqrt{0.1 + 0.8(\tfrac{N-1}{N}\hat{p}_{N-1})}}$, which is in turn equal to 
$
 \frac{0.4}{N\sqrt{0.1+0.8\pi_0(C)}} + o(N^{-1}),
$
where $o(N^{-1})$ is little-o notation, denoting a term whose ratio to $N^{-1}$ tends to zero as $N\rightarrow \infty$.
For full details, see \Cref{sec:glob_approx_eg_deets}. Thus, the overall increase in probability of both copies of Theodora cooperating is then
$
 \frac{0.4}{N\sqrt{0.1+0.8\pi_0(C)}}\sqrt{0.1+0.8\pi_0(C)} + o(N^{-1}) = \frac{0.4}{N} + o(N^{-1})
$.
There are then $2N$ copies of Dorothea as simulations by Theodora, and 1 copy as a simulation by the medicated Dorothea. Hence, the CDT+GSGT utility of cooperating over defecting in the $N$-sample simulation model is given by $3\times\frac{2N}{2N+1}\times\frac{0.4}{N} - 2\times 0.8 \times \frac{1}{2N+1} + o(N^{-1}) = \frac{0.8}{2N+1}+ o(N^{-1})$, which is positive for sufficiently large $N$.

  So for sufficiently large $N$, the only CDT+GSGT policy is deterministically cooperating, the \textit{ex ante} optimal policy.
\end{eg}

\subsection{Local approximation}\label{sec:local_sims}

In \Cref{sec:glob_sims_model}, we tried to model \emph{all} dependence via simulations. However, we know from Theorem \ref{thm:CDT+GT} that the CDT+GT policies are determined by the \emph{local} properties of the \textit{ex ante} utility (i.e., its derivative). In particular, when all dependence functions are the identity, the CDT+GT counterfactuals for taking action $a$ are proportional to the the rate of change in the \textit{ex ante} expected utility under local policy deviations in direction $a$ \cite[cf.][Lemma 17]{dese}. Then, given a policy we expect to follow, we could construct a simulation model which accurately models the change in expected utility under local policy deviations via simulations. The resulting CDT+GT action counterfactuals should also reflect the effect of local deviations on the \textit{ex ante} expected utility. In particular, the \textit{ex ante} optimal policy should then be ratifiable (i.e., optimal given that the agent expects to act according to it). One way to view this is that the agent expects what simulations are run to depend only evidentially but not causally on its policy. Note that whilst this approach needn't guarantee reasonable \emph{policy} counterfactuals (i.e., beliefs about what would happen if the agent deviated to a whole different policy, far from its original policy), CDT does not need such beliefs -- rather, it needs reasonable action counterfactuals given its policy.

The idea here is an equilibrium-like concept (as with ratificationism in general). The agent's beliefs are relevantly correct (including its counterfactual beliefs) given its policy, and the policy is correct (i.e., optimal) given its beliefs. Recall that, as usual, we are looking for beliefs that \emph{justify} good policies, not simply for good policies themselves.

Can we then model the effects of small changes to the policy via simulations, without capturing all dependence in this way? As in section \Cref{sec:glob_sims_model}, we do not explicitly construct the simulation model here, but instead describe the self-locating beliefs corresponding to such a model.

\begin{defn}[Local-simulation-based self-locating beliefs]
Let $(S,S_T,P_0,n,i,A,T,u,\F)$ be a decision problem. Suppose that we can \emph{locally approximate} $\F$ via some number of samples of the agent's actions. That is, suppose that for all $j$ and $\pi_0$, there exists $N_j=N_j(\pi_0)$ and $g_j = g_j(\cdot\mid\pi_0):A^{N_j}\rightarrow \Delta(A)$ such that for $\pi_0'\in \Delta(A)$:
    \begin{align}\label{eq:local_sim}
        &F_j(\pi_0 + \eps(\pi_0' - \pi_0)) \\
        &= \E{g_j(A_{1:N_j}) \mid A_{1:N_j} \simiid \pi_0 + \eps(\pi_0' - \pi_0)}+ o(\eps).\nonumber
    \end{align}

    Then a set of \emph{local-simulation-based Generalised Thirding (LSGT)} beliefs is a set of generalised self-locating beliefs defined by:
    \begin{align*}
        P(s\mid \pi_0) &= \frac{N_{i(s)}(\pi_0)\E{\#(s)\mid \pi_0}}{\sum_{s'\in S-S_T}N_{i(s')}(\pi_0)\E{\#(s')\mid \pi_0}}\\
        \tau_j(a, \pi_0) &= \sum_{k=1}^{N_j}\frac{1}{N_j}\E{g_j(A_{1:N}\mid \pi_0)\mid A_k = a, A_{-k} \simiid \pi_0}
    \end{align*}
\end{defn}
Note that the LSGT beliefs are the same as GSGT beliefs, except that the $N_j$ and $g_j$s now may depend on $\pi_0$. Thus, GSGT beliefs are a special case of LSGT beliefs. LSGT beliefs correspond to applying GT to the model in which dependant $j$ runs $N_j(\pi_0)$ simulations of the agent, and uses these according to $g_j(\cdot\mid\pi_0)$.

When we can model the problem in this way, the LSGT+CDT policies indeed include the \textit{ex ante} optimal policy:

\begin{restatable}{thm}{localsimsresult}\label{thm:local_sims_result}
    Let $(S,S_T,P_0,n,i,A,T,u,\F)$ be a decision problem admitting LSGT beliefs. %
    Then the CDT+LSGT policies are precisely those $\pi_0$ for which $\frac{\partial}{\partial \pi_0(a)}\E{u\mid \pi_0} \leq 0$ for all $a$. Thus, they include the ex ante optimal policies.
\end{restatable}
\excludeforjournal{
\begin{proof}
See \Cref{pf:localsimsresult}, p.\ \pageref{pf:localsimsresult}.
\end{proof}
}

When can we model the dependence in this way? Exactly when the dependence function is differentiable:
\begin{restatable}{prop}{localsampleequiv}\label{prop:local_sample_equiv}
    We can locally approximate $f:\Delta(A)\rightarrow\Delta(A)$ at $\pi_0$ (as in \Cref{eq:local_sim}) if and only if $f$ is differentiable at $\pi_0$.
\end{restatable}
\excludeforjournal{
\begin{proof}
See \Cref{pf:localsampleequiv}, p.\ \pageref{pf:localsampleequiv}.
\end{proof}
}

\begin{eg}[continues= eg:SB_PD_formal]\label{eg:LSGT}
    Let us now apply LSGT to the second version of our example. Recall that our dependence functions were $F_1(\pi_0) = 0.2\times (0.5,0.5)+0.8\pi_0$ (for the medicated copy of Dorothea) and $F_2(\pi_0) = (\sqrt{0.1+0.8\pi_0(C)}, 1- \sqrt{0.1+0.8\pi_0(C)})$ (for Theodora). We have that the unique \textit{ex ante} optimal policy is to cooperate deterministically (cf.\ p.\ \pageref{eg:glob_approx}).

    Since both dependence functions are differentiable, we may apply LSGT to this decision problem. Let's verify that deterministically cooperating ($\pi_0 = (1,0)$) is a CDT+LSGT policy. In order to do this, we only need to form LSGT beliefs at this policy, so we begin by constructing a choice of $N_j$ and $g_j$ (for $j=1,2$) at $\pi_0 = C$. Now, for $j=1$, we may simply take $N_1 = 1$ and $g_1(a) = F_1(a)$, since $F_1$ is already linear. For $j=2$, consider taking $N_2 = 1$, $g_2(C) = F_2(C)$, and $g_2(D) = \tfrac{5}{9}F_2(C) + (0,\tfrac{4}{9})$. Note that to verify this satisfies \Cref{eq:local_sim}, we need only check with $\pi_0' = C$ or $\pi_0' = D$. For $\pi_0' = C$, this amounts to verifying $g_2(C) = F_2(C)$, which holds by construction. For $\pi_0' = D$, we instead need to check that $F_2((1-\eps)C + \eps D) = (1-\eps)g_2(C)+\eps g_2(D) + o(\eps)$. By differentiating, we obtain $F_2((1-\eps)C + \eps D) = F_2(C) + \eps(-\tfrac{0.4}{\sqrt{0.9}},\tfrac{0.4}{\sqrt{0.9}}) + o(\eps)$. Thus, it suffices to check $g_2(D)-g_2(C) = (-\tfrac{0.4}{\sqrt{0.9}},\tfrac{0.4}{\sqrt{0.9}})$, which may be easily verified.

    In order to apply LSGT, we now compute the LSGT probabilities and transformation functions. Since $N_1 = N_2 = 1$, the probabilities are the same as the GT probabilities would be: $\tfrac{1}{3}$ of being the medicated copy of Dorothea, $\tfrac{2}{3}$ of being at one the two awakenings of Theodora. Again since $N_1 = N_2 = 1$, the transformation functions are just $\tau_j(a) = g_j(a)$ (for $\pi_0 = C$).

    To check that $\pi_0 = C$ is a CDT+LSGT policy, it remains only to check that the CDT+LSGT utility of cooperating is at least as high as that of defecting, at $\pi_0 = C$. When we are at the state associated with Dorothea, the utility of cooperating is $2(\tau_2(C)_1-\tau_2(D)_1) = 2\times 0.8$ \emph{less} than that of defecting. When we are at a state associated with Theodora, the utility of cooperating is $3F_2(C)_1(\tau_2(C)_1-\tau_2(D)_1) = 3 \times \sqrt{0.9}\times \frac{0.4}{\sqrt{0.9}}=\tfrac{6}{5}$ \emph{more} than that of defecting. Putting this all together, the CDT+LSGT utility of cooperating at $\pi_0 = C$ is $\tfrac{2}{3}\times \tfrac{6}{5} - \tfrac{1}{3}\times\tfrac{8}{5} > 0$ more than defecting, so $\pi_0 = C$ is indeed a CDT+LGST policy.
\end{eg}

\section{Generalised Generalised Thirding}\label{sec:GGT}

In the previous section, we considered various ways to model the decision problem as involving explicit simulations of the agent, thereby enabling us to apply Generalised Thirding. We now introduce a (generalised) theory of self-locating beliefs that allows CDT to handle Newcomblike decision problems without using such models. We will call this theory of self-locating beliefs, which generalises Generalised Thirding, Generalised Generalised Thirding (GGT).

\subsection{Defining Generalised Generalised Thirding}\label{sec:defGGT}

\begin{restatable}[Equivalent definition of differentiability]{prop}{eqdiff}\label{prop:eq_diff}
    $F: \Delta(A) \rightarrow \Delta(A)$ is differentiable at $\pi_0$ if and only if there exists $\rho \geq 0$ and $\tau: A \rightarrow \Delta(A)$ such that for all $\pi_0'$:
    \begin{align}\label{eq:rho_tau_char}
        &F(\pi_0 + \eps(\pi_0' -\pi_0))\\
        &\quad= F(\pi_0) + \eps \rho \left(\sum_{a'} \tau(a')\pi_0'(a') - F(\pi_0)\right) + o(\eps).\nonumber
    \end{align}
\end{restatable}
\excludeforjournal{
\begin{proof}
See \Cref{pf:eqdiff}, p.\ \pageref{pf:eqdiff}.
\end{proof}
}

\begin{defn}[GGT weights and transformation functions]\label{defn:GGT_wts_transf}
Given a decision problem $(S,S_T,P_0,n,i,A,T,u,\F)$, a set of \textit{GGT weights and transformation functions} is a set of \emph{weights} $\rho_j:\Delta(A) \rightarrow \mathbb{R}_{\geq 0}$ and transformation functions $\tau_j:A\times \Delta(A) \rightarrow\Delta(A)$ for each $j = 1, \ldots, n$ satisfying, for all $\pi_0$:
\begin{align}\label{eq:rho_tau_char_j}
        F_j(\pi_0& + \eps(\pi_0' -\pi_0)) = F_j(\pi_0)\\
        &+\eps \rho_j(\pi_0) \left(\sum_{a} \tau_j(a, \pi_0)\pi_0'(a) - F_j(\pi_0)\right) + o(\eps)\nonumber.
    \end{align}
\end{defn}
By \Cref{prop:eq_diff}, a choice of GGT weights and transformation functions exists whenever $\F$ is differentiable everywhere. We will write $\tau(\pi_0',\pi_0)$ for $\sum_a \pi_0'(a)\tau(a,\pi_0)$, and $\tau(a'\mid a, \pi_0)$ for the $a'$th component of $\tau(a,\pi_0)$.

Intuitively, $\rho_j$ and $\tau_j$ are such that as we move from policy $\pi_0$ towards $a$ at some rate, $F_j$ moves away from its original value and towards $\tau_j(a,\pi_0)$ at $\rho_j(\pi_0)$ times that rate. GGT will function similarly to GT, except it will weight the probability of being dependant $j$ by $\rho_j(\pi_0)$, and will treat taking action $a$ when one is dependant $j$ as instead randomising according to policy $\tau_j(a,\pi_0)$. For an equivalent, more explicit/constructive, definition of the GGT weights and transformation functions, see \Cref{sec:GGT_alt_def}.

\begin{defn}[Generalised Generalised Thirding]\label{def:GGT}
    Given a decision problem with differentiable dependence function, and a set of GGT weights $(\rho_j)_{j=1}^n$, define the Generalised Generalised Thirding (GGT) credences as follows:
    \begin{equation*}
        P_{GGT}(s\mid \pi_0) \defeq \frac{\rho_{i(s)}(\pi_0) \E{\#(s)\mid \pi_0}}{\sum_{s'}\rho_{i(s')}(\pi_0) \E{\#(s')\mid \pi_0}} %
    \end{equation*}
    provided that the denominator is non-zero. If instead $\sum_j \rho_{j}(\pi_0)\E{\#(j)\mid \pi_0} = 0$, divide credence arbitrarily between states $s$ with $\E{\#(s)\mid \pi_0}>0$.
\end{defn}
    
    We can see this is similar to the usual Generalised Thirding credences (in the single observation case), but with each expected instance of dependant $j$ weighted according to $\rho_j$. This also bears striking similarity to LSGT/GSGT credences, as we discuss further in Section \ref{sec:relating_to_sims}.

\excludeforjournal{
\begin{rk}
    Note that there are infinitely many choices of $\rho_j$ for each decision problem, and thus infinitely many sets of GGT beliefs. Thus, for GGT to be a theory of self-locating beliefs, strictly speaking we need a map from decision problems to GGT weights. We will however only ever specify the GGT weights with respect to a given decision problem, since choices of weights for other problems will not be relevant.
\end{rk}
}

\subsection{Main result}\label{sec:main}
We are now ready to give a characterisation of CDT+GGT compatible policies, which gives as an immediate corollary that the \textit{ex ante} optimal policy is CDT+GGT compatible.
\begin{restatable}{thm}{thmmain}\label{thm:main}
    Suppose that $(S,S_T,P_0,n,i,A,T,u,\F)$ is a decision problem and that $\F$ is differentiable. Then, a policy $\pi_0$ is CDT+GGT compatible if and only if for all $a\in A$ we have $\frac{\partial}{\partial \pi_0(a)}\E{u\mid\pi_0} \leq 0$.
\end{restatable}
\excludeforjournal{
\begin{proof}
See \Cref{pf:thmmain}, p.\ \pageref{pf:thmmain}.
\end{proof}
}

\begin{cor}\label{cor:main}
    If the dependence function is differentiable, and $\pi_0^*$ is ex ante optimal, then $\pi_0^*$ is a CDT+GGT policy.
\end{cor}

Our proofs of these results use similar  main ideas to \citeauthor{piccione}'s proof of our \Cref{cor:CDT+GT}.

\begin{eg}[continues=eg:SB_PD]\label{eg:GGT}
We may now apply GGT to the second case of our example. Recall that we have $F_1(\pi_0) = 0.2\times(0.5,0.5)+0.8\pi_0$ and $F_2(\pi_0)= (\sqrt{0.1+0.8\pi_0(C)}, 1- \sqrt{0.1+0.8\pi_0(C)})$. As we calculated on p.\ \pageref{eg:glob_approx}, the unique \textit{ex ante} optimal policy is to cooperate deterministically.

Let's apply GGT to this example. Let $\rho_1=0.8$, and $\tau_1(a,\pi_0) = a$. We can immediately then see that $F_1(\pi_0 + \eps(\pi_0'-\pi_0)) = F_1(\pi_0) + 0.8\eps (\pi_0'-\pi_0) = F_1(\pi_0) + \eps\rho_1(\tau_1(\pi_0',\pi_0)-F_1(\pi_0))$, so this is a valid choice of GGT weight and transformation function for dependant 1 (the medicated version of Dorothea). Meanwhile, let $\rho_2=2$, and
\begin{equation*}
\tau_2(C,\pi_0) = \left(\frac{0.3+0.6\pi_0(C)}{\sqrt{0.8\pi_0(C)+0.1}},1-\frac{0.3+0.6\pi_0(C)}{\sqrt{0.8\pi_0(C)+0.1}}\right)
\end{equation*}
\begin{equation*}
\tau_2(D,\pi_0) = \left(\frac{0.1+0.6\pi_0(C)}{\sqrt{0.8\pi_0(C)+0.1}},1-\frac{0.1+0.6\pi_0(C)}{\sqrt{0.8\pi_0(C)+0.1}}\right).
\end{equation*}
We may verify that $\tau_2$ then always lies in $\Delta(A)$. Meanwhile, we have that
$
\sqrt{0.8(p+\eps)+0.1} = \sqrt{0.8p+0.1} + \frac{0.8}{2\sqrt{0.8p+0.1}}\eps + o(\eps).
$
From this, we may verify that $F_2(\pi_0+\eps(\pi_0'-\pi_0))=F_2(\pi_0)+\eps\rho_2(\tau_2(\pi_0',\pi_0)-F_2(\pi_0))+o(\eps)$. Applying GGT then, Dorothea has credence $\frac{0.8}{2\times2+0.8} = \frac{1}{6}$ in being the copy of herself, and $\frac{5}{6}$ in being one of the copies of Theodora. When she is herself, cooperating loses her an expected $2$ utility (recall that $\rho_1=0.8$, so $\tau_1$ is just the identity). When she is Theodora, cooperating gains her 3 utility when she causes Theodora to cooperate when the other copy is already cooperating, which happens with probability $(\tau_2(C\mid C,\pi_0)-\tau_2(C\mid D,\pi_0))F_2(C\mid\pi_0) = 0.2$. Thus, by GGT lights, cooperating gains Dorothea an expected $3\times0.2\times\frac{5}{6} - 2\times\frac{1}{6} = \frac{1}{6}$ utility.
\end{eg}

\subsection{Relating GGT to the simulation models}\label{sec:relating_to_sims}

In \Cref{sec:samples_model}, we discussed theories of self-locating beliefs based on writing (and later, locally approximating) $\F$ as the distribution arising from applying some function to independent samples of $\pi_0$.
How do these approaches relate to GGT?

We obtain that whenever it is possible to approximate $F_i$ locally as arising from $N_i(\pi_0)$ independent samples of the agent's policy, we may also apply GGT with $\rho_i = N_i$. Moreover, the GGT beliefs then exactly correspond to GT beliefs in the modified problem where dependant $i$ simulates copies of the agent (i.e., to LSGT beliefs). Thus, the GGT+CDT policies are exactly the LSGT+CDT policies, or equivalently, the GT+CDT policies in the modified problem with the simulations. The same holds in the special case where we can apply GSGT.

\begin{restatable}{prop}{localsamplemodel}
\label{prop:local_sample_model}
    Suppose that a decision problem admits LSGT (resp.\ GSGT) beliefs via $g_i$ and $N_i$, for $i=1,\ldots n$. Then we may apply GGT with $\rho_i = N_i$ for each $i$, and the resulting GGT beliefs are identical to the LSGT (resp.\  GSGT) beliefs.

\end{restatable}
\excludeforjournal{
\begin{proof}
See \Cref{pf:localsamplemodel}, p.\ \pageref{pf:localsamplemodel}.
\end{proof}
}

Recall that in order to apply GGT we must have that the dependence functions are differentiable, and hence can be written as functions of \textit{some} number of iid action samples (Proposition \ref{prop:local_sample_equiv}). Proposition \ref{prop:local_sample_model} then gives that the minimum number of simulations we can use for this is at least the minimum possible GGT weight. What about the converse -- is it the case that when we can use GGT weight $\rho_i$, we can always write $F_i$ as depending on $\rho_i$ expected simulations, thus giving that the minimum GGT weight is equal to the minimum number of simulations?

It is clear that we can do this when $\rho_i(\pi_0) \leq 1$: simply take a single simulation with probability $\rho_i$, and then apply $g_i = \tau_i$ to it. The same holds at deterministic policies $\pi_0 = a$. In such cases, we can simply take $g(A_1, \ldots A_N)$ to be $F(\pi_0)$ if all the $A_j$ are $a$, and $\tau(a',\pi_0)$ in the case all but one is $a$ and the other $a'$, and otherwise choose $g$ arbitrarily. However, in general the minimum number of simulations can exceed the minimum GGT weight:

\begin{restatable}{prop}{Nrho}\label{prop:N>rho}
    There exists $F$ differentiable on $\Delta(A)$, and $\pi_0 \in \Delta(A)$, for which for all $a\in A$,
    \begin{equation*}
        F(\pi_0+ \eps(a-\pi_0)) = F(\pi_0) + \eps\rho(\tau(a)-F(\pi_0)) + o (\eps)
    \end{equation*}
    with $\rho>0$ (and $\rho \in \mathbb{N}$) and $\tau:A\rightarrow \Delta(A)$, but where there exists no (potentially random) $N$ and $g_N:A^N \rightarrow \Delta(A)$ with $\E{N} \leq \rho$ forming a local approximation of $F$ at $\pi_0$.
\end{restatable}
\excludeforjournal{
\begin{proof}
See \Cref{pf:Nrho}, p.\ \pageref{pf:Nrho}.
\end{proof}
}

\section{Limitations}\label{sec:limitations}
For all our results, we made two substantial assumptions:
\begin{enumerate}[nolistsep]
    \item That the dependence function is continuous (for many results, also differentiable).
    \item That, given dependants' policies, they randomise independently each time they choose an action. For instance, in \Cref{eg:SB_PD}, given a policy for Theodora (say, $(1/2,1/2)$), we assume that her action on the first awakening is independent of her action on the second awakening.
\end{enumerate}
In \Cref{sec:indep_rand}, we show that we cannot drop the second of these assumptions. If we allow discontinuous dependence functions, is there any generalised theory of self-locating beliefs $X$ such that the \textit{ex ante} optimal policy is always a CDT+$X$ policy? The answer is no, provided that $X$ satisfy the following condition:

\begin{defn}[Faithful]
    Call a generalised theory of self-locating beliefs $X$ \emph{faithful} if whenever $F_j$ is the identity map for some dependant $j$ and dependant $j$ has positive ex ante probability under $\pi_0$ ($\E{\#(j)\mid \pi_0}> 0$), we have that $X$ assigns positive credence to being dependant $j$ (i.e., $P_X(j\mid \pi_0) > 0$ ), and $\tau_j(a,\pi_0)_a>\tau_j(a',\pi_0)_a$ for all $a\neq a'$.
\end{defn}

    Essentially, this says that whenever a decision problem contains an exact copy of the agent, which exists with positive probability, then $X$ has to assign positive probability to being said copy, and that when it is that copy, it should think that taking action $a$ makes it strictly more likely, compared to taking any other action $a'$, that the copy takes action $a$. We define one further property that we would like our theories to avoid:

\begin{defn}[Fanciful]
    Call a generalised theory of self-locating beliefs $X$ \emph{fanciful} if there exists a decision problem $D$ (on which $X$ is defined) such that for some policy $\pi_0$, there are non-terminal states $s$ with zero probability of occurring ex ante to which $X$ assigns non-zero credence (i.e.,$P_X(s\mid\pi_0)>0$).
\end{defn}

We have that GGT is faithful and not fanciful:
\begin{restatable}{prop}{GGTfaithful}\label{prop:GGT_faithful}
    For any (utility independent) function from decision problems with differentiable dependence functions to GGT weights, the resulting version of GGT is faithful and is not fanciful.
\end{restatable}
\excludeforjournal{
\begin{proof}
See \Cref{pf:GGTfaithful}, p.\ \pageref{pf:GGTfaithful}.
\end{proof}
}

However, when the dependence function may be discontinuous, we obtain the following impossibility result:
\begin{restatable}{thm}{faithfulimposs}\label{prop:faithful_imposs}
    There exists a decision problem for which the unique ex ante optimal policy is not a CDT+$X$ policy for any faithful non-fanciful theory of self-locating beliefs $X$.
\end{restatable}
\excludeforjournal{
\begin{proof}
See \Cref{pf:faithful_imposs}, p.\ \pageref{pf:faithful_imposs}.
\end{proof}
}

This result actually holds for a broader definition of self-locating beliefs -- cf.\ \Cref{sec:GSH}.

\bibliography{refs}

\begin{thebibliography}{40}
\providecommand{\natexlab}[1]{#1}
\providecommand{\url}[1]{\texttt{#1}}
\expandafter\ifx\csname urlstyle\endcsname\relax
  \providecommand{\doi}[1]{doi: #1}\else
  \providecommand{\doi}{doi: \begingroup \urlstyle{rm}\Url}\fi

\bibitem[Aaronson(2005)]{Aaronson2005}
Scott Aaronson.
\newblock Dude, it’s like you read my mind, 11 2005.
\newblock URL \url{https://scottaaronson.blog/?p=30}.

\bibitem[Ahmed(2014)]{Ahmed2014}
Arif Ahmed.
\newblock \emph{Evidence, Decision and Causality}.
\newblock Cambridge University Press, 2014.

\bibitem[Aumann et~al.(1997)Aumann, Hart, and Perry]{aumannhartperry}
Robert~J. Aumann, Sergiu Hart, and Motty Perry.
\newblock The absent-minded driver.
\newblock \emph{Games and Economic Behavior}, 20\penalty0 (1):\penalty0
  102--116, 1997.
\newblock ISSN 0899-8256.
\newblock \doi{https://doi.org/10.1006/game.1997.0577}.
\newblock URL
  \url{https://www.sciencedirect.com/science/article/pii/S0899825697905777}.

\bibitem[Barlassina and Gordon(2017)]{sep-folkpsych-simulation}
Luca Barlassina and Robert~M. Gordon.
\newblock {Folk Psychology as Mental Simulation}.
\newblock In Edward~N. Zalta, editor, \emph{The {Stanford} Encyclopedia of
  Philosophy}. Metaphysics Research Lab, Stanford University, {S}ummer 2017
  edition, 2017.

\bibitem[Bell et~al.(2021)Bell, Linsefors, Oesterheld, and Skalse]{GranCanaria}
James Bell, Linda Linsefors, Caspar Oesterheld, and Joar Skalse.
\newblock Reinforcement learning in {N}ewcomblike environments.
\newblock In M.~Ranzato, A.~Beygelzimer, Y.~Dauphin, P.S. Liang, and J.~Wortman
  Vaughan, editors, \emph{Advances in Neural Information Processing Systems},
  volume~34, pages 22146--22157. Curran Associates, Inc., 2021.

\bibitem[Bostrom(2010)]{Bostrom2010}
Nick Bostrom.
\newblock \emph{Anthropic Bias: Observation Selection Effects in Science and
  Philosophy}.
\newblock Studies in Philosophy. Routledge, 2010.

\bibitem[Brams(1975)]{Brams1975}
Steven~J. Brams.
\newblock {N}ewcomb's problem and prisoners' dilemma.
\newblock \emph{The Journal of Conflict Resolution}, 19\penalty0 (4):\penalty0
  596--612, 12 1975.

\bibitem[Briggs(2010)]{Briggs2010}
Rachael Briggs.
\newblock Putting a value on beauty.
\newblock In \emph{Oxford Studies in Epistemology}, volume~3, pages 3--24.
  Oxford University Press, 2010.

\bibitem[Carter and McCrea(1983)]{carter}
B.~Carter and W.~H. McCrea.
\newblock The anthropic principle and its implications for biological evolution
  [and discussion].
\newblock \emph{Philosophical Transactions of the Royal Society of London.
  Series A, Mathematical and Physical Sciences}, 310\penalty0 (1512):\penalty0
  347--363, 1983.
\newblock ISSN 00804614.
\newblock URL \url{http://www.jstor.org/stable/37419}.

\bibitem[Conitzer(2015)]{Conitzer_EDT_shoes}
Vincent Conitzer.
\newblock A dutch book against sleeping beauties who are evidential decision
  theorists.
\newblock \emph{Synthese}, 192\penalty0 (9):\penalty0 2887–2899, February
  2015.
\newblock ISSN 1573-0964.
\newblock \doi{10.1007/s11229-015-0691-7}.
\newblock URL \url{http://dx.doi.org/10.1007/s11229-015-0691-7}.

\bibitem[Dohrn(2015)]{Dohrn2015}
Daniel Dohrn.
\newblock {E}gan and agents: How evidential decision theory can deal with
  {E}gan's dilemma.
\newblock \emph{Synthese}, 192\penalty0 (6):\penalty0 1883--1908, 2015.

\bibitem[Drescher(2006)]{Drescher2006}
Gary~L. Drescher.
\newblock \emph{Good and Real -- Demystifying Paradoxes from Physics to
  Ethics}.
\newblock MIT Press, 2006.
\newblock URL
  \url{https://www.gwern.net/docs/statistics/decision/2006-drescher-goodandreal.pdf}.

\bibitem[Easwaran et~al.()Easwaran, Levinstein, and Shear]{EaswaranUnpublished}
Kenny Easwaran, Ben Levinstein, and Ted Shear.
\newblock Tickles, iteration, and habits.

\bibitem[Elga(2000)]{Elga2000}
Adam Elga.
\newblock Self-locating belief and the sleeping beauty problem.
\newblock \emph{Analysis}, 60\penalty0 (2):\penalty0 143--147, 2000.

\bibitem[Gauthier(1989)]{Gauthier1989}
David Gauthier.
\newblock In the neighbourhood of the {N}ewcomb-predictor (reflections on
  rationality).
\newblock In \emph{Proceedings of the Aristotelian Society, New Series,
  1988–1989}, volume~89, pages 179--194. [Aristotelian Society, Wiley], 1989.

\bibitem[Gibbard and Harper(1981)]{Gibbard1981}
Allan Gibbard and William~L. Harper.
\newblock Counterfactuals and two kinds of expected utility.
\newblock In William~L. Harper, Robert Stalnaker, and Glenn Pearce, editors,
  \emph{Ifs. Conditionals, Belief, Decision, Chance and Time}, volume~15 of
  \emph{The University of Western Ontario Series in Philosophy of Science. A
  Series of Books in Philosophy of Science, Methodology, Epistemology, Logic,
  History of Science, and Related Fields}, pages 153--190. Springer, 1981.
\newblock \doi{10.1007/978-94-009-9117-0_8}.

\bibitem[Hofstadter(1983)]{Hofstadter1983}
Douglas Hofstadter.
\newblock Dilemmas for superrational thinkers, leading up to a luring lottery.
\newblock \emph{Scientific America}, 248\penalty0 (6), 6 1983.

\bibitem[Horgan(1981)]{Horgan1981}
Terence Horgan.
\newblock Counterfactuals and {N}ewcomb's problem.
\newblock \emph{The Journal of Philosophy}, 78\penalty0 (6):\penalty0 331--356,
  6 1981.

\bibitem[Joyce(1999)]{Joyce1999}
James~M. Joyce.
\newblock \emph{The Foundations of Causal Decision Theory}.
\newblock Cambridge Studies in Probability, Induction, and Decision Theory.
  Cambridge University Press, 1999.

\bibitem[Levinstein and Soares(2020)]{Levinstein2020}
Benjamin~A. Levinstein and Nate Soares.
\newblock Cheating {D}eath in {D}amascus.
\newblock \emph{The Journal of Philosophy}, 117\penalty0 (5), 2020.
\newblock \doi{10.5840/jphil2020117516}.

\bibitem[Lewis(1979)]{Lewis1979}
David Lewis.
\newblock Prisoners' dilemma is a {N}ewcomb problem.
\newblock \emph{Philosophy \& Public Affairs}, 8\penalty0 (3):\penalty0
  235--240, 1979.

\bibitem[Lewis(1981)]{Lewis1981}
David Lewis.
\newblock Causal decision theory.
\newblock \emph{Australasian Journal of Philosophy}, 59\penalty0 (1):\penalty0
  5--30, 1981.

\bibitem[Lewis(2001)]{Lewis2001}
David Lewis.
\newblock Sleeping beauty: Reply to {E}lga.
\newblock \emph{Analysis}, 61\penalty0 (3):\penalty0 171--176, 2001.
\newblock ISSN 00032638, 14678284.
\newblock URL \url{http://www.jstor.org/stable/3329230}.

\bibitem[Meacham(2010)]{Meacham2010}
Christopher J.~G. Meacham.
\newblock Binding and its consequences.
\newblock \emph{Philosophical Studies}, 149\penalty0 (1):\penalty0 49--71,
  2010.
\newblock \doi{10.1007/s11098-010-9539-7}.

\bibitem[Neal(2006)]{Neal}
Radford~M. Neal.
\newblock Puzzles of anthropic reasoning resolved using full non-indexical
  conditioning, 2006.
\newblock URL \url{https://arxiv.org/pdf/math/0608592.pdf}.

\bibitem[Nozick(1969)]{Nozick}
Robert Nozick.
\newblock {N}ewcomb's problem and two principles of choice.
\newblock In Nicholas Rescher, editor, \emph{Essays in Honor of Carl G.
  Hempel}, pages 114--146. Reidel, 1969.

\bibitem[Oesterheld and Conitzer(2021)]{CDTDutchbook}
Caspar Oesterheld and Vincent Conitzer.
\newblock Extracting money from causal decision theorists.
\newblock \emph{The Philosophical Quarterly}, 71\penalty0 (4), 2021.
\newblock \doi{10.1093/pq/pqaa086}.

\bibitem[Oesterheld and Conitzer(2022)]{dese}
Caspar Oesterheld and Vincent Conitzer.
\newblock Can de se choice be ex ante reasonable in games of imperfect recall?
\newblock \url{https://www.andrew.cmu.edu/user/coesterh/DeSeVsExAnte.pdf},
  2022.

\bibitem[Pearl(2009)]{Pearl_2009}
Judea Pearl.
\newblock \emph{Causality}.
\newblock Cambridge University Press, 2 edition, 2009.

\bibitem[Piccione and Rubinstein(1997)]{piccione}
Michele Piccione and Ariel Rubinstein.
\newblock On the interpretation of decision problems with imperfect recall.
\newblock \emph{Games and Economic Behavior}, 20\penalty0 (1):\penalty0 3--24,
  1997.
\newblock \doi{10.1006/game.1997.0536}.

\bibitem[Poellinger(2013)]{Poellinger2013}
Roland Poellinger.
\newblock Unboxing the concepts in {N}ewcomb's paradox: Causation, prediction,
  decision.
\newblock 2013.
\newblock URL \url{http://philsci-archive.pitt.edu/9887/7/newcomb_in_ckps.pdf}.

\bibitem[Price(1986)]{Price1986}
Huw Price.
\newblock Against causal decision theory.
\newblock \emph{Synthese}, 67:\penalty0 195--212, 1986.

\bibitem[Price(2012)]{Price2012}
Huw Price.
\newblock Causation, chance, and the rational significance of supernatural
  evidence.
\newblock \emph{Philosophical Review}, 121\penalty0 (4):\penalty0 483--538,
  2012.

\bibitem[Skyrms(1982)]{Skyrms1982}
Brian Skyrms.
\newblock Causal decision theory.
\newblock \emph{The Journal of Philosophy,}, 79\penalty0 (11):\penalty0
  695--711, 1982.
\newblock \doi{10.2307/2026547}.

\bibitem[Spohn(2012)]{Spohn2012}
Wolfgang Spohn.
\newblock Reversing 30 years of discussion: why causal decision theorists
  should one-box.
\newblock \emph{Synthese}, 187\penalty0 (1):\penalty0 95--122, 2012.

\bibitem[Taylor(2016)]{Taylor2016}
Jessica Taylor.
\newblock In memoryless {C}artesian environments, every {UDT} policy is a
  {CDT}+{SIA} policy, 6 2016.
\newblock URL
  \url{https://www.alignmentforum.org/posts/5bd75cc58225bf06703751b2/in-memoryless-cartesian-environments-every-udt-policy-is-a}.

\bibitem[Wedgwood(2013)]{Wedgwood2013}
Ralph Wedgwood.
\newblock Gandalf's solution to the {N}ewcomb problem.
\newblock \emph{Synthese}, 190\penalty0 (14):\penalty0 2643--2675, 2013.
\newblock \doi{10.1007/s11229-011-9900-1}.

\bibitem[Weirich(2016)]{Weirich2016}
Paul Weirich.
\newblock Causal decision theory.
\newblock In \emph{The Stanford Encyclopedia of Philosophy}. Metaphysics
  Research Lab, Stanford University, spring 2016 edition, 2016.

\bibitem[Yudkowsky(2010)]{Yudkowsky2010}
Eliezer Yudkowsky.
\newblock Timeless decision theory, 2010.
\newblock URL \url{http://intelligence.org/files/TDT.pdf}.

\bibitem[Yudkowsky and Soares(2018)]{Yudkowsky2018}
Eliezer Yudkowsky and Nate Soares.
\newblock Functional decision theory: A new theory of instrumental rationality,
  5 2018.
\newblock URL \url{https://arxiv.org/abs/1710.05060v2}.

\end{thebibliography}

\onecolumn
\appendix

\section{Different notions of \textit{ex ante} expected utility}\label{sec:ex_ante_EU_rk}
\begin{rk}
\Cref{defn:ex_ante_EU} is, as mentioned, the EDT \textit{ex ante} expected utility. The CDT \textit{ex ante} expected utility can differ from the EDT \textit{ex ante} expected utility in cases where even from the \textit{ex ante} position taken, the evidential effects of one's actions are not fully captured via causal effects of one's actions. For instance, consider a version of Newcomb's problem (cf.\ Example \ref{eg:Newcombs}) in which the prediction is made based on the agent's genes (and perhaps the state of the universe at the point it was born). Ignoring questions of anthropic uncertainty, if the \textit{ex ante} perspective taken is at any time after the agent is born, comitting to one-box would have no \textit{causal} impact on the action of the predictor. Thus, the CDT \textit{ex ante} optimal policy would still be to two-box, while the EDT policy (\textit{ex ante} optimal or otherwise) would be to one-box, as usual. Additionally, the CDT \textit{ex ante} utility differs in this case depending on the precise point in time that we consider to be the \textit{ex ante} one -- one could in principle take an \textit{ex ante} perspective prior to the agent being born, and imagine that it can causally intervene on its own genes in order to determine whether they one-box or two-box. We neglect to define by what mechanism, causal or otherwise, the dependence occurs in our setting, and so we do not discuss the CDT \textit{ex ante} expected utility insofar as it differs from our definition of the (EDT) \textit{ex ante} expected utility. This is equivalent to the CDT \textit{ex ante} expected utility where we assume an \textit{ex ante} perspective such that all dependence is causal.
\end{rk}

\section{Action spaces}\label{sec:act_space_WLOG}
We assume that each dependant has the same action space as the agent. This is for simplicity -- our results would also go through if the dependants had different action spaces from the agent. It is also without loss of generality: First, note that if a dependants has at most as many actions as the agent, we may make them have the same number actions by adding actions that do the same thing as other actions, and relabelling the actions as needed. If instead a dependant has more actions than the agent (say, it has $k$ actions, $a_1,\ldots a_k$, and the agent only has $l<k$ actions) , we may split up each state associated with that agent into at most $(k+1)/l$ states, each associate with a dependant that has only two actions. The first state can distinguish between $a_1,\ldots a_{l-1}$ and $\{a_l,\ldots a_k\}$, the second between $a_l,\ldots a_{2l-1}$ and $\{a_{2l},\ldots a_k$, and so on. The dependence functions can then be constructed for each of the dependants such that it allocates appropriate overall probability to each of the $k$ actions.

\section{Multiple possible observations}\label{rk:obs}

Previous work (e.g. \cite{piccione} and \cite{dese}) considers different states being associated with different \textit{observations}, and has the agent distributing its credence over states with the same observations it has. We omit to consider different observations here, for the following reasons:
    \begin{itemize}
        \item We will be interested in theories of self-locating beliefs that don't restrict the agent to assigning positive credence only to states naturally associated with the same observations, for the following reasons: Firstly, when there is a discrepancy between the set of entities that the agent thinks it might be and the set of entities whose policies subjectively depend on that of the agent, this can prevent CDT from being compatible with the \textit{ex ante} optimal policy. For instance, this is the case in a twin Prisoner's Dilemma in which the twins have differently coloured shoes, say. (This can also be a problem for EDT, cf.\ \cite{Conitzer_EDT_shoes}). Moreover, we want to leave the mechanism of dependence open, rather than assuming from the outset that a dependant runs some known number of simulations of the agent with particular observations.
        \item In settings with simulation, there will often not be a unique observation associated with a particular state. For instance, an agent might be simulated multiple times with various different observations at different states. We could of course split states up such that each only has a single observation, but this would presuppose a simulation model.
        \item We could instead incorporate different observations for the agent by having different sets of dependence functions associated with different observations, and consider optimising the agent's policy over all possible observation states. In this case, holding fixed the agent's policy for other observation states, the problem for a given observation reduces to the single observation case. Thus, our results carry over straightforwardly to this multiple observation case. We discuss this in more detail below. For simplicity, we therefore restrict ourselves to the single observation case in the main text.
    \end{itemize}

    \subsection{Results for multiple possible observations}
    We could modify our setting to allow for observations as follows: First, we have a set $O$ of \textit{observations}. Then, we take the agent's policy to be in $\Delta(A)^{|O|}$, returning, for each observation $o\in O$, a distribution over actions $\pi_0(\cdot\mid o) \in \Delta(A)$. We then have $\pi_0 = (\pi_0(\cdot\mid 1), \ldots, \pi_0(\cdot\mid |O|))$. The dependence function $F$ is then a function from $\pi_0 \in \Delta(A)^{|O|}$ to $\bm{\pi}\in \Delta(A)^n$.

    Note that we don't need to consider the different dependants also having different observations, since we may WLOG just associate different observations with different dependants. Moreover, we don't need to specify a mapping from states to observations, for our purposes -- this has no impact on the \textit{ex ante} optimal policy, and will also have no impact on GGT: In non-Newcomblike environments, the \textit{ex ante} expected utility is affected by the mapping from states to observations because it is assumed that the agent's policy at a particular observation determines its actions at that observation, and not at any other observations. By contrast, in our setting, the dependence function determines how the agent's policy affects the actions taken at the various states. If the dependence function were given by $F_j(\pi_0) = \pi_0(\cdot\mid o(j))$, this would be equivalent to a non-Newcomblike setting in which the agent observes $o(i(s))$ at state $s$.

    GGT then generalises straightforwardly to this setting: we simply apply GGT separately for each possible observation, holding the policy at other observations fixed. That is, for each dependant $j$ and observation $o$, we choose $\rho_j(\pi_0,o)\geq 0$ and $\tau_j(\cdot,\pi_0, o):A\rightarrow\Delta(A)$ such that for all $o$, $\pi_0$, and $\pi_0'$ with $\pi_0'(\cdot\mid o')=\pi_0(\cdot\mid o')$ for $o'\neq o$, we have
    \begin{equation*}
        F_j(\pi_0+\eps(\pi_0'-\pi_0)) = F_j(\pi_0) + \eps\rho_j(\pi_0,o)\left(\sum_{a}\tau_j(a,\pi_0, o)\pi_0'(a\mid o)-F_j(\pi_0)\right) +o(\eps)
    \end{equation*}

    and $\tau_j(a,\pi_0,o) = F_j(\pi_0)$ whenever $\rho_j(\pi_0,o) = 0$.

    Then, we define GGT credences via transformation functions $\tau_j(\cdot,\pi_0,o)$ for each $o$, and

    \begin{equation*}
        P_{GGT}(s\mid \pi_0,o) = \frac{\rho_{i(s)}(\pi_0,o)\E{\#(s)\mid \pi_0}}{\sum_{s'}\rho_{i(s')}(\pi_0,o)\E{\#(s')\mid \pi_0}},
    \end{equation*}

    again dividing credence arbitrarily between states $s$ with $\E{\#(s)\mid \pi_0} >0$ when the denominator above would be zero.

    Similarly to before, we may then define the CDT+$X$ utility of choosing action $a$ given observation $o$ as

    \begin{equation*}
        \Ex_X[u\mid \Do(a),o,\pi_0] = \sum_s P_X(s\mid \pi_0, o) \sum_{s'} T(s'\mid s, \tau_{i(s)}(a,\pi_0,o))\Ex^{s'}[u\mid \pi_0].
    \end{equation*}
    Also, as before, write $\Ex_X[u\mid \Do(\pi_0'),o,\pi_0]$ for $\sum_a \pi_0'(a)\Ex_X[u\mid \Do(a),o,\pi_0]$.

    Finally, we may then say that $\pi_0$ is a CDT+$X$ policy if for each $o\in O$, we have $\Ex_X[u\mid \Do(\pi_0(\cdot\mid o)),\pi_0] = \max_{a} \Ex_X[u\mid \Do(a),\pi_0]$.

    Now, we obtain the following result in this setting with multiple possible observations:

    \begin{cor}
    Suppose that $(S,S_T,P_0,n,i,A,T,u,\F,O)$ is a decision problem with observation set $O$ and that the partial derivative of $\F$ with respect to $\pi_0(\cdot\mid o)$ exists for all $o\in O$. Then, a policy $\pi_0$ is CDT+GGT compatible if and only if for all $a\in A$ and $o\in O$ we have $\frac{\partial}{\partial \pi_0(a\mid o)}\E{u\mid\pi_0} \leq 0$.
\end{cor}
\begin{proof}
    Choose some GGT weights and transformation functions $\rho_j(\pi_0,o)$ and $\tau_j(\cdot,\pi_0,o)$, defined as above. Fix some policy $\pi_0\in \Delta(A)$, and some observation $o\in O$. Consider the decision problem where we hold $\pi_0(\cdot\mid o')$ fixed for $o'\neq o$, and consider varying $\pi_0(\cdot\mid o)$. Note that $\rho_j^o(\pi_0'(\cdot\mid o)) = \rho_j(\pi_0'(\cdot\mid o),\pi_0(\cdot\mid -o),o)$ and $\tau_j^o(a,\pi_0'(\cdot\mid o)) = \tau_j(a,\pi_0'(\cdot\mid o),\pi_0(\cdot\mid -o),o)$ are a valid choice of GGT weights and transformation functions for this decision problem.
    
    We may thus apply \Cref{thm:main} to this decision problem, to obtain that $\pi_0(\cdot\mid o)$ is a CDT+GGT policy for it if and only if $\frac{\partial}{\partial \pi_0(a\mid o)}\E{u\mid\pi_0} \leq 0$ for all $a$.

    Now, we may observe that $\pi_0'(\cdot\mid o)$ is a CDT+GGT policy with respect to this problem if and only if $\sum_s P_{GGT}(s\mid \pi_0'(\cdot\mid o))\Ex^s[u\mid \Do(\tau^o_{i(s)}(\pi_0'(\cdot\mid o),\pi_0'(\cdot\mid o))),\pi_0'(\cdot\mid o)] = \max_a\sum_s P_{GGT}(s\mid \pi_0'(\cdot\mid o))\Ex^s[u\mid \Do(\tau^o_{i(s)}(a,\pi_0'(\cdot\mid o))),\pi_0'(\cdot\mid o)]$. Applying this for $\pi_0' = \pi_0$, we then obtain that $\pi_0(\cdot\mid o)$ is a CDT+GGT policy for the single observation problem if and only if, in the original problem
    
    \begin{align*}
        &\sum_s P_{GGT}(s\mid \pi_0, o)\Ex^s[u\mid \Do(\tau_{i(s)}(\pi_0(\cdot\mid o),\pi_0,o)),\pi_0] = \max_a\sum_s P_{GGT}(s\mid \pi_0, o)\Ex^s[u\mid \Do(\tau_{i(s)}(a,\pi_0,o)),\pi_0],
    \end{align*}

    or equivalently, $\Ex_{GGT}[u \mid \Do(\pi_0(\cdot\mid o)), o,\pi_0] = \max_a \Ex_{GGT}[u \mid \Do(a), o,\pi_0]$.

    Putting everything together, $\frac{\partial}{\partial \pi_0(a\mid o)}\E{u\mid\pi_0} \leq 0$ for all $a$ if and only if $\Ex_{GGT}[u \mid \Do(\pi_0(\cdot\mid o)), o,\pi_0] = \max_a \Ex_{GGT}[u \mid \Do(a), o,\pi_0]$. Then, $\pi_0$ is a CDT+GGT policy if and only if this holds for every $o$, i.e., if and only if $\frac{\partial}{\partial \pi_0(a\mid o)}\E{u\mid\pi_0} \leq 0$ for all $a\in A$ and $o\in O$.
\end{proof}
 We obtain the same immediate corollary in this setting:
\begin{cor}
    If the dependence function has partial derivatives with respect to each $\pi_0(\cdot\mid o)$, and $\pi_0^*$ is ex ante optimal, then $\pi_0^*$ is a CDT+GGT policy.
\end{cor}

The simulation approaches generalise similarly -- we may simply construct simulation models for each possible observation, while holding the agent's policy at other observations fixed.

\section{Additional results for Section \ref{sec:glob_sims_model}}\label{sec:more_poly_stuff}

In this section, we give further results on when dependence functions may be written as arising from functions of action samples, as in \Cref{eq:global_samples}. Recall \Cref{prop:global_sim}:

\globalsim*

First, we obtain a sufficient condition for condition \ref{poly} from \Cref{prop:global_sim} to hold:
\begin{restatable}{lemma}{heatdiffusion}\label{lemma:heat_diffusion}
    Let $\Delta^{k-1}$ be the simplex in $\mathbb{R}^k$. Suppose that $f:\Delta^{k-1}\rightarrow \mathbb{R}$ is a polynomial satisfying $f(\p)>0$ for all $\p \in \Delta^{k-1}$. Then $f$ may be written as a polynomial with non-negative coefficients.
\end{restatable}
\begin{proof}
    See \Cref{pf:heat_diffusion}, p.\ \pageref{pf:heat_diffusion}.
\end{proof}
This gives the following immediate corollary:
\polyfactorchar*
Note that \Cref{cor:poly_factor_char} also goes through if $F$ could be written as a product of a polynomial with only non-negative coefficients and a polynomial with no zeros on the simplex.

In the case where there are only two possible actions, we obtain similar conditions that are also necessary. It is enough that the polynomial have no zeros on the interior of the simplex:

\begin{restatable}{cor}{Aequalstwo}\label{cor:A=2}
    Suppose $|A|=2$. Then $f:\Delta(A)\rightarrow\Delta(A)$ may be written as arising from samples of actions as in equation (\ref{eq:global_samples}) if and only if it is a polynomial such that either:
    \begin{itemize}
        \item $f(\pi_0)$ is in the interior of the simplex (i.e., is not $(1,0)$ or $(0,1)$) for all $\pi_0$ in the interior of the simplex.
        \item $f$ is a constant function, equal to either $(1,0)$ or $(0,1)$ on the entire simplex.
    \end{itemize}
    These conditions amount to: if the polynomial touches the edges of the simplex on the interior of the simplex, it must be constant.
\end{restatable}
\excludeforjournal{
\begin{proof}
See \Cref{pf:Aequalstwo}, p.\ \pageref{pf:Aequalstwo}.
\end{proof}
}
We can easily see why the above condition is necessary with an example:
\begin{eg}
    Suppose $f(p,1-p) = (4p(1-p),1-4p(1-p))$, as shown in \Cref{fig:4p(1-p)}. Can we write $f$ in the form $f(p,1-p) = \E{g(A_{1:N})\mid A_{1:N}\simiid (p,1-p)}$ for some $N$, $g$? In order to do so, we must have $g(0,0,\ldots,0) = (0,1)$, so that $g_2(0,0,\ldots,0) = 1 > 0$. But when $p=\tfrac{1}{2}$, we then have that with probability $p^N$, all $A_i=0$. Hence, we would need $f_2(\tfrac{1}{2}) \geq p^Ng_2(0,0,\ldots,0) > 0$, a contradiction.
    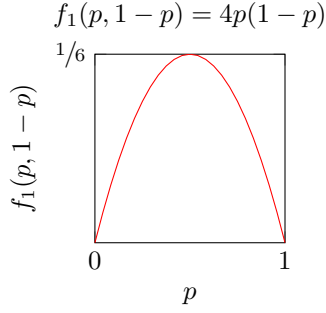
\begin{figure}[h]
\centering
\begin{tikzpicture}
\begin{axis}[
title={$f_1(p,1-p)=4p(1-p)$},
width = 2.5cm,
height = 2.5cm,
 xlabel={$p$},
ylabel={$f_1(p,1-p)$},
xmin=0, xmax=1,
ymin=0, ymax=1,
ytick={1},
yticklabels ={$\sfrac{1}{6}$,$\sfrac{5}{6}$,$1$},
xtick={0,1},
enlargelimits=false,
scale only axis=true
]
\addplot[color=red,domain=0:1]{4*x*(1-x)};

\end{axis}
\end{tikzpicture}
\caption{A graph of $4p(1-p)$.}\label{fig:4p(1-p)}
\end{figure}
\end{eg}

Now, the a generalisation of the condition in \Cref{cor:A=2} is also necessary for $|A|>2$: Whenever $f_i(p) = 0$, and $q$ is zero wherever $p$ is zero ($p_j = 0 \implies q_j=0$), then $f_i(q)= 0$. (Or equivalently, whenever $f_i$ is zero on the interior of some sub-simplex, $f_i$ must be zero on the whole sub-simplex). Is this sufficient? In fact, we can give a stronger necessary condition for $|A|>2$, which will show it is not:
\begin{restatable}{prop}{globsimsnecessary}\label{prop:glob_sims_necessary}
    Suppose that $f:\Delta(A)\rightarrow\Delta(A)$ may be written as in \Cref{eq:global_samples}. Then, for all $i$, for all $t\in(0,1]$, there exists some $\lambda(t) > 0$ such that whenever $p,q\in \Delta(A)$ are such that $p_j\geq t q_j$ for all $j$, $f_i(p)\geq \lambda(t) f_i(q)$.
\end{restatable}
\excludeforjournal{
\begin{proof}
See \Cref{pf:globsimsnecessary}, p.\ \pageref{pf:globsimsnecessary}.
\end{proof}
}

We may now construct an example such that whenever $f_i(p)=0$, and $q$ is such that $q_j=0$ whenever $p_j=0$, we have $f_i(q)=0$, but which does not satisfy the above condition.

\begin{eg}\label{eg:k>2notsuff}
    Let $f_1(\p) = ((p_1-\tfrac{1}{2})^2+\frac{1}{4}p_3^2)p_3$. Let $f_2(\p) = f_3(\p) = \frac{1-f_1(\p)}{2}.$ 
    
    First, we show that whenever $f_i(\p)=0$, and $\bm{q}$ satisfies $p_j = 0 \implies q_j=0$, we have $f_i(\bm{q})=0$: We have that $f_2$ and $f_3$ have no zeros, so here the condition trivially holds. Then, $f_1(p)$ is zero exactly when $p_3=0$. Hence, if $f_1(p)=0$ and $q_j=0$ whenever $p_j=0$, we have $p_3=0$, and so $q_3=0$, and so $f(q)= 0$.

    Now, we will show that there exists some $t>0$ such that for all $\lambda> 0$, there exist $\p$ and $\bm{q}$ with $p_j\geq tq_j$ for all $j$, but $f_1(\p)< \lambda f_1(\bm{q})$. Hence, by \Cref{prop:glob_sims_necessary}, we will deduce that $f_1$ cannot be expressed as arising from simulations of actions. Let $t=\frac{1}{3}$, and let $\p(\eps) = (\tfrac{1}{2},\tfrac{1}{2}-\eps,\eps)$, and $\bm{q}(\eps) = (0,1-\eps,\eps)$, for $0< \eps\leq\frac{1}{4}$, so that $\p\geq t\bm{q}$. Then $f_1(\p) = \frac{1}{4}\eps^3$, while $f_1(\bm{q})= \frac{1}{4}(1+\eps^2)\eps> \frac{1}{4}\eps$. Hence, $f_1(\p)/f_1(\bm{q}) < \eps^2\rightarrow0$ as $\eps\rightarrow0$.
\end{eg}

\section{Why many policies are the limit of CDT+GSGT policies for \emph{some} sequence of approximating dependence functions}\label{app:some_approx_weak}

We show in \Cref{cor:CDTGT_limits} that for \emph{any} sequence of GSGT compatible approximating functions $(\hat{\F}_N)$ to the dependence function, there is some sequence of CDT+GSGT policies $(\pihat_0^N)$ converging to the set of \emph{ex ante} optimal policies. It is less remarkable for a policy $\pi_0$ to be the limit of CDT+GSGT policies for \emph{some} sequence of GSGT compatible approximating function to the dependence function.

To see this, consider a decision problem with only one non-terminal state, two actions, and utility 0 for choosing action 0, utility 1 for choosing action 1. By \Cref{thm:global_sims_result}, any point $\pi_0^*$ at which the derivative of the \textit{ex ante} utility is zero (that is, any stationary point of the \textit{ex ante} utility) is a CDT+GSGT policy. Writing the dependence function as a function from $[0,1]$ to itself, the \textit{ex ante} utility is just $F(\pi_0)$. Thus, every stationary point of the dependence function is a CDT+GSGT policy. Then, by making our approximating dependence functions have small wiggles in the right places (without moving far from the original dependence function), we can make any policies we wish be CDT+GSGT policies of every approximating problem (see \Cref{fig:wiggles}).

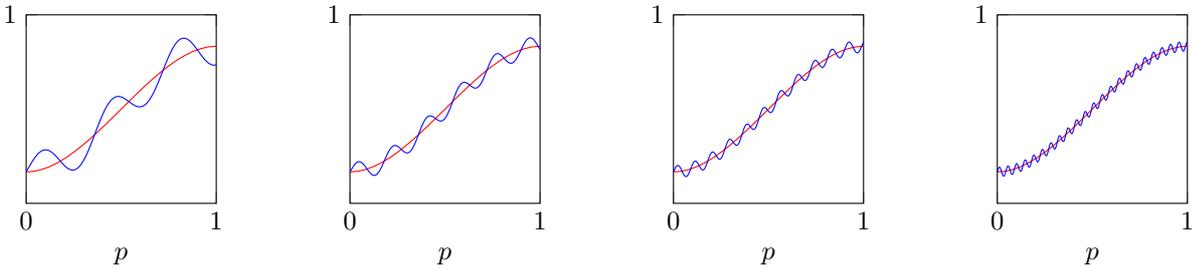
\begin{figure}[h]
    \centering
    \begin{minipage}{0.23\textwidth}
        \begin{tikzpicture}
            \begin{axis}[
                title={},
                width = 2.5cm,
                height = 2.5cm,
                xlabel={$p$},
                ylabel={},
                xmin=0, xmax=1,
                ymin=0, ymax=1,
                ytick={1},
                yticklabels ={$1$},
                xtick={0,1},
                enlargelimits=false,
                scale only axis=true
            ]
            \addplot[color=red,domain=0:1]{1/6+2*x^2 - 4*x^3/3};
            \addplot[color=blue,domain=0:1,samples = 1000]{1/6+2*x^2 - 4*x^3/3+sin(1000*x)/10};
            \end{axis}
        \end{tikzpicture}
    \end{minipage}
    \begin{minipage}{0.23\textwidth}
        \begin{tikzpicture}
            \begin{axis}[
                title={},
                width = 2.5cm,
                height = 2.5cm,
                xlabel={$p$},
                ylabel={},
                xmin=0, xmax=1,
                ymin=0, ymax=1,
                ytick={1},
                yticklabels ={$1$},
                xtick={0,1},
                enlargelimits=false,
                scale only axis=true
            ]
            \addplot[color=red,domain=0:1]{1/6+2*x^2 - 4*x^3/3};
            \addplot[color=blue,domain=0:1,samples = 1000]{1/6+2*x^2 - 4*x^3/3+sin(2000*x)/20};
            \end{axis}
        \end{tikzpicture}
    \end{minipage}
    \begin{minipage}{0.23\textwidth}
        \begin{tikzpicture}
            \begin{axis}[
                title={},
                width = 2.5cm,
                height = 2.5cm,
                xlabel={$p$},
                ylabel={},
                xmin=0, xmax=1,
                ymin=0, ymax=1,
                ytick={1},
                yticklabels ={$1$},
                xtick={0,1},
                enlargelimits=false,
                scale only axis=true
            ]
            \addplot[color=red,domain=0:1]{1/6+2*x^2 - 4*x^3/3};
            \addplot[color=blue,domain=0:1,samples = 1000]{1/6+2*x^2 - 4*x^3/3+sin(4000*x)/30};
            \end{axis}
        \end{tikzpicture}
    \end{minipage}
    \begin{minipage}{0.23\textwidth}
        \begin{tikzpicture}
            \begin{axis}[
                title={},
                width = 2.5cm,
                height = 2.5cm,
                xlabel={$p$},
                ylabel={},
                xmin=0, xmax=1,
                ymin=0, ymax=1,
                ytick={1},
                yticklabels ={$1$},
                xtick={0,1},
                enlargelimits=false,
                scale only axis=true
            ]
            \addplot[color=red,domain=0:1]{1/6+2*x^2 - 4*x^3/3};
            \addplot[color=blue,domain=0:1,samples = 1000]{1/6+2*x^2 - 4*x^3/3+sin(8000*x)/40};
            \end{axis}
        \end{tikzpicture}
    \end{minipage}
    \caption{In red, we show the original dependence function $F(p)$. In blue, we show four increasingly good approximations to it. While the approximations get close to the original function, they become increasingly wiggly, and therefore have increasingly many new stationary points (that aren't stationary points of the original dependence function). Each of these would be a CDT+GSGT policy in our example problem. Using this idea, we can create a sequene of approximating dependence functions such that every policy in $[0,1]$ is the limit of some sequence of CDT+GSGT policies for the decision problems where the dependence function is replaced by its approximations.}\label{fig:wiggles}
\end{figure}

We may show this formally as follows. For simplicity, assume that $F(\pi_0)\in(0,1)$ for all $\pi_0$. Suppose that we have a sequence of approximating dependence functions $(\hat{F}_N)$, which we may WLOG assume all lie in $[\delta, 1-\delta]$ for some $\delta>0$. Assume that the derivative of $\hat{F}_N$ (as a function from $[0,1]$ to $[0,1]$) has maximum magnitude at most $\lambda_N$, where WLOG $\lambda_N \geq 1$ (note that the derivative must be bounded since $\hat{F}_N$ is a polynomial by \Cref{prop:global_sim}). Now, consider $f_N(p) = \frac{1}{N} \sin(3N\lambda_N p)$, which has derivate $f'_N(p) = 3\lambda_N\cos(3N\lambda_N p)$. Then, we may find a polynomial, $q_N$, approximating $f_N$ and its derivative such that $\norm{q_N-f_N}_\infty< \tfrac{1}{N}$ and $\norm{q'_N-f'_N}_\infty< \lambda_N$. Then we have that the derivative of $q_N(p)$ at $p = 0, \tfrac{2\pi}{3N\lambda_N}, \tfrac{4\pi}{3N\lambda_N},\ldots$ is at least $2\lambda_N$, and its derivative at $p = \tfrac{\pi}{3N\lambda_N}, \tfrac{3\pi}{3N\lambda_N},\ldots$ is at most $-2\lambda_N$. Then consider $\tilde{F}_N = \Fhat_N + q_N$. We have that for $N$ large enough, $\tilde{F}_N$ lies in $(0,1)$, and is a polynomial, and hence is GSGT compatible by \Cref{cor:poly_factor_char}. Moreover, $\norm{\tilde{F}_N-\Fhat_N}_\infty = \norm{q_N}_\infty \leq \norm{f_N}_\infty + \norm{q_N-f_N}_\infty <\tfrac{2}{N} \rightarrow 0$ as $N\rightarrow \infty$. Hence, $\tilde{F}_N$ converges uniformly to $F$. Meanwhile, at $p = 0, \tfrac{2\pi}{3N\lambda_N}, \tfrac{4\pi}{3N\lambda_N},\ldots$, we have $\tilde{F}'_N(p) \geq 2\lambda_N - \norm{\Fhat_N'}_\infty \geq \lambda_N \geq 1$. Similarly, at $p = \tfrac{\pi}{3N\lambda_N}, \tfrac{3\pi}{3N\lambda_N},\ldots$, we have $\tilde{F}'_N(p)\leq -1$. Thus, since $\tilde{F}_N$ is a polynomial, the intermediate value theorem gives that it must have stationary points in the intervals $(0,\tfrac{2\pi}{3N \lambda_N}), (\tfrac{2\pi}{3N \lambda_N},\tfrac{4\pi}{3N \lambda_N}),\ldots$. Then, for any policy $\pi_0^*\in [0,1]$, we may find a sequence $(\pi_0^N)$ such that $\pi_0^N$ is a stationary point of $\tilde{F}_N$ such that $|\pi_0^N-\pi_0^*| \leq \tfrac{2\pi}{3N\lambda_N}$ (by choosing $\pi_0^N$ to lie in the interval containing $\pi_0^*$). Thus, $\pi_0^N\rightarrow \pi_0^*$, and for all $N$, we have that $\pi_0^N$ is a CDT+GSGT policy in the decision problem where $F$ is replaced with $\tilde{F}_N$.

More generally, if a point $\pi_0^*$ satisfies $\frac{\partial F_j(\pi_0^*)}{\partial\pi_0(a)} = 0$ for all $j$ and $a$, then we must have that $\frac{\partial \E{u\mid \pi_0^*}}{\partial\pi_0(a)} = 0$ for all $a$, by \Cref{cor:ex_ante_Lipschitzish} or the proof of \Cref{lemma:main}. Thus, by \Cref{thm:global_sims_result}, if $\hat{\F}_N$ has a stationary point at $\pi_0^*$ (i.e., its derivative is 0 at $\pi_0^*$), then $\pi_0^*$ is a CDT+GSGT policy for the decision problem where $\F$ is replaced by $\hat{\F}_N$.

\section{Alternative definition of GGT weights and transformation functions}\label{sec:GGT_alt_def}
\begin{notn}
    Provided $F_j$ is differentiable on the simplex, write
    $
        \delta_{j}(a, \pi_0) \defeq \frac{\partial F_j}{\partial\pi_0(a)}
    $
    and 
    $
        \delta_{j}(\pi_0', \pi_0) \defeq \sum_a\pi_0'(a)\frac{\partial F_j}{\partial\pi_0(a)}.
    $
    Also write
    $
        \delta_{j}(a' \mid a, \pi_0) \defeq \frac{\partial F_j(a' \mid \pi_0)}{\partial\pi_0(a)}
    $
    for the component of $\delta_{j}(a, \pi_0)$ in the direction of $a'$.
\end{notn}

\begin{restatable}[Equivalent definition of GGT weights and transformation functions]{lemma}{equivwtstransform}\label{lemma:equiv_wts_transf}
    Suppose $\F$ is differentiable everywhere.
    
    Let
    \begin{equation*}
        \Gamma_j(\pi_0) = \max_{a',a: F_j(a'\mid \pi_0) \neq 0} \frac{-\delta_j(a' \mid a, \pi_0)}{F_j(a'\mid \pi_0)}.
    \end{equation*}
    
    Then the GGT weights and transformation functions are exactly those $\rho_j$ and $\tau_j$ for which $\rho_j(\pi_0) \geq \Gamma_j(\pi_0)$ for all $\pi_0, j$, and for which $\tau_j(a, \pi_0)= F_j(\pi_0) + \frac{\delta_j(a, \pi_0)}{\rho_j(\pi_0)}$ when $\rho_j(\pi_0) \neq 0$, and $\tau_j(a,\pi_0) \in \Delta(A)$ when $\rho_j(\pi_0)=0$.
    
\end{restatable}
\begin{proof}
    This follows immediately from Lemma \ref{lemma:eq_diff}.
\end{proof}

\section{Why not model dependants as directly applying the dependence function?}\label{sec:N_sample_AO}
Why model dependant $j$ as taking some number of \textit{samples} of the agent's policy and applying some function to these, as opposed to just directly observing the agent's policy and applying $F_j$ to it?

Here, we will give an example in which our approaches work, but this approach would not work.

Consider the following version of the Adversarial Offer \citep{CDTDutchbook}:
\begin{itemize}
    \item An agent faces the following decision: Pay \$1 to take a single box — either box 1 or box 2 — or pay \$0 and take neither box.
    \item Prior to this, a predictor predicted which box the agent would take, and put \$3 in the other box, and nothing in the predicted box.
    \begin{itemize}
        \item So, if it predicted the agent would (deterministically) take box 1, it put \$3 in box 2, and \$0 in box 1.
        \item Similarly, if it predicted the agent would take box 2, it put \$3 in box 1, and nothing in box 2.
        \item If it predicted the agent would take neither box, it simply randomises which box it puts the \$3 in.
        \item Finally, if it predicts the agent will randomise between different actions, it puts \$0 in both boxes.
    \end{itemize}
    \item Specifically, assume here that the probability of the predictor predicting the agent will take a particular action $i$, given that the agent takes that action with probability $\pi_0(i)$, is given by $\pi_0(i)^4$. With the remaining probability, it predicts that the agent randomises.
\end{itemize}
Note that this is equivalent to the predictor running four simulations of the agent, observing only what actions the simulated agents take, and then if not all simulations take the same action, filling no boxes, and if all simulations do take the same action, filling boxes according to that action.

More formally, let $A= \{a_1,a_2,a_3\}$ be the set of actions, where $a_1$ corresponds to taking/filling box 1 for the agent/predictor (resp.), $a_2$ to taking/filling box 2, and $a_3$ to taking/filling neither box. Let dependant $1$ be the agent and dependant 2 be the predictor. Let $S=\{x_0,x_1,x_2,x_3\}\cup S_T$, where $S_T = \{0,-1,3\}$, where $x_0$ is the initial state ($P_0 = x_0$), at which the prediction is made, $x_1,x_2,x_3$ are states at which the game is played, given that box 1, 2 or neither box, respectively, was filled. The transition function is deterministic, with $T(x_0,a_i) = x_i$, $T(x_i,a_3) = 0$, $T(x_3,a_1) = T(x_3,a_2)=-1$, $T(x_1,a_2)=T(x_2,a_1)=3$, and $T(x_1,a_1) = T(x_2,a_2)=-1$. The transitions are shown in \Cref{fig:AO}. The index function is $i(x_0) = 2$, $i(x_j) = 1$. The utility function is simply given by $u(j) = j$.

\begin{figure}[h]
    \centering
	\begin{tikzpicture}[->, >=stealth', auto, semithick, node distance= 0.75cm and 2.75cm, on grid,
	terminal/.style ={},
	accepting/.style = {shape = rectangle},
	every state/.style = {fill=none,draw=black,thick,text=black},
	every text node part/.style={align=center}]
	\node[state,initial]    (s0) {$x_0$};
	\node[state]    (s1)[above right = 2cm and 3cm of s0] {$x_1$};
    \node[state]    (s3)[right = 3cm of s0]   {$x_3$};
	\node[state]    (s2)[below right = 2cm and 3cm of s0]   {$x_2$};
    \node[state,accepting]    (0)[right = 4cm of s3]{$0$};
    \node[state,accepting]    (-1)[right = 3cm of 0]{$-1$};
    \node[state,accepting]    (3)[right = 3cm of -1]{$3$};
    \node[state,accepting]    (-1b)[above right = 1 and 1cm of s3]{$-1$};
	\path
	(s0) edge[bend left, above left]	node{$a_1$}	(s1)
	(s0) edge[bend right, below left]		node{$a_2$}	(s2)
    (s0) edge[below]		node{$a_3$}	(s3)
    (s1) edge [bend left, below left] node{$a_3$} (0)
    (s3) edge [below] node{$a_3$} (0)
    (s2) edge [bend right, above left] node{$a_3$} (0)
    (s1) edge [bend left, below left] node{$a_1$} (-1)
    (s3) edge [bend left, above left] node{$a_1/a_2$} (-1b)
    (s2) edge [bend right, above left] node{$a_2$} (-1)
    (s1) edge [bend left, above right] node{$a_2$} (3)
    (s2) edge [bend right, below right] node{$a_1$} (3)
    ;
	\end{tikzpicture}
	\caption{A graph of (our version of) the adversarial offer, with states shown as nodes, and transitions shown as edges.}\label{fig:AO}
\end{figure}
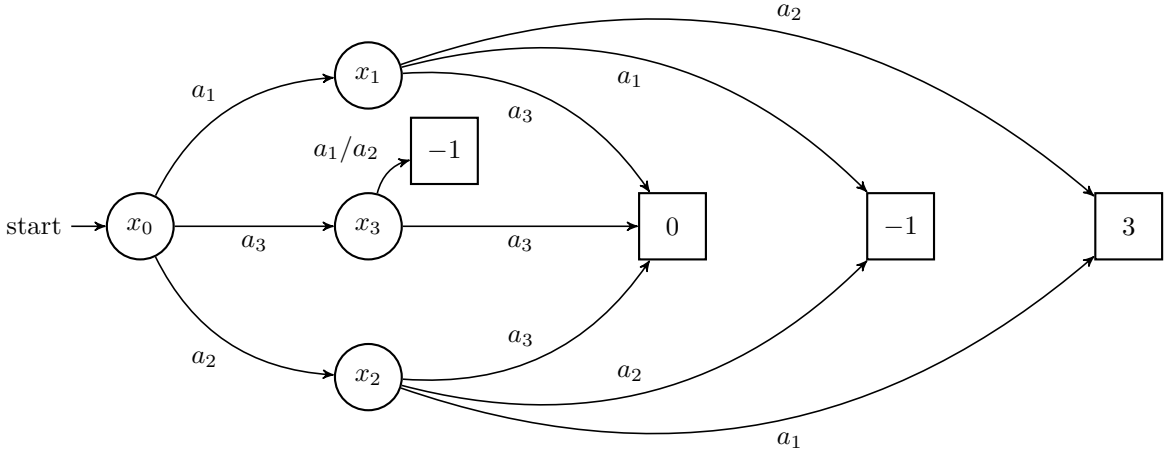

The dependence functions are $F_1(\pi_0)=\pi_0$, and $F_2$ defined as follows:

\begin{align*}
    F_2(\pi_0) &= (\pi_0(a_2)^4 + \frac{1}{2}\pi_0(a_3)^4)a_1 + (\pi_0(a_1)^4 + \frac{1}{2}\pi_0(a_3)^4)a_2 + (1- \pi_0(a_1)^4 - \pi_0(a_2)^4 - \pi_0(a_3)^4)a_3\\
    &= \pi_0(a_1)^4 a_2 + \pi_0(a_2)^4 a_1 + \pi_0(a_3)^4 \left(\tfrac{1}{2}a_1 + \tfrac{1}{2}a_2\right) + (1-\pi_0(a_1)^4-\pi_0(a_2)^4-\pi_0(a_3)^4)a_3\\
    &= \E{g(A_1,\ldots A_4)\mid A_1,\ldots,A_4 \simiid \pi_0 }
\end{align*}
where $g$ is defined as follows:
\begin{equation*}
    g(A_1,\ldots A_4) =
    \begin{cases*}
        a_2 & if $A_1 = A_2 = A_3 = A_4 = a_1$\\
        a_1 & if $A_1 = A_2 = A_3 = A_4 = a_2$\\
        \tfrac{1}{2}a_1 + \tfrac{1}{2}a_2 & if $A_1 = A_2 = \cdots = A_4 = a_3$\\
        a_3 & otherwise.
    \end{cases*}
\end{equation*}

Note that this is a simple decision problem (cf.\ \Cref{sec:GSH}), and so we will assume that theories of self-locating beliefs do not update about actions at other states from their existence.

What is the \textit{ex ante} optimal policy here? Parameterise the agent's policy by $p$, the probability of taking a box, and $\theta$, the probability of taking box 1 given that the agent takes a box. I.e., $\pi_0 = (\theta p, (1-\theta) p, 1-p)$. In any case, the cost of taking a box contributes expected utility $-p$. Then the predictor fills a random box with probability $(1-p)^4$, in which case the agent obtains an additional $\tfrac{3}{2}p$ expected utility. Similarly, the predictor fills box $1$ with probability $\theta^4p^4$ and box $2$ with probability $(1-\theta)^4p^4$, then giving additional expected utility $3(1-\theta)p$, or $3\theta p$, respectively. Otherwise, if the predictor fills no boxes, no additional utility is obtained. Thus, the overall expected utility is
\begin{align*}
    \E{u\mid \pi_0 = (\theta_1p,\theta_2p,1-p)} &= -p + \tfrac{3}{2}p(1-p)^4 + 3\theta^4p^4(1-\theta)p + 3(1-\theta)^4p^4\theta p\\
    &= -p + \tfrac{3}{2}p(1-p)^4 + 3p^{5}\theta(1-\theta)(\theta^{3}+(1-\theta)^{3}).
\end{align*}

Now, for any positive $p$, the maxima for $\theta$ are the maxima of $\theta(1-\theta)(\theta^3+(1-\theta)^3)$, which by differentiating and factorising, we may verify occur at $\theta = \tfrac{1}{2}-\tfrac{1}{2\sqrt{3}} \approx 0.21$ and $\theta = \tfrac{1}{2}+\tfrac{1}{2\sqrt{3}} \approx 0.79$. At this value of $\theta$, we then have $\theta(1-\theta)(\theta^3+(1-\theta)^3)= \tfrac{1}{12}$. Given this value for $\theta$, the \textit{ex ante} expected utility is $-p+\tfrac{3}{2}p(1-p)^4 + \tfrac{1}{4}p^5$, which is maximised near $p\approx 0.046$.

Now, suppose the agent instead assigns total credence $q_1$ to being the real agent, and its action applying as normal, and credence $q_2$ to being the simulated agent, in which case its policy is transformed directly according to $F_2$. What policies are then ratifiable? We have that if $\pi_0(a_1)>\pi_0(a_2)$, the strictly optimal policy as the simulated agent is then to take box 2 (deterministically). Meanwhile, if the policy attains positive \textit{ex ante} utility (as the \textit{ex ante} optimal policy does), then it must be strictly better to take either box 1 or box 2, as the real agent, than to take neither box (since randomising according to the prior probabilities gives the \textit{ex ante} utility, and by linearity). Hence, as the real agent, the strictly optimal policy is also to take box 2 deterministically. Thus, overall the unique optimal policy is to take box 2 deterministically. Thus, no policy with $\pi_0(a_1)\neq \pi_0(a_2)$ can be ratifiable, given these beliefs -- in particular, the \textit{ex ante} optimal policy cannot be ratifiable.

Meanwhile, we know that our approaches do work here (see \Cref{thm:global_sims_result,thm:main}). This shows that to achieve the \textit{ex ante} optimal policy we need to use some kind of transformation function — we can’t just directly use the dependance function.

\section{Constructing the simulation model for \Cref{eg:SB_PD}}\label{sec:glob_sims_eg_construction}
In the main text, we only sketched how a simulation model of the first version of \Cref{eg:SB_PD} would work. Here, we give a formal construction (for the same for general decision problems, see the proof of \Cref{thm:global_sims_result} on page \pageref{pf:globalsimsresult}, \Cref{pf:globalsimsresult}). Recall that we have dependence functions: $F_1(\pi_0) = \pi_0$ and $F_2(\pi_0) = (\tfrac{1}{6}+2\pi_0(C)^2 - \tfrac{4}{3}\pi_0(C)^3,\tfrac{5}{6}-2\pi_0(C)^2+ \tfrac{4}{3}\pi_0(C)^3)$. We again take, as a simulation model of $F_2$, $g(a_1,a_2,a_3) = (\tfrac{5}{6},\tfrac{1}{6})$ when at least two of $a_{1:3} = C$, and $(\tfrac{1}{6}, \tfrac{5}{6})$ otherwise.

We now construct the simulation model. Write $D$ for the original decision problem, and $D'$ for the new decision problem, in which all the dependence functions will be the identity. Construct $D'$ as follows:
\begin{itemize}[nolistsep]
    \item For each state $s$ in the original problem associated with Theodora ($\mathit{Th}_0, \mathit{Th}_C,\mathit{Th}_D$), `expand' state $s$ to create a depth-three tree of new states $s$, $(s,C)$, $(s,D)$, $(s,CC)$, $(s,CD)$, $(s,DC)$, $(s,DD)$. These are each simulations of Dorothea, and the labels keep track of the actions of the previous simulations at this awakening of Theodora (hence why we need more than three of these states per each state corresponding to Theodora). Leave the states associated with Dorothea, and the terminal states, unchanged.
    \item The initial distribution is just $P_0'((\mathit{Th}_0,0)) = 1$.
    \item The index function is just 1 everywhere (every state is now just a copy of Dorothea), and the dependence function is just the identity.
    \item The set of actions and the utility function is as before.
    \item The transition function $T'$ works as follows. If $s$ is one of the original states associated with Theodora, $T'(\cdot\mid s,a)$ maps deterministically to $(s,a)$, and $T'(\cdot\mid(s,a_1),a_2)$ maps deterministically to $(s,a_1a_2)$. $T'(\cdot\mid (s,a_1a_2),a_3)$ is instead distributed as $T(\cdot \mid s, g(a_1,a_2,a_3))$. In words, we progress through the tree of simulations keeping track of simulation actions for this awakening of Theodora, until all three simulations have been run. We then apply $g$ to the simulations, and transition to the next state as in the original problem, in accordance with $g$. For instance, $T'(\cdot\mid ((\mathit{Th_0},CD),C))$ is distributed as $T(\cdot \mid \mathit{Th_0}, g(C,D,C)) = T(\cdot \mid \mathit{Th_0}, \tfrac{1}{6}D + \tfrac{5}{6}D) = \tfrac{1}{6}\mathit{Th}_D + \tfrac{5}{6}\mathit{Th}_C$.
\end{itemize}
We illustrate the first part of the decision problem, with the `expanded' version of state $\mathit{Th}_0$ in figure \Cref{fig:expanded_state}.

\begin{figure}
\centering
\begin{tikzpicture}[->, >=stealth', auto, semithick, node distance= 0.75cm and 2.75cm, on grid,
	accepting/.style = {shape = rectangle, minimum size=0.5pt, inner sep = 3pt},
	every state/.style = {fill=none,draw=black,thick,text=black, minimum size = 0.5pt, scale=0.9, inner sep = 1pt},
	every text node part/.style={align=center}]
	\node[state,initial]    (Th0) {$\mathit{Th}_0$};
	\node[state]    (Th0C)[above right = 1.5cm and 2cm of Th0] {$\mathit{Th}_0,C$};
	\node[state]    (Th0D)[below right = 1.5cm and 2cm of Th0]   {$\mathit{Th}_0,D$};
    \node[state]    (Th0CC)[above right = 0.75cm and 2cm of Th0C] {$\mathit{Th}_0,CC$};
    \node[state]    (Th0CD)[below right = 0.75cm and 2cm of Th0C] {$\mathit{Th}_0,CD$};
    \node[state]    (Th0DC)[above right = 0.75cm and 2cm of Th0D] {$\mathit{Th}_0,DC$};
    \node[state]    (Th0DD)[below right = 0.75cm and 2cm of Th0D] {$\mathit{Th}_0,DD$};
	\path
	(Th0) edge[bend left, above left]	node[scale = 0.8]{$C$}	(Th0C)
	(Th0) edge[bend right, below left]		node[scale = 0.8]{$D$}	(Th0D)
	
	(Th0C) edge[bend left, above left]	node[scale = 0.8]{$C$}	(Th0CC)
	(Th0C) edge[bend right, below left]		node[scale = 0.8]{$D$}	(Th0CD)

 	(Th0D) edge[bend left, above left]	node[scale = 0.8]{$C$}	(Th0DC)
	(Th0D) edge[bend right, below left]		node[scale = 0.8]{$D$}	(Th0DD)
    ;
	\end{tikzpicture}
 \caption{A graph of the start of the simulation model of \Cref{eg:SB_PD}, showing the expansion of $\mathit{Th}_0$ into a tree of new states. The nodes show the states and the edges show the transitions. All transitions shown are deterministic.}\label{fig:expanded_state}
\end{figure}
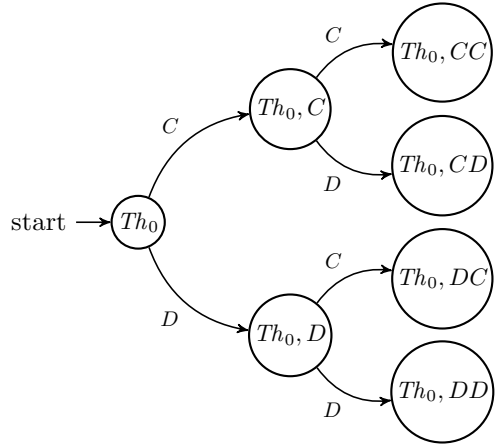

Since this new all dependence functions are the identity in this new decision problem, we may apply CDT+GT. It remains to show that the new decision problem is equivalent to the old one, and that CDT+GT on the new problem is equivalent to CDT+GSGT on the old one. In particular, we will show that:
\begin{enumerate}[nolistsep,label = (\roman*)]
    \item The \textit{ex ante} utility is the same (for all policies $\pi_0$) for the two problems.\label{item:ex_ante_same}
    \item The GSGT probability of state $s$ in $D$ is the same as the total GT probability of the entire expanded tree of states associated with $s$ in $D'$, given any policy $\pi_0$.\label{item:same_probs}
    \item The CDT+GSGT expected value of taking action $a$ in state $s$ in $D$ is the same as the average CDT+GT expected value of taking action $a$ in states associated with $s$ in $D'$, where the average is taken with states weighted by their GT probability in $D'$.\label{item:same_counterfactuals}
    \item The dynamics of the decision problems are the same, with states replaced by trees of states. I.e. for all states $s$ in the original problem $D$, the probability of transitioning directly from $s$ to $s'$ given policy $\pi_0$ is the same as the probability of transitioning from $s$ to $s'$ (via only states in the tree associated with $s$) given policy $\pi_0$ in $D'$.\label{item:dynamics_same}.
\end{enumerate}

Since the distribution over initial states is the same, it is easy to see that \ref{item:ex_ante_same} follows from \ref{item:dynamics_same}. We start by showing \ref{item:dynamics_same}. By construction of $T'$ and of the expanded state trees, we have that for each state $s$ associated with Theodora, starting at $s$, the chance of arriving at state $(s,a_1a_2)$ and then choosing action $a_3$ is given by $T'((s,a_1)\mid s, \pi_0)T'((s,a_1a_2)\mid (s,a_1),\pi_0)\pi_0(a_3) = \pi_0(a_1)\pi_0(a_2)\pi_0(a_3)$. The probability of then transitioning to state $s'$ is $T'(s'\mid (s,a_1a_2),a_3) = T(s'\mid s, g(a_1,a_2,a_3))$. Thus, the overall probability of transitioning from the $s$-tree to the $s'$ tree is $\sum_{a_{1:3}}\pi_0(a_1)\pi_0(a_2)\pi_0(a_3) T(s'\mid s, g(a_1,a_2,a_3)) = \E{T(s'\mid s, g(A_{1:3}))\mid A_{1:3} \simiid \pi_0} = T(s'\mid s, \E{g(A_{1:3})\mid A_{1:3}\simiid \pi_0} = T(s'\mid s, F_2(\pi_0))$. This is just the probability of transitioning from state $s$ to $s'$, so we are done.

We now turn to \ref{item:same_probs}. The number of expected copies of Dorothea in $D'$ is just $7$, since there are 7 instances of Dorothea in every history (6 as simulations, 1 as herself). Then, the number of expected copies of Dorothea in the tree of states associated with $\mathit{Th}_0$ is 3, since there are always exactly 3 such states (not always the same ones). Meanwhile, by the previous analysis, the probability of reaching state $\mathit{Th}_a$, for each $a$, is as in the original problem, and there are then always three states passed through in the associated tree. Thus, the expected number of copies of Dorothea in the tree of states associated with $\mathit{Th}_C$ is $3F_2(\pi_0)_1$ and for the tree associated with $\mathit{Th}_D$ is $3F_2(\pi_0)_2$. Thus, the GT probabilities of these trees of states are $\tfrac{3}{7}$, $\tfrac{3F_2(\pi_0)_1}{7}$ and $\tfrac{3F_2(\pi_0)_2}{7}$, the same as the GSGT probabilities of the corresponding states in the original problem. The states associated with Dorothea just have the same number of expected copies as in the original problem, and so similarly the GT probabilities of these states in the new problem are the same as the GSGT probabilities in the original problem. This is because the number of states passed through in the new problem (the GT denominator) is the same as the simulation-sample weighted number of states passed through in the original problem (the GSGT denominator).

Finally, we turn to \ref{item:same_counterfactuals}. For the states associated with Dorothea, this is easy to see. Consider then a state $s$ associated with Theodora. We have that the CDT+GT average expected utility of taking action $a$ in the GT-probability-weighted tree of states associated with $s$ in $D'$ is:
\begin{align*}
    &\sum_{k=0}^2 \frac{P_{GT}(s,a_{1:k}\mid\pi_0)}{\sum_{k',a_{1:k'}'}P_{GT}(s,a'_{1:k'}\mid\pi_0)}\Ex^{(s,a_{1:k})}_{D'}[u\mid \Do(a),\pi_0]\\
    &= \frac{1}{3}\Ex^s_{D'}[u\mid \Do(a),\pi_0] + \sum_{a_1}\frac{\pi_0(a_1)}{3}\Ex^{(s,a_1)}_{D'}[u\mid \Do(a),\pi_0]+\sum_{a_1,a_2}\frac{\pi_0(a_1)\pi_0(a_2)}{3}\Ex^{(s,a_1a_2)}_{D'}[u\mid \Do(a),\pi_0]
\end{align*}

Now, $\Ex_{GT,D'}^s[u\mid \Do(a),\pi_0] = \Ex_{D'}^{(s,a)}[u\mid \pi_0] = \sum_{a_2}^{(s,aa_2)}\pi_0(a_2)\Ex_{D'}[u\mid \pi_0] = \sum_{a_2,a_3}\pi_0(a_2)\pi_0(a_3)\sum_{s'}T'(s'\mid (s,aa_2),a_3)\Ex_{D'}^{s'}[u\mid \pi_0] = \sum_{a_2,a_3}\pi_0(a_2)\pi_0(a_3)\sum_{s'}T(s'\mid s,g(a,a_2,a_3))\Ex_{D'}^{s'}[u\mid \pi_0]$ which is just $\sum_{s'}T\left(s'\mid s, \E{g(a,A_{2:3}\mid A_{2:3}\simiid \pi_0}\right)\Ex_{D'}^{s'}[u\mid \pi_0]$. Similarly, $\sum_{a_1}\pi_0(a_1)\Ex_{D'}^{(s,a_1)}[u\mid \Do(a),\pi_0] = \sum_{a_1,a_3}\pi_0(a_1)\pi_0(a_3)\sum_{s'}T(s'\mid s,g(a_1,a,a_3))\Ex_{D'}^{s'}[u\mid \pi_0] =\sum_{s'}T\left(s'\mid s, \E{g(A_{1},a,A_{3}\mid A_{1},A_3\simiid \pi_0}\right)\Ex_{D'}^{s'}[u\mid \pi_0]$ and $\sum_{a_1,a_2}{\pi_0(a_1)\pi_0(a_2)}\Ex_{D'}^{(s,a_1a_2)}[u\mid \Do(a),\pi_0] = \sum_{s'}T\left(s'\mid s, \E{g(A_{2:3},a\mid A_{2:3}\simiid \pi_0}\right)\Ex_{D'}^{s'}[u\mid \pi_0].$

Thus, the CDT+GT average expected utility of taking action $a$ in the GT-probability-weighted tree associated with $s$ in $D'$ is just $\sum_{s'}T\left(s'\mid s, \tfrac{1}{3} \sum_{k=1}^3 \E{g(A_{1:3})\mid A_k = a, A_{-k}\simiid \pi_0}\right)\Ex_{D'}^{s'}[u\mid \pi_0] = \sum_{s'}T\left(s'\mid s, \tau_{GSGT,2}(a,\pi_0)\right)\Ex_{D'}^{s'}[u\mid \pi_0] = \Ex_{GSGT,D}^{s}[u\mid \Do(a),\pi_0]$, and we are done.
\section{Independent randomisation}\label{sec:indep_rand}

We assume that given dependants' policies, actions are sampled at random. Can we relax this assumption?

First note that if a CDT agent cannot randomise at all, even in the simpler setting of \citet[Section 6.1]{dese}, there may exist no ratifiable policies, or only Dutch book policies.

What if we allow randomisation, but with dependence? We still run into issues, as the following example shows:

\begin{eg}\label{eg:AO_no_random}
Consider a variant of the adversarial offer \citep{CDTDutchbook}: An agent faces a choice between paying \$1 to take a box (either box 1 or box 2), or taking neither box. Prior to this, a prediction of the agent is made, assumed here to be via a simulation of the agent. If that simulation takes box 1 (resp. 2), the predictor puts $\$3$ in box 2 (resp. 1), and box nothing in box 1 (resp. 2). Otherwise, if the simulated agent takes no boxes, the predictor randomises which box contains $\$3$. The simulated agent is just a copy of the real agent, and in the same observation state, so both choose the same policy. If the two instances of the agent can independently randomise, then the \textit{ex ante} optimal policy is to randomise between taking each box, which is the only CDT+GT policy. Suppose then that the simulated agent and real agent can't randomise independently from one another. The unique \textit{ex ante} optimal policy is then to take no boxes. However, whatever the agent believes the simulated agent does, taking whichever box the simulated agent is less likely to take (if either) dominates taking no boxes as a policy for the real agent. Meanwhile, if the agent believes that the real agent takes no boxes (in accordance with \Cref{asm:zeroimposs}), it is neutral between all policies. Thus, provided that the transformation function for the real agent is the identity, and the agent assigns non-zero credence to being this agent, it then prefers to deterministically take a box.
\end{eg}
So, it seems that the assumption of independent randomisation is necessary.
\subsection{CDT without ratificationism}

What if we give up hope of a ratifiable optimal policy, and instead ask whether we can find a policy where under \textit{some} prior beliefs about its actions, CDT can choose the \textit{ex ante} optimal policy?

In \Cref{eg:AO_no_random}, this still does not help us, at least under the assumption that the agent does not update from its existence about the actions of the other copy of it, and assuming that the transformation function for exact copies of the agent is the identity. Indeed, suppose that the agent believes \textit{a priori} that it will choose box 1 with probability $p_1$ and box 2 with probability $p_2$, WLOG $p_1\geq p_2$. Now, if the agent knew it were the simulated agent, it is either neutral between all actions (in the case $p_1 = p_2$), or strictly prefers to take box 2. Meanwhile, if it knew it were the real agent, it either strictly prefers to take box 2, or is neutral about which box it takes but strictly prefers taking a box over not taking a box. Thus, if it assigns non-zero credence to being the real agent, it will prefer to take a box.

To conclude, when randomisation is not possible, even with modified anthropics, there may be no beliefs for which the \textit{ex ante} optimal policy is recommended.

\section{Extending impossibility results to single-halfing and similar theories}\label{sec:GSH}
\subsection{Preliminaries}

Previously, we defined the CDT+$X$ expected utility of taking action $a$ as ${\Ex_X[u\mid \Do(a),\pi_0]} \defeq\sum_{s} {P_X(s\mid\pi_0)}\sum_{s'}{T(s'\mid s,\tau_{i(s)}(a,\pi_0))}\Ex^{s'}[u\mid\pi_0]$. This assumes that, given we believe we are in state $s$, we make no further anthropic updates about future events (in other words, the update we make from our existance is screened off by the knowledge of our state). This is in contrast to some theories of self-locating beliefs -- most notably Generalised Single Halfing (GSH)\footnote{Also known as the `self-sampling assumption' \cite{Bostrom2010}, GSH is implicitly assumed in the \textit{doomsday argument} \cite{carter}. GSH is described in the context of Sleeping Beauty in \cite{Lewis2001}, where it is the assumed halfer position (Generalised Double Halfing is not considered).} which does, even given its current state, update from its observations about future events if the future proportion of agents in its reference class with its observations is stochastic (cf.\ \citet{Lewis2001} and \citet[Sec 5.2]{dese}). For instance, in the Sleeping Beauty problem, GSH may think that the coin is more likely to come up Heads than Tails, from the perspective of before the coin toss (it reasons that since there are more future copies of Beauty with different observations if the coin comes up Tails, it would be less likely to find itself before the coin toss in this case, and hence updates against the Tails world from this observation).

In this section, we restrict our attention to decision problems in which the sequence of dependants that exists is fixed, regardless of any actions. This then removes any reason for theories of self-locating beliefs to update about the actions of dependants at other states (even if the agent can observe which dependant it is). We will then still allow for theories that make such updates, but merely assume that any such theory $X$ does not assign positive credence to dependants taking actions that they take with \emph{ex ante} probability 0. More precisely, we restrict attention to the following set of decision problems:
\begin{defn}[Simple]
    Call a decision problem $(S,S_T,P_0,n,i,A,T,u,\F)$ \emph{simple} if all of the following holds:
    \begin{enumerate*}[leftmargin=*,label=(\roman*)]
        \item The transition function is deterministic.
        \item $P_0$ is deterministic.
        \item The set of non-terminal states $S-S_T$ may be partitioned into sets $S_1,S_2,\ldots,S_t$ such that: the index function $i$ is constant on each set $S_j$; $P_0$ is zero except on $S_1$; for all $j= 1, \ldots, t-1$, $s\in S_j$ and $a\in A$, we have $T(\cdot\mid s,a)\in S_{j+1}$; and for all $s\in S_t$ and $a \in A$, we have $T(\cdot\mid s,a)\in S_T$.
    \end{enumerate*}
\end{defn}

Then, given a simple decision problem, we will consider generalised theories of self-locating beliefs that assign credences not only to states, but to states and future actions.
\begin{defn}[Self-locating beliefs (\Cref{sec:GSH} version)]\label{def:limitationsSLB}
    For the purposes of this section, given a simple decision problem with state partition $S_1,\ldots,S_t$, and agent policy $\pi_0$, \emph{self-locating beliefs} are probability distributions $P$ over non-terminal states and actions of future dependants. I.e., $P(s,a_{(j+1):t})$ is defined whenever $s\in S_j$ and $a_j,\ldots a_t \in A$, and $P$ defines a probability distribution over such tuples $(s,a_{(j+1):t})$.
\end{defn}

\begin{defn}[Generalised theory of self-locating beliefs (\Cref{sec:GSH} version)]\label{def:limitationsGTSLB}
    For this section, a \emph{(generalised) theory of self-locating beliefs} is a map from (some subset of) simple decision problems $D$ and policies $\pi_0$ to self-locating beliefs (\Cref{def:limitationsSLB}) and sets of transformation functions, independent of the utility function in $D$.
\end{defn}

\begin{notn}
    Given such a theory of self-locating beliefs, write $P_X(s\mid \pi_0) \defeq \sum_{a_{j+1},\ldots,a_t}P_X(s,a_{(j+1):t}\mid\pi_0)$. Also, write $P_X(j\mid \pi_0) = \sum_{s:i(s)=j}P_X(s\mid \pi_0)$.
\end{notn}

\begin{defn}[CDT+X utility (\Cref{sec:GSH} version)]
In this section, given a simple decision problem with non-terminal state partition $S_j,\ldots S_t$, and theory of self-locating beliefs $X$ (\Cref{def:limitationsGTSLB}), assigning credence $P_X(s,a_{(j+1):t})$ to state $s$ and future actions $a_{(j+1):t}$, the CDT+X expected utility is:
\begin{align*}
    &\Ex_X[u\mid \Do(a), \pi_0]\\
    &\defeq \sum_{\substack{j = 1,\ldots, t\\s\in S_j\\a_{(j+1):t}}}P_X(s,a_{(j+1):t}\mid \pi_0)\Ex^{T(\cdot\mid s,\tau_{i(s)}(a,\pi_0))}[u\mid a_{(j+1):t}]
\end{align*}
where $\Ex^s[u\mid a_{(j+1):t}]$ denotes the expected utility starting at state $s\in S_j$ and acting at future states according to action sequence $a_{(j+1):t}$.
\end{defn}

Note that, on simple decision problems, these are generalisations of our previous definitions.

Now, we will impose the following restriction, which is a generalisation of non-fancifulness (and is satisfied by any non-fanciful generalised theory of self-locating beliefs as defined in \ref{sec:setting}):
\begin{asm}\label{asm:zeroimposs}
     If $P_X(s,a_{(j+1):t}\mid \pi_0)>0$, we must have that state $s$ is reached with positive probability given policy $\pi_0$, and that actions $a_{j+1},\ldots a_t$ are taken by dependants $i(S_{j+1}),\ldots i(S_t)$ with positive probability given policy $\pi_0$.
\end{asm}
This essentially says that $X$ cannot assign positive credence to dependants at other states (past or future) taking actions that they take with $\textit{ex ante}$ probability 0.

Finally, we may additionally want to consider transformation functions $\tau_j(\cdot,\pi_0)$ taking as input a policy in $\Delta(A)$, rather than just an action, and transform it possibly non-linearly. We will refer to such transformation functions as \textit{non-linear transformation functions}, and state explicitly where we allow them.

\subsection{Results}
We obtain similar results to \Cref{sec:limitations}:

\begin{restatable}{prop}{faithfulimpossgeneral}\label{prop:faithful_imposs_general}
    There exists a simple decision problem for which the unique ex ante optimal policy is not a CDT+$X$ policy for any faithful theory of self-locating beliefs $X$ (possibly with non-linear transformation functions) satisfying \Cref{asm:zeroimposs}.
\end{restatable}
\begin{proof}
The proof is essentially identical to that of \Cref{prop:faithful_imposs}, noting the decision problem given is simple and that the arguments there go through due to \Cref{asm:zeroimposs}.
\end{proof}

Finally, we can consider dropping even \Cref{asm:zeroimposs} (but disallowing non-linear transformation functions). We then obtain the following result:

\begin{restatable}{prop}{pragmaticimposs}\label{prop:pragmaticimposs}
    Suppose $X$ is a generalised theory of self locating beliefs (\Cref{def:limitationsGTSLB}) defined on all simple decision problems. Then there exists a sequence of simple decision problems $(D_n)$ where the dependence functions converge uniformly to the identity as $n\rightarrow \infty$, and where the ex ante expected utility is non-constant on the interior of the simplex in $D_n$ for all $n$, for which all policies in the interior of the simplex are CDT+$X$ policies for $D_n$.
\end{restatable}
\begin{proof}
See \Cref{pf:pragmaticimposs}, p.\ \pageref{pf:pragmaticimposs}.
\end{proof}

\section{Details for examples}
\subsection{Case \ref{case:theodorosimple} of Example \ref{eg:GSGT}}\label{sec:ThDo_1_details}
We have, using $\pi_0(C)+\pi_0(D)=1$:
\begin{align*}
    F_2(\pi_0)_1 &= \tfrac{1}{6}+2\pi_0(C)^2 - \tfrac{4}{3}\pi_0(C)^3\\
    &= \tfrac{1}{6}(\pi_0(C)+\pi_0(D))^3 +2\pi_0(C)^2(\pi_0(C)+\pi_0(D)) -\tfrac{4}{3}\pi_0(C)^3\\
    &= (\tfrac{1}{6}+2 -\tfrac{4}{3})\pi_0(C)^3 + (\tfrac{1}{2}+2)\pi_0(C)^2\pi_0(D) + \tfrac{1}{2}\pi_0(C)\pi_0(D)^2 + \tfrac{1}{6}\pi_0(D)^3\\
    &= \tfrac{5}{6}\pi_0(C)^3 + 3\times \tfrac{5}{6}\pi_0(C)^2\pi_0(D)+3\times\tfrac{1}{6}\pi_0(C)\pi_0(D)^2 + \tfrac{1}{6}\pi_0(D)^3\\
    &= \tfrac{1}{6}\prob{\text{At least two of } A_{1:3} = D \mid A_{1:3}\simiid \pi_0} + \tfrac{5}{6}\prob{\text{At least two of } A_{1:3} = C \mid A_{1:3}\simiid \pi_0}
\end{align*}

Computation of $\tau_2$:
\begin{align*}
    \tau_2(C,\pi_0) &= \E{g(C,A_2,A_3)\mid A_{2:3} \simiid \pi_0} = \binom{\sfrac{1}{6}}{\sfrac{5}{6}}\pi_0(D)^2 + \binom{\sfrac{5}{6}}{\sfrac{1}{6}}(1-\pi_0(D)^2)\\
    \tau_2(D,\pi_0) &= \E{g(D,A_2,A_3)\mid A_{2:3} \simiid \pi_0} = \binom{\sfrac{1}{6}}{\sfrac{5}{6}}(1-\pi_0(C)^2) + \binom{\sfrac{5}{6}}{\sfrac{1}{6}}\pi_0(C)^2.
\end{align*}

The increase in probability of both copies of Theodora cooperating rather than defecting, from one of the simulations cooperating rather than defecting is given by:
\begin{align*}
    (\tau_2(C,\pi_0)_1 - \tau_2(D,\pi_0)_1)F_2(\pi_0)_1 &= (\tfrac{5}{6}-\tfrac{1}{6})\prob{A_{2:3} = (C,D) \text{ or } (D,C) \mid A_{2:3}\simiid \pi_0}(\tfrac{1}{6}+2\pi_0(C)^2-\tfrac{4}{3}\pi_0(C)^3)\\
    &= \tfrac{2}{3}(2\pi_0(C)\pi_0(D))(\tfrac{1}{6}+2\pi_0(C)^2-\tfrac{4}{3}\pi_0(C)^3)\\
    &= \tfrac{2}{9}\pi_0(C)\pi_0(D)+\tfrac{8}{3}\pi_0(C)^3\pi_0(D)-\tfrac{16}{9}\pi_0(C)^4\pi_0(D).
\end{align*}

We obtain:

\begin{align*}
    \Ex_{GSGT}&[u\mid do(C), \pi_0] - \Ex_{GSGT}[u\mid do(D),\pi_0]\\
    &= -2\prob{\mathit{Do}_C\cup \mathit{Do}_D} + 3\prob{\mathit{Th}_0\cup \mathit{Th}_C \cup \mathit{Th}_D}(\tau_2(C,\pi_0)_1 - \tau_2(D,\pi_0)_1)F_2(\pi_0)_1\\
    &= -2\times\tfrac{1}{7} + 3\times \tfrac{6}{7}(\tfrac{2}{9}\pi_0(C)\pi_0(D)+\tfrac{8}{3}\pi_0(C)^3\pi_0(D)-\tfrac{16}{9}\pi_0(C)^4\pi_0(D))\\
    &= -\tfrac{2}{7}+\tfrac{4}{7}\pi_0(C)\pi_0(D)+\tfrac{48}{7}\pi_0(C)^3\pi_0(D) -\tfrac{32}{7}\pi_0(C)^4\pi_0(D)\\
    &= \tfrac{1}{7}(32\pi_0(C)^5-80\pi_0(C)^4 +48\pi_0(C)^3-4\pi_0(C)^2+4\pi_0(C)-2).
\end{align*}

We have \textit{ex ante} expected utility

\begin{align*}
    &\E{u\mid \pi_0} = 2(1-\pi_0(C)) + 3 (\tfrac{1}{6}+2\pi_0(C)^2-\tfrac{4}{3}\pi_0(C)^3)^2\\
    &= \tfrac{16}{3}\pi_0(C)^6-16\pi_0(C)^5+12\pi_0(C)^4-\tfrac{4}{3}\pi_0(C)^3+2\pi_0(C)^2-2\pi_0(C)+\tfrac{25}{12}.
\end{align*}

We can verify that the directional derivatives in the directions $C-\pi_0$ and $D-\pi_0$ are given by:

\begin{align*}
    \frac{\partial}{\partial \pi_0(C)}\E{u\mid \pi_0} &= (1-\pi_0(C))(32\pi_0(C)^5-80\pi_0(C)^4+48\pi_0(C)^3-4\pi_0(C)^2+4\pi_0(C)-2),\\
    \frac{\partial}{\partial \pi_0(D)}\E{u\mid \pi_0} &= -\pi_0(C)(32\pi_0(C)^5-80\pi_0(C)^4+48\pi_0(C)^3-4\pi_0(C)^2+4\pi_0(C)-2).\\
\end{align*}
Both of these are at most zero when either:
\begin{itemize}
    \item $32\pi_0(C)^5-80\pi_0(C)^4+48\pi_0(C)^3-4\pi_0(C)^2+4\pi_0(C)-2 = 0$, giving $\pi_0(C) \approx 0.88$ or $\approx 0.36$,
    \item or when $\pi_0(C) = 1 $ and $32\pi_0(C)^5-80\pi_0(C)^4+48\pi_0(C)^3-4\pi_0(C)^2+4\pi_0(C)-2 \geq 0$, giving no solutions,
    \item or when $\pi_0(C)= 0$ and $32\pi_0(C)^5-80\pi_0(C)^4+48\pi_0(C)^3-4\pi_0(C)^2+4\pi_0(C)-2 \leq 0$, at $\pi_0(C)=0$.
\end{itemize}

\subsection{Details for the global approximation in the limit approach to Example \ref{eg:glob_approx}}\label{sec:glob_approx_eg_deets}

In the $N$-sample simulation model, when Dorothea is one of Theodora's simulations, the increase in probability of that copy of Theodora cooperating from Dorothea cooperating rather than defecting is given by:

\begin{align*}
    &\E{\sqrt{0.1 + 0.8(\tfrac{N-1}{N}\hat{p}_{N-1}+\tfrac{1}{N})}-\sqrt{0.1 + 0.8(\tfrac{N-1}{N}\hat{p}_{N-1})}}\\
    &= \frac{0.8}{N}\E{\left(\sqrt{0.1 + 0.8(\tfrac{N-1}{N}\hat{p}_{N-1}+\tfrac{1}{N})}+\sqrt{0.1 + 0.8(\tfrac{N-1}{N}\hat{p}_{N-1})}\right)^{-1}}\\
    &= \frac{0.4}{N}\E{(0.1 + 0.8\hat{p}_{N-1})^{-1/2}}+o(N^{-1})
\end{align*}

where the last line follows by noting that the difference in $\frac{0.8}{N}\E{\left(\sqrt{0.1 + 0.8(\tfrac{N-1}{N}\hat{p}_{N-1}+\tfrac{1}{N})}+\sqrt{0.1 + 0.8(\tfrac{N-1}{N}\hat{p}_{N-1})}\right)^{-1}}$ and $\frac{0.4}{N}\E{(0.1 + 0.8\hat{p}_{N-1})^{-1/2}}$ may be upper bounded by the maximum difference for any fixed $\hat{p}_{N-1}$, and then Taylor expanding. In particular, if

\begin{equation*}
    f_p(x) = \left(\sqrt{0.1 + 0.8p +0.8x(1-p)}+\sqrt{0.1 + 0.8p - 0.8xp}\right)^{-1}
\end{equation*}

Then for all $p$, $f_p$ is differentiable in $x$ in a neighbourhood of 0, and the derivative is bounded in $p$, and so we may Taylor expand $f_p$ about 0, and hence obtain $\sup_p |f_p(x)-f_p(0)| = O(x)$.

Meanwhile, we have

\begin{equation*}
    (0.1+0.8p')^{-1/2} = (0.1+0.8p)^{-1/2}+O(|p'-p|).
\end{equation*}
Hence, for some $\lambda$,
\begin{align*}
    \left| \E{(0.1 + 0.8\hat{p}_{N-1})^{-1/2}}- (0.1+0.8\pi_0(C))^{-1/2} \right|\lesssim \lambda \E{|\hat{p}_{N-1}-\pi_0(C)|}\leq \lambda \E{(\hat{p}_{N-1}-\pi_0(C))^2}^{1/2} &= \lambda (\Var \hat{p}_{N-1})^{1/2}\\&= O(N^{-1/2}).
\end{align*}
 Thus, we have increase in probability of one copy of Theodora cooperating from one of the simulations of Dorothea cooperating equal to 
 \begin{equation*}
 \frac{0.4}{N\sqrt{0.1+0.8\pi_0(C)}} + o(N^{-1}).
  \end{equation*}

\section{Proofs}
\subsection{Additional notation}
For clarity, we use the following notation in some of our proofs:
\begin{notn}
    We use a slight variant on Landau little-$o$ notation: Given variables $x_{1:l}$, $y_{1:m}$, denote by $o(\eps\mid x_{1:l})$ a term $h(\eps,x_{1:l},y_{1:m})$ such that $\sup_{y_{1:m}} \frac{h}{\eps} \rightarrow 0$ as $\eps\rightarrow 0$.
\end{notn}
\subsection{Lemma \ref{cor:exp_hist_length_finite}}\label{sec:hist_length_exp}
\exphistlengthfinite*
\begin{proof}
    Follows immediately from Lemma \ref{lemma:hist_length_mgf}.
\end{proof}
\begin{lemma}\label{lemma:hist_length_mgf} Assume \Cref{asm:finite_hist} holds. Let $L$ be the history length. Then $\sup_{s,\vpi}\Ex^s[L \mid\vpi]$ is finite.
\end{lemma}
\begin{proof}
Since the termination time under $\bm{\pi}$ is almost surely finite starting in state $s$, we may pick $t_s(\bm{\pi})$ such that $\mathbb{P}^s(L \leq t_s \mid \bm{\pi}) \geq \frac{1}{2}$. Then, setting $t(\bm{\pi})=\max_s t_s(\bm{\pi})$, we also have $\mathbb{P}^s(L \leq t \mid {\bm{\pi}}) \geq \frac{1}{2}$ for all $s$, so that $\mathbb{P}^s(L>t \mid \bm{\pi}) \leq \frac{1}{2}$.

Then, let $t_0$ be the maximum value of $t(\bm{\pi})$, taken over pure policies $\bm{\pi}$. Now, for a mixed policy $\bm{\pi}$, we have probability at least $|A|^{-t}$ of acting according to some pure policy for the first $t$ steps. Hence, for any policy $\bm{\pi}$, the probability of terminating in at most $t$ steps is $\mathbb{P}^s(L\leq t) \geq \frac{1}{2|A|^t}$.

But then, for all $\bm{\pi}$, by considering the state after $t$ steps, we get

\begin{align*}
    \max_s\Ex^s[{L\mid \bm{\pi}}] &\leq \frac{1}{2|A|^t}t + \left(1-\frac{1}{2|A|^t}\right)\frac{1}{2|A|^t}(2t) + \cdots\\
    &=t\sum_{j=1}^\infty \frac{1}{2|A|^t}\left(1-\frac{1}{2|A|^t}\right)^{j-1} j.
\end{align*}
which is just $t$ times the mean of a geometric random variable with success probability $p = \frac{1}{2|A|^t}$, i.e., $2|A|^t$. Thus, $\sup_{\vpi}\max_s\Ex^s[{L\mid \bm{\pi}}] \leq 2t|A|^t.$
\end{proof}
\subsection{Proofs for Section \ref{sec:glob_sims_model}}

\globalsimsresult*
\begin{proof}\label{pf:globalsimsresult}
    Call the decision problem $D$. We proceed as discussed above, by constructing a modified problem, $D'$ in which all dependence functions are the identity. We get that GT applied to $D'$ is equivalent to GSGT applied to the original problem, and that the \textit{ex ante} utilities are the same between the two problems. We then apply the results on CDT+GT.

    Construct the modified problem, $D'$ as follows:
    \begin{itemize}
        \item For each non-terminal state in the original problem, create new states corresponding to each simulation and to the actions of previous simulations. That, is, the new non-terminal states are given by:
        \begin{equation*}
            S'-S_T' = \{(s,a_{1:k}) : s\in S-S_T, a_{1:k} \in A^k, 0\leq k < N_{i(s)}\}.
        \end{equation*}
        The new terminal states are instead:
        \begin{equation*}
            S_T' = \{(s,()) : s\in S_T\}.
        \end{equation*}
        where $()$ is the empty tuple.
        \item The initial distribution $P'_0$ is given by $P'_0(s,a_{1:k}) = I\{k=0\} P_0(s)$.
        \item The index function is just 1 everywhere.
        \item The set of actions is $A$, as before.
        \item The transition function, $T'(\cdot\mid (s,a_{1:k}), a)$, has two cases:
        \begin{itemize}
            \item For $k< N_{i(s)}-1$, we have $T'((s, (a_{1:k},a))\mid (s,a_{1:k}), a) = 1$. I.e., we just deterministically append the new action to the list of actions at this state so far.
            \item For $k = N_{i(s)}-1$, and $s'\in S$, we have $T'((s',())\mid (s,a_{1:k}), a) = T(s'\mid s,g_{i(s)}(a_{1:k},a))$
            and $T'(\cdot\mid (s,a_{1:k}), a)$ is zero everywhere else.
        \end{itemize}
        \item The utility function is just $u'((s,())) = u(s)$ for $s\in S_T$.
        \item The dependence function is just deterministically the identity.
        \end{itemize}

        First, we will show that the \textit{ex ante} expected utility for $D'$ is the same as for $D$.

        Note that each time non-terminal state $(s,())$ is reached in $D'$, it is followed by $N_{i(s)}-1$ states of the form $(s,a_1),\ldots, (s,a_{1:N_{i(s)}})$, before we transition to a new state of the form $(s',())$. We may therefore map histories in $D'$ to histories in $D$ in the following way: For each sequence of $N_{i(s)}$ states of the form $(s,a_{1:k})$, replace this with one state, $s$. Replace terminal states $(s,())$ with terminal states $s$.
        
         Write $H'$ for a random variable distributed as a history in the $D'$, $\omega$ for the map from such histories to histories in $D$, and $H$ for a random variable distributed as histories in $D$. Note that the map $\omega$ preserves utility. It therefore suffices to show that $\omega(H')$ follows the same distribution as $H$. In turn, for this it suffices to show that the initial state of $\omega(H')$ follows the same distribution as for the original problem, and that the transition probabilities are the same. We have:

         \begin{equation*}
             \prob{\text{Initial state in $\omega(H')$ is $s$}} = P_0'((s,()))= P_0(s) = \prob{\text{Initial state in $H$ is $s$}}.
         \end{equation*}

         Meanwhile,

         \begin{align*}
             &\prob{\text{Next state in $\omega(H')$ is $s'$} \mid \text{current state is $s$, policy $\pi_0$}}\\
             &= \sum_{a_{1:N_{i(s)}}}T\left(s'\mid s, g_{i(s)}(a_{1:N_{i(s)}})\right)\prod_{j=1}^{N_{i(s)}}\pi_0(a_j)\\
             &=\E{T\left(s'\mid s, g_{i(s)}(A_{1:N_{i(s)}})\right)\mid A_{1:N_{i(s)}}\simiid \pi_0}\\
             &=T\left(s'\mid s, \E{g_{i(s)}(A_{1:N_{i(s)}})\mid A_{1:N_{i(s)}}\simiid \pi_0}\right)\\
             &=T\left(s'\mid s, F_{i(s)}(\pi_0)\right)\\
             &= \prob{\text{Next state in $H$ is $s'$}\mid\text{current state is $s$, policy $\pi_0$}}.
         \end{align*}

         We conclude that
         \begin{equation*}
             \E{u'(H')\mid \pi_0} = \E{u(\omega(H'))\mid \pi_0} = \E{u(H)\mid \pi_0}.
         \end{equation*}

         Now, consider applying GT to $D'$. From the above, we have that, given policy $\pi_0$ the expected number of instances of non-terminal state $(s,())$ is $\Ex_D[{\#(s)\mid \pi_0}]$. Thus,
         \begin{align*}
             \sum_{a_{1:k}}P_{GT}((s,a_{1:k})\mid \pi_0) &= \frac{\sum_{a_{1:k}}\Ex_{D'}[\#(s,a_{1:k})]}{\sum_{r\in S'-S'_T} \Ex_{D'}[\#(r)\mid \pi_0]}\\
             &= \frac{N_{i(s)}\Ex_D[{\#(s)\mid \pi_0}]}{\sum_{s\in S-S_T}N_{i(s)}\Ex_{D}[\#(s)\mid \pi_0]}\\
             &= P_{GSGT}(s\mid \pi_0).
         \end{align*}

         Then, conditional on being at some state of the form $(s,a_{1:k})$ in $D'$ (with $a_{1:k}$ and $k$ unknown), the CDT+GT expected value of taking action $a$ is:
         \begin{align*}
             &\sum_{k, a_{1:k}}\frac{P_{GT}(s,a_{1:k}\mid \pi_0)}{P_{GSGT}(s\mid \pi_0)}\Ex_{D'}^{(s,a_{1:k})}[u\mid \Do(a),\pi_0]\\
             &= \sum_{k, a_{1:k}} \frac{\prod_{j=1}^k \pi_0(a_j)}{N_{i(s)}}\Ex_{D}^{s}\left[u\middle| \Do\left(\E{g_{i(s)}((a_{1:k},a,A_{(k+2):N_{i(s)}}))\middle| A_{(k+2):N_{i(s)}} \simiid \pi_0}\right),\pi_0\right]\\
             &= \sum_{k} \frac{1}{N_{i(s)}}\Ex_{D}^{s}\left[u\middle| \Do\left(\E{g_{i(s)}((A_{1:k},a,A_{(k+2):N_{i(s)}})\middle| A_{-k} \simiid \pi_0}\right),\pi_0\right]\\
             &= \Ex_{D}^{s}\left[u\middle| \Do\left(\tau_{GSGT,i(s)}(a,\pi_0)\right),\pi_0\right]
         \end{align*}

         Hence, the CDT+GT policies in $D'$ are exactly the CDT+GSGT policies in $D$. Applying the results on CDT+GT (Theorem \ref{thm:CDT+GT} and Corollary  \ref{cor:CDT+GT}) to $D'$, we therefore get that the CDT+GSGT policies in $D$ are exactly those $\pi_0$ for which $\frac{\partial}{\partial\pi_0(a)}\E{u\mid \pi_0}\leq 0$, including all \textit{ex ante} optimal policies.
\end{proof}

\globalsim*
\begin{proof}\label{pf:globalsim}
    First, suppose that \ref{global_sims} holds, i.e. $F(\pi_0) = \E{g(A_1,\ldots,A_N)\mid A_1,\ldots,A_N \sim \pi_0}.$
    Then, we have $F(\pi_0) = \sum_{a_1,\ldots, a_N}g(a_1,\ldots,a_N)\prod_{j=1}^N\pi_0(a_j)$,
 which is a polynomial of degree at most $N$ with non-negative coefficients.

    For the converse, suppose \ref{poly} holds. Now, by successively multiplying any terms with degree less than $N$ by $1 = (\pi_0(1)+\cdots +\pi_0(k))$, we may assume $F$ may be written as a polynomial with non-negative coefficients and all terms of exactly degree $N$, i.e,
    \begin{equation}
        F(\pi_0) = \sum_{\substack{n_{1:k}\geq 0\\\sum_j n_j = N}}c_{n_{1:k}}\prod_{j=1}^k \pi_0(j)^{n_j}
    \end{equation}
    where $k = |A|$ and each $c_{n_{1:k}} \in \mathbb{R}^k$ has $c_{n_{1:k}}\geq 0$ (i.e., $(c_{n_{1:k}})_j\geq 0$ for all $j$). Taking the dot product with $\bm{1} = (1,1,\ldots,1) \in \mathbb{R}^k$, we have
    \begin{equation*}
        1 = \bm{1}^TF(\pi_0) = \sum_{\substack{n_{1:k}\geq 0\\\sum_j n_j = N}}\bm{1}^Tc_{n_{1:k}}\prod_{j=1}^k \pi_0(j)^{n_j}
    \end{equation*}
    for all $\pi_0 \in \Delta(A)$. Now, we also have that
    
    \begin{equation*}
        1 = \left(\pi_0(1)+\pi_0(2)+\cdots+\pi_0(k)\right)^N = \sum_{\substack{n_{1:k}\geq 0\\\sum_j n_j = N}}\binom{N}{n_1, \ldots, n_k}\prod_{j=1}^k \pi_0(j)^{n_j}
    \end{equation*}
    for all $\pi_0 \in \Delta(A)$, where $\binom{N}{n_1, \ldots, n_k}$ is the multinomial coefficient $\binom{N}{n_1, \ldots, n_k} \defeq \frac{N!}{n_1!\cdots n_k!}$.

    Since these two polynomials agree on $\Delta(A)$, and are both in the form where all terms have degree $N$, Lemma \ref{lemma:poly_unique} gives us that their coefficients must be the same. Thus, we have for all $n_{1:k}$, $\bm{1}^Tc_{n_{1:k}} = \binom{N}{n_1, \ldots, n_k}$. Hence, since also $c_{n_{1:k}} \geq 0$, we have $\binom{N}{n_1, \ldots, n_k}^{-1}c_{n_{1:k}}\in \Delta(A)$.

    Then, we can just set $g$ symmetric such that when $|\{i:a_i = j\}| = n_j$ for each $j$, $g(a_1,\ldots,a_N) = \binom{N}{n_1, \ldots, n_k}^{-1}c_{n_{1:k}}$. We obtain

    \begin{align*}
        F(\pi_0) &= \sum_{\substack{n_1,\ldots,n_k\geq 0\\\sum_j n_j = N}} \binom{N}{n_1, \ldots, n_k}\binom{N}{n_1, \ldots, n_k}^{-1}c_{n_{1:k}} \prod_{j=1}^k \pi_0(j)^{n_j}\\
        &= \sum_{\substack{n_1,\ldots,n_k\geq 0\\\sum_j n_j = N}}\sum_{\substack{a_{1:N}\\\sum_i I\{a_i=j\} = n_j \forall j}}  g(a_1,\ldots,a_N) \prod_{j=1}^k \pi_0(j)^{n_j}\\
        &= \sum_{a_1,\ldots, a_N}g(a_1,\ldots,a_N)\prod_{j=1}^N\pi_0(a_j)\\
        &=\E{g(A_1,\ldots,A_N)\mid A_1,\ldots,A_N \sim \pi_0}
    \end{align*}

    and we are done.
    
\end{proof}

\begin{lemma}\label{lemma:poly_unique}
    Suppose that two polynomials from $\mathbb{R}^k$ to $\mathbb{R}$ with all terms of degree $N$ are equal on $\Delta^{k-1}$ (the $k-1$ dimensional unit simplex in $\mathbb{R}^k$). Then their coefficients are equal.

    In other words, if for all $\bm{p} \in \Delta^{k-1}$,
    \begin{equation*}
        \sum_{\substack{n_1,\ldots n_k \geq 0\\\sum_j n_j = N}} a_{n_{1:k}} \prod_{j=1}^k p_j^{n_j} = \sum_{\substack{n_1,\ldots n_k \geq 0\\\sum_j n_j = N}} b_{n_{1:k}} \prod_{j=1}^k p_j^{n_j}
    \end{equation*}
    then for all $n_{1:k}$, $a_{n_{1:k}} = b_{n_{1:k}}$.
\end{lemma}
\begin{proof}
    Consider the map, call it $\phi$, from coefficients of polynomials on $\mathbb{R}^k$ with all terms having degree $N$ to coefficients of polynomials in $\mathbb{R}^{k-1}$ (of degree at most $N$) formed by replacing $p_k$ with $(1-p_1\cdots -p_k)$ and then multiplying out. Let $\psi$ similarly be the map from coefficients of polynomials of degree at most $N$ in $\mathbb{R}^{k-1}$ to polynomials with all terms having degree $N$ in $\mathbb{R}^k$ formed by multiplying terms of degree $N-j$ by $(p_1+\cdots p_k)^j$ and then multiplying out.

    We will show that $\psi\circ\phi$ is the identity. To see this, note that replacing $p_k$ by $(1-p_1\cdots -p_k)$ results in the polynomial
    \begin{equation*}
        \sum_{\substack{n_1,\ldots n_k \geq 0\\\sum_j n_j = N}} a_{n_{1:k}} \left(\prod_{j=1}^{k-1} p_j^{n_j}\right) \left(1-\sum_{j=1}^{k-1}p_j\right)^{n_k}.
    \end{equation*}
    prior to multiplying out. Then, multiplying terms of degree $N-j$ (after multiplying out), by $(p_1+\cdots +p_k)^j$ is equivalent to, in the above expression, replacing $1$ by $(p_1+\cdots p_k)^j$ in the factors of $\left(1-\sum_{j=1}^{k-1}p_j\right)$. But this just gives

    \begin{equation*}
        \sum_{\substack{n_1,\ldots n_k \geq 0\\\sum_j n_j = N}} a_{n_{1:k}} \left(\prod_{j=1}^{k-1} p_j^{n_j}\right) p_k^{n_k},
    \end{equation*}
    which is the original polynomial. Hence, $\psi\circ\phi$ is indeed the identity.

    Now, since the polynomials resulting from applying $\phi$ to $(a_{n_{1:k}})$ and $(b_{n_{1:k}})$ are polynomials on $\mathbb{R}^{k-1}$ agreeing on an open ball in $\mathbb{R}^{k-1}$, they must be the same polynomial, and their coefficients must be equal. Hence, $\phi(a_{n_{1:k}}) = \phi(b_{n_{1:k}})$. But then $\psi\circ\phi((a_{n_{1:k}})) = \psi\circ\phi((b_{n_{1:k}}))$, and so $(a_{n_{1:k}})= (b_{n_{1:k}})$.
\end{proof}

\polyfactorchar*
\begin{proof}\label{pf:polyfactorchar}
    Consider the factorisation of $F_i$: we have that the polynomial with no zeros on the simplex must be positive on the entire simplex. Then, from Lemma \ref{lemma:heat_diffusion}, we can express it as a polynomial with only positive coefficients. Multiplying out, we obtain a polynomial with only non-negative coefficients. This holds for all $i$, so applying \Cref{prop:global_sim}, the result follows.
\end{proof}

\heatdiffusion*
\begin{proof}\label{pf:heat_diffusion}
    We will proceed by induction on $k$.
    
    For $k=1$, the result is trivial: $\Delta^0 =\{1\}$, and $f$ is simply some positive constant.

    Then consider $k\geq 2$. Assume the result holds for all $r<k$. Write

    \begin{equation*}
        f(\p) = \sum_{\substack{n_{1:k}\geq 0\\ \sum_i n_i = N}} a_{n_{1:k}}p_1^{n_1}\cdots p_k^{n_k}
    \end{equation*}

    for some coefficients $a_{n_{1:k}}\in \mathbb{R}$. If all such coefficients are non-negative, we are done. Otherwise, write $a_{\min}<0$ for the lowest coefficient.
    Let $f_{\min}> 0$ be the smallest value taken by $f$ on $\Delta^k$ (well-defined by continuity of $f$ and compactness of the simplex). Then, write

    \begin{equation*}
        f^+(\p) = \sum_{\substack{n_{1:k}\geq 0\\ \sum_i n_i = N \\ a_{n_{1:k}}\geq 0}} a_{n_{1:k}}p_1^{n_1}\cdots p_k^{n_k}
    \end{equation*}
    and let $f^+_{\max}$ be the maximum value taken by $f^+$ on the simplex. This is again well-defined since $f^+$ is continuous and $\Delta^k$ is compact.

    Consider expressing $f$ as a degree $M$ polynomial, where $M > N$ (by multiplying terms by $(p_1+\cdots+p_k)^{M-N}$):

    \begin{equation*}
        f(\p) = \sum_{\substack{n_{1:k}\geq 0\\ \sum_i n_i = M}} b_{n_{1:k}}p_1^{n_1}\cdots p_k^{n_k}.
    \end{equation*}
    Let $m = M-N$. How are the $b_{n_{1:k}}$ and the $a_{n_{1:k}}$ are related? We have that the contribution of the $a_{r_{1:k}}p_1^{r_1}\cdots p_k^{r_k}(p_1+\cdots+p_k)^m$ term to the coefficient of $p_1^{n_1}\cdots p_k^{n_k}$ is given by $a_{r_{1:k}}$ times the coefficient of $p_1^{n_1-r_1}\cdots p_k^{n_k-r_k}$ in $(p_1+\cdots+p_k)^m$: namely, zero unless $r_{1:k}\leq n_{1:k}$, and otherwise $a_{r_{1:k}}$ times the multinomial coefficient $\binom{m}{n_1-r_1,\ n_2 - r_2, \cdots, n_k - r_k}$. Thus, we have:

    \begin{align}
        b_{n_{1:k}} &= \sum_{\substack{0\leq r_{1:k}\leq n_{1:k}\\ \sum_i r_i = N}} a_{r_{1:k}}\binom{m}{n_1-r_1,\ n_2 - r_2, \cdots, n_k - r_k}\\
        &=\sum_{\substack{0\leq r_{1:k}\leq n_{1:k}\\ \sum_i r_i = N}} a_{r_{1:k}}\frac{m!}{(n_1-r_1)!(n_2-r_2)!\cdots(n_k-r_k)!}
    \end{align}

    Now, let $\theta\in(0,1)$, be small enough that $\theta < \frac{1}{3k}$ (so that $\frac{\theta}{k^{-1}-2\theta}<1$) and $\frac{\theta}{k^{-1}-3\theta}\leq\frac{f_{\min}}{-2a_{\min}}$ (possible since $a_{\min}<0$ and $\frac{\theta}{k^{-1}-3\theta}\rightarrow 0$ as $\theta \rightarrow 0$). Then, let $M$ be large enough that $\theta M > N$ (we will further restict $M$ later).
    
    Fix $n_{1:k}$ s.t. $\sum_i n_i = M$. We have two cases:
    \begin{enumerate}
        \item For all $i=1,\ldots,k$, $n_i \in (\theta M, (1-\theta)M)$.
        \item There is some set of indices $\varnothing \subset J\subset \{1,\ldots, k\}$, such that for all $i\in J$, $n_i < \theta M$.
    \end{enumerate}
    (Note that if $n_i>(1-\theta)M$ for some $i$, we have $n_j<\theta M$ for all $j\neq i$, and $k\geq 2$, so these two cases are exhaustive.)

    We turn now to the first case. Our strategy here will be to use the fact that central multinomial coefficients are approximately in geometric sequence, allowing us to relate the coefficient $b_{n_{1:k}}$ to $f$ at some point.
    
    We have:

    \begin{align*}
        b_{n_{1:k}} &=\sum_{\substack{0\leq r_{1:k}\leq n_{1:k}\\ \sum_i r_i = N}} a_{r_{1:k}}\frac{m!}{(n_1-r_1)!(n_2-r_2)!\cdots(n_k-r_k)!}\\
        &=\frac{m!}{n_1!n_2!\cdots n_k!}\sum_{\substack{0\leq r_{1:k}\leq n_{1:k}\\ \sum_i r_i = N}} a_{r_{1:k}}\prod_{i=1}^k n_i (n_i-1)\cdots (n_i-(r_i-1))\\
        &=\frac{m!}{n_1!n_2!\cdots n_k!}\sum_{\substack{0\leq r_{1:k}\leq n_{1:k}\\ \sum_i r_i = N}} a_{r_{1:k}}\prod_{i=1}^k n_i^{r_i} (1-\tfrac{1}{n_i})\cdots (1-\tfrac{r_i-1}{n_i})\\
        &=\frac{m!M^N}{n_1!n_2!\cdots n_k!}\sum_{\substack{0\leq r_{1:k}\leq n_{1:k}\\ \sum_i r_i = N}} a_{r_{1:k}}\prod_{i=1}^k \frac{n_i^{r_i}}{M^{r_i}} (1-\tfrac{1}{n_i})\cdots (1-\tfrac{r_i-1}{n_i})\\
        &=\frac{m!M^N}{n_1!n_2!\cdots n_k!}\sum_{\substack{0\leq r_{1:k}\leq n_{1:k}\\ \sum_i r_i = N}} a_{r_{1:k}}\left(1-\left(1-\prod_{i=1}^k(1-\tfrac{1}{n_i})\cdots (1-\tfrac{r_i-1}{n_i})\right)\right)\prod_{i=1}^k \left(\frac{n_i}{M}\right)^{r_i}\\
        &=\frac{m!M^N}{n_1!n_2!\cdots n_k!}\left(f((\tfrac{n_1}{M},\ldots\tfrac{n_k}{M})) - \sum_{\substack{0\leq r_{1:k}\leq n_{1:k}\\ \sum_i r_i = N}} a_{r_{1:k}}\left(1-\prod_{i=1}^k(1-\tfrac{1}{n_i})\cdots (1-\tfrac{r_i-1}{n_i})\right)\prod_{i=1}^k \left(\frac{n_i}{M}\right)^{r_i}\right)
        \end{align*}
        
        Now, we bound the second term in the bracket above:
        \begin{align*}
        &\sum_{\substack{0\leq r_{1:k}\leq n_{1:k}\\ \sum_i r_i = N}} a_{r_{1:k}}\left(1-\prod_{i=1}^k(1-\tfrac{1}{n_i})\cdots (1-\tfrac{r_i-1}{n_i})\right)\prod_{i=1}^k \left(\frac{n_i}{M}\right)^{r_i}\\
        &\leq \sum_{\substack{0\leq r_{1:k}\leq n_{1:k}\\ \sum_i r_i = N\\a_{r_{1:k}}\geq 0}} a_{r_{1:k}}\left(1-\prod_{i=1}^k(1-\tfrac{N}{n_i})^{r_i }\right)\prod_{i=1}^k \left(\frac{n_i}{M}\right)^{r_i} &\text{(Noting that $n_i\geq\theta M>N>r_i-1$)}\\
        &\leq \sum_{\substack{0\leq r_{1:k}\leq n_{1:k}\\ \sum_i r_i = N\\a_{r_{1:k}}\geq 0}} a_{r_{1:k}}\left(1-(1-\tfrac{N}{\theta M})^N\right)\prod_{i=1}^k \left(\frac{n_i}{M}\right)^{r_i} &\text{(Using $n_i\geq \theta M$, and $\theta M > N$)}\\
        &\leq \sum_{\substack{0\leq r_{1:k}\leq n_{1:k}\\ \sum_i r_i = N\\a_{r_{1:k}}\geq 0}} a_{r_{1:k}}\frac{N^2}{\theta M}\prod_{i=1}^k \left(\frac{n_i}{M}\right)^{r_i} &\text{(Using that for $X\sim \Bin(N,\tfrac{N}{\theta M})$, $\Ex X \geq \prob{X\geq 1}$)}\\
        &\leq \frac{N^2}{\theta M}f^{+}((\tfrac{n_1}{M},\ldots, \tfrac{n_k}{M})).
    \end{align*}

    Putting this together, we have:

    \begin{align*}
        b_{n_{1:k}} &\geq \frac{m!M^N}{n_1!n_2!\cdots n_k!}\left(f((\tfrac{n_1}{M},\ldots\tfrac{n_k}{M})) - \frac{N^2}{\theta M}f^{+}((\tfrac{n_1}{M},\ldots, \tfrac{n_k}{M}))\right)
        \geq \frac{m!M^N}{n_1!n_2!\cdots n_k!}\left(f_{\min} - \frac{N^2}{\theta M}f^{+}_{\max}\right)\\
        &> 0 \qquad\text{for $M>\frac{N^2f^{+}_{\max}}{f_{\min}}$}
    \end{align*}

    Thus, for $M$ large enough, we have that $b_{n_{1:k}} > 0$ when $n_i\in(\theta M, (1-\theta)M)$ for all $i$.

    We turn now to the second case: There is some set of indices $\varnothing \subset J\subset \{1,\ldots, k\}$, such that for all $i\in I$, $n_i < \theta M$.

    Considering $j\in J$, we have that $\sum_{i\neq j}n_i > (1-\theta)M$. Hence, for some $i$, we must have $n_i\geq\frac{(1-\theta)}{k-1}M>\frac{(1-\theta)}{k}M$. For ease of notation, let us reorder indices so that $1\in J$ (i.e. $n_1<\theta N$), and $n_2>\frac{(1-\theta)}{k}M> (\frac{1}{k}-\theta)M$.
    
    Recall that $\theta\in(0,1)$ is defined to be small enough that $\frac{\theta}{k^{-1}-2\theta}<1$ and $\frac{\theta}{k^{-1}-3\theta}\leq\frac{f_{\min}}{-2a_{\min}}$. Let $M$ be large enough that $\theta M>N$, as before.

    Now, we will express $b_{n_{1:k}}$ as a sum over $a_{r_{1:k}}$ by splitting into terms according to $r_1+r_2$ and $r_1$. We will show that the terms with $r_1> 0$ are small, and this will allow us to express $b_{n_{1:k}}$ in terms of a polynomial on $\mathbb{R}^{k-1}$, allowing us to apply our induction assumption. We have:

    \begin{align*}
        b_{n_{1:k}} &= \sum_{s=0}^N\sum_{\substack{0\leq r_{3:k}\leq n_{3:k}\\ \sum_{i=3}^k r_i \leq N-s}} \sum_{t=0}^s a_{(t,s-t,r_{3:k})}\frac{m!}{(n_1-t)!(n_2-(s-t))!\prod_{i=3}^k (n_i-r_i)!}\\
        &= \sum_{s=0}^N\sum_{\substack{0\leq r_{3:k}\leq n_{3:k}\\ \sum_{i=3}^k r_i \leq N-s}} \frac{m!}{\prod_{i=3}^k (n_i-r_i)!}\left(a_{(0,s,r_{3:k})}\frac{1}{n_1!(n_2-s)!} + \sum_{t=1}^s a_{(t,s-t,r_{3:k})}\frac{1}{(n_1-t)!(n_2-(s-t))!}\right)\\
        &= \sum_{s=0}^N\sum_{\substack{0\leq r_{3:k}\leq n_{3:k}\\ \sum_{i=3}^k r_i \leq N-s}} \frac{m!}{n_1!(n_2-s)!\prod_{i=3}^k (n_i-r_i)!}\left(a_{(0,s,r_{3:k})} + \sum_{t=1}^s a_{(t,s-t,r_{3:k})}\frac{n_1(n_1-1)\cdots (n_1-(t-1))}{(n_2-(s-t))\cdots(n_2-(s-1))}\right)\\
        &\geq \sum_{s=0}^N\sum_{\substack{0\leq r_{3:k}\leq n_{3:k}\\ \sum_{i=3}^k r_i \leq N-s}} \frac{m!}{n_1!(n_2-s)!\prod_{i=3}^k (n_i-r_i)!}\left(a_{(0,s,r_{3:k})} + \sum_{t=1}^s a_{\min}\frac{n_1^t}{(n_2-s)^t}\right)\\
        &= \sum_{\substack{0\leq r_{1:k}\leq n_{1:k}\\ \sum_{i=1}^k r_i \leq N\\r_1=0}} \frac{m!}{\prod_{i=1}^k (n_i-r_i)!}\left(a_{r_{1:k}} + a_{\min}\sum_{t=1}^{r_2} \frac{n_1^t}{(n_2-r_2)^t}\right)\\
        \end{align*}
        Now, we have, for any $r_2$, using that $n_1<\theta M$, that $n_2> (\tfrac{1}{k}-\theta)M$, that $r_2<N$, and that $(\tfrac{1}{k}-\theta)M-N > (\tfrac{1}{k}-2\theta)M> 0$:
        \begin{align*}
            \sum_{t=1}^{r_2} \frac{n_1^t}{(n_2-r_2)^t}  \leq \sum_{t=1}^{r_2} \frac{(\theta M)^t}{((\frac{1}{k}-\theta)M-N)^t} &\leq \sum_{t=1}^{r_2} \frac{\theta^t}{(\frac{1}{k}-2\theta)^t}\\
            &\leq \sum_{t=1}^\infty \frac{\theta^t}{(\frac{1}{k}-2\theta)^t} &\text{(Since $\tfrac{\theta}{k^{-1}-2\theta}<1$)}\\
            &= \frac{\theta}{\tfrac{1}{k}-2\theta}\left(1-\frac{\theta}{\tfrac{1}{k}-2\theta}\right)^{-1} = \frac{\theta}{\tfrac{1}{k}-3\theta}\leq \frac{f_{\min}}{-2a_{\min}} &\text{(From the defn. of $\theta$)}
        \end{align*}

    So,
    \begin{align*}
        b_{n_{1:k}} &\geq \sum_{\substack{0\leq r_{1:k}\leq n_{1:k}\\ \sum_{i=1}^k r_i \leq N\\r_1=0}} \frac{m!}{\prod_{i=1}^k (n_i-r_i)!}\left(a_{r_{1:k}} + a_{\min}\frac{f_{\min}}{-2a_{\min}}\right)\\
        &=\sum_{\substack{0\leq r_{1:k}\leq n_{1:k}\\ \sum_{i=1}^k r_i \leq N\\r_1=0}} \frac{m!}{\prod_{i=1}^k (n_i-r_i)!}\left(a_{r_{1:k}} -\frac{1}{2}f_{\min}\right)\\
        &=\frac{m!}{n_1!(m-n_1)!}\sum_{\substack{0\leq r_{2:k}\leq n_{2:k}\\ \sum_{i=2}^k r_i \leq N}} \frac{(m-n_1)!}{\prod_{i=2}^k (n_i-r_i)!}\left(a_{r_{1:k}} -\frac{1}{2}f_{\min}\right).
    \end{align*}

    The last line there is then simply the $\binom{m}{n_1}$ times the $n_{2:k}$ coefficient of, when expressed as a polynomial of degree $M-n_1$, the $(k-1)$-dimensional polynomial $f_1$ defined on $\Delta^{k-2}$:
    \begin{equation*}
        f_1(p_{2:k})\defeq\sum_{\substack{r_{2:k}\geq 0\\ \sum_i r_i = N}}\left(a_{(0,r_{2:k})}-\frac{1}{2}f_{\min}\right)p_2^{r_2}\cdots p_k^{r_k}
    \end{equation*}

    Note that 
    \begin{align*}
        f_1(p_{2:k}) &\geq \sum_{\substack{r_{2:k}\geq 0\\ \sum_i r_i = N}}\left(a_{(0,r_{2:k})}-\frac{1}{2}f_{\min}\binom{N}{r_2,\ldots, r_k}\right)p_2^{r_2}\cdots p_k^{r_k}\\
        &=  -\frac{1}{2}f_{\min}(p_2+\cdots +p_k)^N +\sum_{\substack{r_{2:k}\geq 0\\ \sum_i r_i = N}}a_{(0,r_{2:k})}p_2^{r_2}\cdots p_k^{r_k}\\
        &=  -\frac{1}{2}f_{\min} +f((0,p_{2:k}))\\
        &> 0 \text{ for all $p_{2:k}$}.
    \end{align*}
    Hence, by the induction assumption, there exists $M_1$ such that the coefficients of $f_1$, expressed as a polynomial of degree at least $M_1$ polynomial, are all positive. Thus, setting $M>\frac{M_1}{1-\theta}$ (noting that $f_1$, and hence $M_1$, does not depend on $M$), we have $M-n_1 > M_1$, and $b_{n_{1:k}}$ must then be positive.
    
     We may similarly define $f_2,\ldots, f_k$ and $M_2,\ldots, M_k$. Then, provided $M>\frac{N^2f^+_{\max}}{\theta f_{\min}}$ (and so $M > N\theta$) and $M>\frac{\max_i M_i}{1-\theta}$, all $b_{n_{1:k}}$ the result follows for degree $k$ polynomials. By induction, we are then done.
\end{proof}

\Aequalstwo*
\begin{proof}\label{pf:Aequalstwo}
    To see that this condition is necessary, suppose that we can write $f$ as in equation (\ref{eq:global_samples}):

    \begin{equation*}
        f(\pi_0) = \E{g(A_1,\ldots, A_N)\mid A_1,\ldots,A_N\simiid \pi_0}.
    \end{equation*}
    
    Suppose further that for some $i$, and some $\pi_0=(p,1-p)$, with $p\in(0,1)$, we have $f_i(\pi_0) = 0$. Then, if $A_1,\ldots,A_N \simiid \pi_0$, for any $(a_1,\ldots,a_N) \in \{1,2\}^N$, we have $\prob{A_{1:N}=a_{1:N}}>0$. Thus, since $g(a_{1:N})\geq 0$ for all $a_{1:N}$, we must have that $g(a_{1:N}) = 0$ for all $a_{1:N}$.

    For the converse, we have two cases. If $f$ is a constant, $f$ can trivially be written as a function of action samples via taking $g$ to be a constant function.

    Otherwise, we have, for each $i$, $f_i(\pi_0) \in(0,1)$ for all $\pi_0$ in the interior of the simplex. Now, write $f_i$ as a polynomial in $p=\pi_0(1)$ (i.e., replace occurrences of $\pi_0(2)$ with $1-p$). Then, if $f_i(0)=0$, $f_i$ must have a factor of $p$. Similarly, if $f_i(1)=0$, $f_i$ must have a factor of $1-p$. Hence, we may divide through by factors of $p$ or $(1-p)$ and write $f_i$ in the form:

    \begin{equation*}
        f_i(p) = p^l(1-p)^m h_i(p)
    \end{equation*}
    where $h_i$ is a polynomial that is positive on the entire simplex.

    Equivalently, $f_i(\pi_0) = \pi_0(1)^l\pi_0(2)^m h_i(\pi_0)$. The same holds for each $i$, and applying Corollary \ref{cor:poly_factor_char}, we are then done.
\end{proof}

\globsimsnecessary*
\begin{proof}\label{pf:globsimsnecessary}
    Suppose that we can write
    \begin{equation*}
        f(\pi_0) = \E{g(A_1,\ldots,A_N)\mid A_1,\ldots,A_N \simiid \pi_0}.
    \end{equation*}
    Fix $t>0$, and suppose $p,q\in\Delta(A)$ are such that $p_j\geq t q_j$ for all $j$. Then we may write $p = tq + (1-t)r$ for some $r\in\Delta(A)$. Hence, for $A_i\sim p$, write $X_i \sim$ Bernoulli($t$) and we may view $A_i$ as following a mixture distribution: $A_i\mid X_i \sim X_iq + (1-X_i)r$, i.i.d. for each $i$. From this, we obtain:
    \begin{align*}
        f(p) &= \E{g(A_1,\ldots,A_N)}\\
        &= \E{\E{g(A_1,\ldots,A_N)\mid X_1,\ldots, X_N}}\\
        &\geq \prob{X_1,\ldots,X_N=1}\E{g(A_1,\ldots,A_N)\mid X_1=\cdots=X_N=1}\\
        &= t^N\E{g(A_1',\ldots,A_N')\mid A_1',\ldots,A_N' \simiid q}\\
        &= t^N f(q).
    \end{align*}

    Thus, the condition holds with $\lambda(t) = t^N$.
\end{proof}

\subsection{Proofs for Section \ref{sec:limit_approx}}

\unifconvergence*
\begin{proof}\label{pf:unifconvergence}
    Write $\pihat_N$ for the empirical distribution $\pihat_N(a_1,a_2, \ldots, a_N) \defeq \frac{1}{N}\sum_{k=1}^N a_k$.
    Note that $F$ is continuous and $F$'s domain, $\Delta(A)$, is compact. Thus, the Heine-Cantor theorem gives us that $F$ is uniformly continuous: for all $\eps>0$, there exists $\delta(\eps)>0$ such that, for all $\pi_0,\pi_0'\in \Delta(A)$, we have $\norm{\pi_0-\pi_0'}_2<\delta(\eps) \implies \norm{F(\pi_0)-F(\pi_0')}_2<\eps$.
    
    Let $g_N(A_1,\ldots,A_N) = F(\pihat_N(A_1,A_2, \ldots, A_N))$. We have $\pihat_N(A_1,A_2,\ldots,A_N)_j \sim \frac{1}{N}\Bin(N,\pi_0(j))$ for each $j$, which has mean $\pi_0(j)$ and variance $\frac{1}{N}\pi_0(j)(1-\pi_0(j))\leq \frac{1}{4N}$.

    Let $\eps > 0$. Then we have, for all $\pi_0$:
    \begin{align*}
        \norm{\Fhat_N(\pi_0) - F(\pi_0)}_2 &= \norm{\E{F(\pihat_N(A_1,\ldots, A_N))\mid A_1, \ldots A_N \sim \pi_0} - F(\pi_0)}_2 &\\
        &\leq \E{\norm{F(\pihat_N(A_1,\ldots, A_N))- F(\pi_0)}_2 \mid A_1, \ldots A_N \sim \pi_0}&\text{(Jensen's ineq.)}\\
        &\leq \eps \prob{\norm{\pihat_N(A_1,\ldots, A_N) - \pi_0}_2<\delta(\eps)} + 2\prob{\norm{\pihat_N(A_1,\ldots, A_N) - \pi_0}_2\geq\delta(\eps)} &\text{(Unif. cty)}\\
        &\leq \eps + 2\frac{\E{\norm{\pihat_N(A_1,\ldots, A_N) - \pi_0}_2^2}}{\delta(\eps)^2}&\text{(Markov's ineq.)}&\\
        &= \eps + 2\frac{\sum_j \text{Var}(\pihat_N(A_1,\ldots, A_N)_j)}{\delta(\eps)^2}&\\
        &= \eps + \frac{2}{N\delta(\eps)^2}\sum_j\pi_0(j)(1-\pi_0(j))\\
        &\leq \eps + \frac{2}{N\delta(\eps)^2}&
    \end{align*}

    Hence, for all $\eps > 0$
    \begin{align*}
        \sup_{\pi_0}\norm{\Fhat_N(\pi_0) - F(\pi_0)}_2 &\leq \eps + \frac{2}{N\delta(\eps)^2}\\
        &\rightarrow \eps \text{ as } N \rightarrow \infty
    \end{align*}
    and so we must have $\lim_{N\rightarrow\infty}\sup_{\pi_0}\norm{\Fhat_N(\pi_0) - F(\pi_0)}_2  = 0$. This establishes uniform convergence.
\end{proof}

\exantelimit*
\begin{proof}\label{pf:exantelimit}
    First note that the \textit{ex ante} expected utility $\E{u\mid \vpi}$ is continuous as a function of $\vpi$, by Corollary \ref{cor:ex_ante_Lipschitzish}. Also, by the uniform limit theorem, $\F$ must be continuous. Hence, writing $EU(\vpi) = \E{u\mid \vpi}$, we have $EU\circ \F$ and $EU \circ \Fhat_N$ are continuous in $\pi_0$, for all $N$. Since $\Delta(A)$ is compact, continuous functions attain their suprema on $\Delta(A)$. Thus, $\Theta_N$ is non-empty for all $N$, as is $\Theta$.
    
    Now, let $\eps = \limsup_{N\rightarrow\infty}\sup_{\pi\in\Theta_N}d(\pi,\Theta) \geq 0$. We will show $\eps = 0$, from which it immediately follows that $\sup_{\pi\in\Theta_N} d(\pi,\Theta)\rightarrow 0$ as $N \rightarrow \infty$.

    By construction, we can find a sequence $(\pi_{0,N})$ with $\pi_{0,N}\in \Theta_N$, and $d(\pi_{0,N},\Theta) \geq \eps/2$ infinitely often. Thus, there exists a subsequence $(\pi_{0,k_N})$ with $d(\pi_{0,k_N},\Theta) \geq \eps/2$ for all $N$. Moreover, since $\Delta(A)$ is compact, it is sequentially compact, and thus there exists a convergent subsequence of the $(\pi_{0,k_N})$, call it $(\pi_{0,m_N})$. Let $\pi_0 = \lim_{N\rightarrow\infty} \pi_{0,m_N}$. Then, since $d(\pi_{0,m_N},\Theta)\geq\eps/2$ for all $N$, we have $d(\pi_0,\Theta) \geq \eps/2$.

    Using $\Fhat_N \rightarrow \F$ uniformly, we have

    \begin{align*}
        \norm{\Fhat_{m_N}(\pi_{0,m_N})-\F(\pi_0)} &\leq \norm{\Fhat_{m_N}(\pi_{0,m_N}) - \F(\pi_{0,m_N})} + \norm{\F(\pi_{0,m_N}) - \F(\pi_0)}\\
        &\leq \sup_{\pi_0'} \norm{\Fhat_{m_N}(\pi_0') - \F(\pi_0')} + \norm{\F(\pi_{0,m_N}) - \F(\pi_0)}\rightarrow 0 \quad \text{as }N\rightarrow\infty.
    \end{align*}
    Hence, $\Fhat_{m_N}(\pi_{0,m_N}) \rightarrow \F(\pi_0)$ as $N \rightarrow \infty$. Thus, by continuity, $EU(\Fhat_{m_N}(\pi_{0,m_N})) \rightarrow EU(\F(\pi_0))$.

    Meanwhile, for $\pi_0^*\in\Theta$, $EU(\Fhat_{m_N}(\pi_{0,m_N}))\geq EU(\Fhat_{m_N}(\pi_0^*))$, by construction. Moreover, $EU(\Fhat_{m_N}(\pi_0^*)) \rightarrow EU(\F(\pi_0^*))$, by continuity.

    As a result, we must have $EU(\F(\pi_0))\geq EU(\F(\pi_0^*))$. Hence, $\pi_0 \in \Theta$. Since $d(\pi_0,\Theta) \geq \eps/2$, we must have $\epsilon = 0$, as required.
\end{proof}

\CDTGTlimits*
\begin{proof}\label{pf:CDTGTlimits}
    This follows immediately from Propositions \ref{prop:unif_convergence} and \ref{prop:ex_ante_limit} and Theorem \ref{thm:global_sims_result}.
\end{proof}

\subsection{Proofs for Section \ref{sec:local_sims}}

\localsimsresult*
\begin{proof}\label{pf:localsimsresult}
    This is a corollary of later results (Theorem \ref{thm:main} and Proposition \ref{prop:local_sample_model}), but we give a sketch proof here for intuition. Fix $\pi_0$, and consider the decision problem where we take $F_j(\pi_0')$ to be $\E{g_j(A_{1:N_j(\pi_0)}\mid \pi_0)\mid A_{1:N_j}\simiid \pi_0'}$ for all $\pi_0'$. The resulting problem then admits a global simulation model, and the GSGT beliefs in this modified problem at $\pi_0$ are identical to the LSGT beliefs at $\pi_0$. Hence $\pi_0$ is a CDT+GSGT policy in the modified problem if and only if it is a CDT+LSGT policy in the original problem. Now, since the local change in the dependence function is captured in the simulation model, we have that the derivative of the \textit{ex ante} expected utility for the modified problem, at $\pi_0$, is the same as the derivative of the \textit{ex ante} expected utility for the original problem at $\pi_0$.
    
    Thus, the original problem has $\frac{\partial}{\partial \pi_0(a)}\E{u\mid \pi_0} \leq 0$ for all $a$ at $\pi_0$ if and only if this holds also for the modified problem. In turn, applying Theorem \ref{thm:global_sims_result}, this holds if and only if $\pi_0$ is a GSGT+CDT policy in the modified problem. This is then true if and only if it is a LSGT+CDT policy in the original problem.
\end{proof}

\begin{lemma}\label{lemma:g_equiv}
For any $g:A^N \rightarrow \Delta(A)$ and $\pi_0$, we have
\begin{align}\label{eq:g_small_eps}
\begin{split}
    &\E{g(A_1,\ldots,A_N)\mid A_1,\ldots A_N \simiid \pi_0 +\eps(\pi_0'-\pi_0)} - \E{g(A_1,\ldots,A_N)\mid A_1,\ldots A_N \simiid \pi_0}\\
    &= \eps N\left(\frac{1}{N}\sum_{i=1}^N \E{g(A_1,\ldots,A_N \mid A_{-i}\simiid \pi_0, A_i \sim \pi_0'} - \E{g(A_1,\ldots,A_N)\mid A_1,\ldots A_N \simiid \pi_0}\right)+ o(\eps).
\end{split}
\end{align}

As a consequence, we have, at $\pi_0$,
    \begin{equation}\label{eq:F_approx_g}
        \forall \pi_0': F(\pi_0+\eps(\pi_0'-\pi_0)) = \E{g(A_1,\ldots,A_N)\mid A_1,\ldots A_N \simiid \pi_0 +\eps(\pi_0'-\pi_0)}+o(\eps)
    \end{equation}
if and only if $F$ is differentiable at $\pi_0$ and
\begin{align}\label{eq:F_g_a_sum}
    \forall a: F(\pi_0+\eps(a-\pi_0)) &= F(\pi_0)+ \eps N\left(\frac{1}{N}\sum_{i=1}^N \E{g(A_1,\ldots,A_N \mid A_{-i}\simiid \pi_0, A_i = a} - F(\pi_0)\right) + o(\eps).
\end{align}
In other words, equation (\ref{eq:F_approx_g}) holds with respect to $g$ if and only if we may use GGT with $F$ as a dependence function with $\rho = N$ and
\begin{equation}
    \tau(a) = \frac{1}{N}\sum_{i=1}^N \E{g(A_1,\ldots,A_N) \mid A_{-i}\simiid \pi_0, A_i = a}.
\end{equation}
\end{lemma}
\begin{proof}
    To establish equation (\ref{eq:g_small_eps}), we may consider the distribution of $A_i$ as depending on some latent variable $X_i$, where $X_i \sim$ Bernoulli($\eps$), and $A_i\mid X_i \sim X_i\pi'_0 + (1-X_i)\pi_0$. Then, marginalising over the $X_i$, and noting that the probability $j$ of the $X_i$ have $X_i = 1$ is $\prob{\sum_i X_i = j}=O(\eps^j)$ (and that $g$ is bounded):
    \begin{align*}
        &\E{g(A_1,\ldots,A_N)\mid A_1,\ldots A_N \simiid \pi_0 +\eps(\pi_0'-\pi_0)} - \E{g(A_1,\ldots,A_N)\mid A_1,\ldots A_N \simiid \pi_0}\\
        &= \E{\E{g(A_1,\ldots,A_N)\mid A_1,\ldots A_N \simiid \pi_0 +\eps(\pi_0'-\pi_0)\mid X_1,\ldots, X_N}} - \E{g(A_1,\ldots,A_N)\mid A_1,\ldots A_N \simiid \pi_0}\\
    &= (\prob{\forall i, X_i = 0}-1) \E{g(A_1, \ldots A_N)\mid A_1,\ldots, A_N \sim \pi_0}\\
    &\phantom{= -(1-\prob{\forall i, X_i = 0})}+ \sum_{k=1}^{N}\prob{X_k=1 \cap X_{-k} = 0}\E{g(A_1,\ldots, A_{N})\mid A_{-k} \simiid \pi_0, A_k \sim \pi_0'} + o(\eps)\\
    &= -(1-(1-\eps)^N) \E{g(A_1, \ldots A_N)\mid A_1,\ldots, A_N \sim \pi_0}\\
    &\phantom{= -(1-\prob{\forall i, X_i = 0})}+ \eps(1-\eps)^{N-1}\sum_{k=1}^{N}\E{g(A_1,\ldots, A_{N})\mid A_{-k} \simiid \pi_0, A_k \sim \pi_0'} + o(\eps)\\
    &=- \eps N \E{g(A_1, \ldots A_N)\mid A_1,\ldots, A_N \sim \pi_0} + \eps N \frac{1}{N}\sum_{k=1}^{N}\E{g(A_1,\ldots, A_{N})\mid A_{-k} \simiid \pi_0, A_k \sim \pi_0'} + o(\eps)\\
     &= \eps N\left(\frac{1}{N}\sum_{i=1}^N \E{g(A_1,\ldots,A_N \mid A_{-i}\simiid \pi_0, A_i \sim \pi_0'} - \E{g(A_1,\ldots,A_N)\mid A_1,\ldots A_N \simiid \pi_0}\right)+ o(\eps).
    \end{align*}

Now, for the second part, note that $F$ is differentiable at $\pi_0$ and
\begin{align}
    \forall a: F(\pi_0+\eps(a-\pi_0)) &= F(\pi_0)+ \eps N\left(\frac{1}{N}\sum_{i=1}^N \E{g(A_1,\ldots,A_N \mid A_{-i}\simiid \pi_0, A_i = a} - F(\pi_0)\right) + o(\eps) \tag{\ref{eq:F_g_a_sum}}
\end{align}
if and only if 
\begin{align}\label{eq:F_g_pi'}
    \forall \pi_0': F(\pi_0+\eps(\pi_0'-\pi_0)) &= F(\pi_0)+ \eps N\left(\frac{1}{N}\sum_{i=1}^N \E{g(A_1,\ldots,A_N \mid A_{-i}\simiid \pi_0, A_i \sim \pi_0'} - F(\pi_0)\right) + o(\eps).
\end{align}

Meanwhile, equation (\ref{eq:F_approx_g}) implies immediately that $F(\pi_0) =\E{g(A_1,\ldots,A_N)\mid A_1,\ldots A_N \simiid \pi_0}$. Similarly, by considering $\pi_0' = \pi_0$ and noting that the $O(\eps)$ term must then vanish, we also have that equation (\ref{eq:F_g_pi'}) gives $F(\pi_0) =\E{g(A_1,\ldots,A_N)\mid A_1,\ldots A_N \simiid \pi_0}$.

Thus, combining this with equation (\ref{eq:g_small_eps}):

\begin{align*}
    &\forall \pi_0': F(\pi_0+\eps(\pi_0'-\pi_0)) = F(\pi_0)+ \eps N\left(\frac{1}{N}\sum_{i=1}^N \E{g(A_1,\ldots,A_N \mid A_{-i}\simiid \pi_0, A_i \sim \pi_0'} - F(\pi_0)\right) + o(\eps)\tag{\ref{eq:F_g_pi'}}\\
    \iff &\forall \pi_0': F(\pi_0+\eps(\pi_0'-\pi_0)) = F(\pi_0)+ \eps N\Biggl(\frac{1}{N}\sum_{i=1}^N \E{g(A_1,\ldots,A_N \mid A_{-i}\simiid \pi_0, A_i \sim \pi_0'}\\
    &\phantom{\forall \pi_0': F(\pi_0+\eps(\pi_0'-\pi_0)) = F(\pi_0)+ \eps N(\frac{1}{N}\sum_{i=1}^N \E{g(A_1)}}-\E{g(A_1,\ldots,A_N)\mid A_1,\ldots A_N \simiid \pi_0}\Biggr) + o(\eps)\\
    \iff &\forall \pi_0': F(\pi_0+\eps(\pi_0'-\pi_0)) = F(\pi_0)+ \E{g(A_1,\ldots,A_N)\mid A_1,\ldots A_N \simiid \pi_0 +\eps(\pi_0'-\pi_0)}\\
    &\phantom{\forall \pi_0': F(\pi_0+\eps(\pi_0'-\pi_0)) = F(\pi_0)+ \E{g(A_1,\ldots,A_N)\mid A}}- \E{g(A_1,\ldots,A_N)\mid A_1,\ldots A_N \simiid \pi_0} + o(\eps)\\
    \iff &\forall \pi_0': F(\pi_0+\eps(\pi_0'-\pi_0)) = \E{g(A_1,\ldots,A_N)\mid A_1,\ldots A_N \simiid \pi_0 +\eps(\pi_0'-\pi_0)} + o(\eps) \tag{\ref{eq:F_approx_g}}
\end{align*}
concluding our proof.
\end{proof}

\localsampleequiv*
\begin{proof}\label{pf:localsampleequiv}
That differentiability is necessary follows immediately from \Cref{lemma:g_equiv}. We now turn to the proof that differentiability is sufficient. 
By Proposition \ref{prop:eq_diff}, get that  $f:\Delta(A)\rightarrow\Delta(A)$ is such that for some $\pi_0 \in \Delta(A)$, there exists $\rho \geq 0$ and $\tau:\Delta(A)\rightarrow \Delta(A)$ linear, with:
    \begin{equation}\label{eq:rho_tau_characterization}
        f(\pi_0+\eps(\pi_0'-\pi_0)) = f(\pi_0) + \eps\rho(\tau(\pi_0')-f(\pi_0)) + o(\eps)
    \end{equation}

for all $\pi_0' \in \Delta(A)$.

We then wish to find $g$ (symmetric) such that, for all $a$,

\begin{equation*}
    \E{g(A_{1:N})\mid A_{1:N}\simiid \pi_0+\eps(\pi_0'-\pi_0)} = f(\pi_0)+\eps\rho(\tau(\pi_0')-f(\pi_0)) + o(\eps).
\end{equation*}

By Lemma \ref{lemma:g_equiv}, differentiability of $f$, and symmetry of $g$, this is then equivalent to finding $g$ such that
\begin{equation*}
    N \left( \E{g(a,A_2,\ldots A_N)\mid A_2,\ldots A_N \simiid \pi_0}-f(\pi_0) \right) = \rho(\tau(a)-f(\pi_0))
\end{equation*}
i.e., such that
\begin{equation}\label{eq:g_simple_char}
    \E{g(a,A_2,\ldots A_N)\mid A_2,\ldots A_N \simiid \pi_0} = \frac{\rho}{N}\tau(a) +\left(1-\frac{\rho}{N}\right)f(\pi_0).
\end{equation}

Our strategy will then be as follows: We essentially wish to construct $g$ so as to detect $\eps$-deviations from the policy $\pi_0$. This is easier the larger $N$ is, so we will choose $N$ to be sufficiently large. Deviations in the direction $a$, where $\pi_0(a)=0$, are easy to detect (we can do so with probability $\sim \eps N$). For these, we may simply take $g(a,A_{2:N}) = \frac{\rho}{N}\tau(a) + (1-\frac{\rho}{N})f(\pi_0)$, for some $N\geq\rho$, regardless of $A_{2:N}$. For $a$ with $\pi_0(a)\in(0,1)$, we will instead set thresholds, depending on the number of $A_{1:N}$ equal to $a$, such that the difference in probability of hitting the threshold when there is an $\eps$-deviation in the direction $a$ compared to when there is not such a deviation is (approximately) maximised. We shall use \Cref{lemma:central_density} to set these thresholds. We will then set $g$'s value depending on which of these thresholds are crossed, in such a way that \Cref{eq:g_simple_char} is satisfied. If no thresholds are crossed, we simply take $g$ to be $f(\pi_0)$.

For simplicity, assume WLOG that $\pi_0(j) > 0$ for $j = 1, \ldots k$ and $\pi_0(j) = 0$ for $j>k$, where $1 \leq k \leq |A|$, so that $\pi_0 = (\pi_0(1),\ldots,\pi_0(k),0\ldots, 0)$. Similarly, assume WLOG that $f(\pi_0)_j > 0$ for $j = 1, \ldots m$ and $f(\pi_0)_j = 0$ for $j > m$, where $1 \leq m \leq |A|$, so that $f(\pi_0) = (f(\pi_0)_1,\ldots f(\pi_0)_m,0,\ldots,0)$. Note that we must have $f(\pi_0) = \tau(\pi_0) = \sum_a \pi_0(a)\tau(a)$, by considering taking $\pi_0'=\pi_0$ in \Cref{eq:rho_tau_characterization}. Hence, for $j\leq k$, we must have that $\tau(j)_{(m+1):|A|} = 0$.

Let $\pihat(a_1,\ldots a_N) \defeq \frac{1}{N}\sum_{j=1}^N a_i$, interpreting actions $a_i$ as elements of $\Delta(A)\subset \mathbb{R}^A$. I.e., the empirical mean of $a_1,\ldots a_N$.

Then, define $g$ as follows:
\begin{itemize}
    \item If for some $j>k$, we have $\pihat(a_{1:N})_j > 0$, set $g(a_{1:N}) = f(\pi_0) + \frac{\rho}{N}(\tau(j)-f(\pi_0))$, where in the case of multiple possible choices of $j$, we take the smallest.
    \item Otherwise, if $\pihat(a_{1:N})_{(k+1):|A|}= 0$, take $g(a_{1:N})_{(m+1):|A|} = 0$, and $g(a_{1:N})_{1:m} = \tilde{g}(a_{1:N})$, where $\tilde{g}:\{1,\ldots,k\}^N\rightarrow\Delta^{m-1}\subset \mathbb{R}^m$ will be specified later.
\end{itemize}

Now, for $j> k$, we then have $\E{g(j,A_2,\ldots A_N)\mid A_2,\ldots A_N \simiid \pi_0} = \tfrac{\rho}{N}\tau(j) + (1-\tfrac{\rho}{N})f(\pi_0)$, as required.

Meanwhile, for $j \leq k$, $\E{g(j,A_2,\ldots A_N)\mid A_2,\ldots A_N \simiid \pi_0} = \E{(\tilde{g}(j,A_2,\ldots A_N),0,\ldots,0)\mid A_2,\ldots A_N \simiid \pi_0}$. Hence it remains to construct
$\tilde{g}: \{1,\ldots,k\}^N\rightarrow\Delta^{m-1}$ such that for $j= 1\ldots k$,

\begin{equation*}
    \E{\tilde{g}(j,A_2,\ldots A_N)\mid A_2,\ldots A_N \simiid \pi_0} = \tfrac{\rho}{N}\tau(j)_{1:m} + (1-\tfrac{\rho}{N})f(\pi_0)_{1:m},
\end{equation*}
noting that $\tau(j){(m+1):|A|} = f(\pi_0)_{(m+1):|A|} = 0$.

Now, if $k = 1$, we must have $\tau(1) = f(\pi_0)$, and we may just take $\tilde{g}$ constant, equal to $f(\pi_0)$. %
Otherwise, we have $k>1$, and for $1\leq j \leq k$, $\pi_0(j) \in (0,1)$.

Then, for $N$ sufficiently large, define $\tilde{g}$ as follows:

\begin{equation*}
    \tilde{g}(a_1,\ldots,a_N) = f(\pi_0)_{1:m} +\frac{1}{k}\sum_{j = 1}^k I\{\pihat(a_{1:N})_j \geq b_j\} (\bm{s}_j -f(\pi_0)_{1:m}),
\end{equation*}
where $\bm{s}_j \in \Delta^{m-1}\subset \mathbb{R}^m$ is to be determined, and $b_j\defeq (1-\tfrac{1}{N})\pi_0(j)$ is, by \Cref{lemma:central_density}, such that if $X_j \sim \Bin(N-1,\pi_0(j))$, we have $\prob{b_j-\tfrac{1}{N}\leq X_j<b_j}\geq \tfrac{1}{2\sqrt{N}}$. 
Let $\prob{X\geq b_j\mid X\sim \tfrac{1}{N}\Bin(N-1,\pi_0(j)}= p_j$, and $\prob{b_j-\tfrac{1}{N}\leq X< b_j\mid X\sim \tfrac{1}{N}\Bin(N-1,\pi_0(j)}= q_j \geq \frac{1}{2\sqrt{N}}$.

Now, we have
\begin{equation*}
    \E{\tilde{g}(a,A_{2:N})\mid A_2,\ldots A_N \simiid \pi_0} = f(\pi_0)_{1:m} +\frac{1}{k}\sum_{j = 1}^k \prob{\pihat(a,A_{2:N})_j \geq b_j} (\bm{s}_j -f(\pi_0)_{1:}).
\end{equation*}

Meanwhile, for $j\neq a$,
\begin{equation*}
    \prob{\pihat(a,A_{2:N})_j \geq b_j} = \prob{\tfrac{N-1}{N}\pihat(A_{2:N})_j \geq b_j} =p_j
\end{equation*}
and for $j = a$,
\begin{equation*}
    \prob{\pihat(a,A_{2:N})_j \geq b_j} = \prob{\tfrac{1}{N}+\tfrac{N-1}{N}\pihat(A_{2:N})_j \geq b_j} = p_j + q_j.
\end{equation*}

Let $\bm{q}$ be the (column) vector $(q_1,\ldots, q_k)$, and similarly $\bm{p} = (p_1,\ldots, p_k)$. Let $Q = \text{diag}(\bm{q})$, i.e. $Q_{ij} = I\{i=j\}q_j$. Then let $P$ be the $k$ by $k$ matrix $P = \bm{p}\bm{1}^T + Q$, where $\bm{1}= (1,\ldots,1)$, so that $P_{ij} = p_j + I\{i=j\}q_j$. Finally, define ($k$ by $m$) matrices $V$, $T$ and $S$ by, for $1\leq i \leq k$ and $1\leq j \leq m$, $V_{ij} = f(\pi_0)_j$, $T_{ij} = \tau(i)_j$, and $S_{ij} = (\bm{s}_i)_j$. I.e., the $i$th row of $V$ is $f(\pi_0)$, the $i$th row of $T$ is $\tau(i)$, and the $i$th row of $S$ is $\bm{s}_i$.

Then we have that

\begin{align*}
    \E{\tilde{g}(a,A_{2:N})\mid A_2,\ldots A_N \simiid \pi_0} &= f(\pi_0)_{1:m} + \frac{1}{k}\sum_{j = 1}^k (p_j+ I\{a=j\}q_j) (\bm{s}_j -f(\pi_0)_{1:m})\\
    &= f(\pi_0)_{1:m} + \frac{1}{k}\sum_{j = 1}^k P_{aj} (\bm{s}_j -f(\pi_0)_{1:m}).
\end{align*}

Thus,
\begin{equation*}
    \E{\tilde{g}(a,A_{2:N})\mid A_{2:N} \simiid \pi_0}_i - f(\pi_0)_i = \frac{1}{k}\sum_{j = 1}^k P_{aj} ((\bm{s}_j)_i -f(\pi_0)_i)= \frac{1}{k}\sum_{j = 1}^k P_{aj}(S_{ji}-V_{ji})= \frac{1}{k}(P(S-V))_{ai}.
\end{equation*}

It remains to construct $S$, with rows in $\Delta^{m-1}$, such that for $a=1,\ldots, k$, and $i = 1,\ldots, m$, we have $\frac{1}{k}(P(S-V))_{ai} = \frac{\rho}{N}(\tau(a)_i -f(\pi_0)_i) = \frac{\rho}{N}(T_{ai} - V_{ai})$, i.e. such that
\begin{equation}\label{eq:mat_eq}
    P(S-V) = \frac{\rho k}{N}(T-V).
\end{equation}

Now, we will show that $P$ is invertible. This will allow us to set $S = V + \frac{\rho k}{N}P^{-1}(T-V)$ so that \Cref{eq:mat_eq} is satisfied. We will then show that $S\bm{1} = \bm{1}$, and that for $N$ large enough, $S_{ij}\geq 0$ for all $i,j$ so that $\bm{s}_i \in \Delta^{m-1}$ for all $i$, which will establish the desired result.

\textbf{$P$ is invertible:} Recall that $P = \bm{p}\bm{1}^T + Q$, where $Q=\text{diag}(\bm{q})$ is non singular (since $q_j \geq \frac{1}{2\sqrt{N}}$). We may therefore apply the Sherman-Morrison formula, which gives that $P$ is invertible provided that $1 + \bm{1}^TQ^{-1}\bm{p} \neq 0$, and that

\begin{equation}\label{eq:Sherman-Morrison}
    P^{-1} = Q^{-1} + \frac{Q^{-1}\bm{p}\bm{1}^TQ^{-1}}{1 + \bm{1}^TQ^{-1}\bm{p}}.
\end{equation}

Now, we indeed have $1 + \bm{1}^TQ^{-1}\bm{p} = 1 + \sum_{j=1}^k \frac{p_j}{q_j} \geq 1 > 0$,
which establishes that $P$ is invertible.

Set $S = V + \frac{\rho k}{N}P^{-1}(T-V)$.

$S\bm{1} = \bm{1}$:  Since the rows of $V$ and $T$ are probability distributions, $V\bm{1} = T\bm{1} = \bm{1}$. Hence,
\begin{equation*}
    S\bm{1} = V\bm{1} + \frac{\rho k}{N}P^{-1}(T-V)\bm{1} = \bm{1} + \frac{\rho k}{N}P^{-1}(\bm{1}-\bm{1}) = \bm{1},
\end{equation*}
as required.

\textbf{$S\geq 0$ for $N$ large enough:} Note that $V > 0$ and so it suffices to establish that $\frac{1}{N}\norm{P^{-1}(T-V)}_\infty \rightarrow 0$ as $N \rightarrow \infty$. Now, the $i$th row of $P^{-1}(T-V)$ is given by $P^{-1}(\tau(i)-f(\pi_0))$, and so

\begin{align*}
\norm{P^{-1}(T-V)}_{\infty} &= \max_{1\leq i\leq k}
    \norm{P^{-1}(\tau(i)-f(\pi_0))}_{\infty} \\
    &\leq \max_{1\leq i\leq k}\norm{P^{-1}(\tau(i)-f(\pi_0))}_2\\
    &\leq 2\norm{P^{-1}}_{\mathit{op}}\\
    &\leq 2\norm{Q^{-1}}_{\mathit{op}} + 2\frac{\norm{Q^{-1}\bm{p}\bm{1}^TQ^{-1}}_{\mathit{op}}}{1 + \bm{1}^TQ^{-1}\bm{p}}\\
    &= 2\max_i q_i^{-1} + 2\frac{\norm{Q^{-1}\bm{p}\bm{1}^TQ^{-1}}_{\mathit{op}}}{1 + \sum_{j=1}^k \frac{p_j}{q_j}}\\
    &\leq 4\sqrt{N} + 2\frac{\norm{Q^{-1}\bm{p}\bm{1}^TQ^{-1}}_{\mathit{op}}}{1 + \sum_{j=1}^k \frac{p_j}{q_j}}.
\end{align*}
Now, 
\begin{equation*}
Q^{-1}\bm{p}\bm{1}^TQ^{-1} =
    \begin{pmatrix} p_1/q_1 \\ \vdots \\ p_k/q_k \end{pmatrix}\bm{1}^TQ^{-1} = \begin{pmatrix} p_1/q_1 \\ \vdots \\ p_k/q_k \end{pmatrix}(q_1^{-1} \cdots q_k^{-1}).
\end{equation*}

So for any $\bm{v} \in \mathbb{R}^k$ with $\norm{\bm{v}}_2 = 1$, we have
\begin{equation*}
    \bm{v}^TQ^{-1}\bm{p}\bm{1}^TQ^{-1}\bm{v} = \sum_j v_j \tfrac{p_j}{q_j}\sum_i v_i q_i^{-1} \leq \sum_j |v_j| \tfrac{p_j}{q_j}\sum_i |v_i| q_i^{-1}\leq k\max_i q_i^{-1} \sum_j\tfrac{p_j}{q_j} \leq 2|A|\sqrt{N}\sum_j\tfrac{p_j}{q_j}.
\end{equation*}

Hence $\norm{Q^{-1}\bm{p}\bm{1}^TQ^{-1}}_{\mathit{op}} \leq 2|A|\sqrt{N}\sum_j\tfrac{p_j}{q_j}$, and

\begin{equation*}
    \norm{P^{-1}(T-V)}_{\infty} \leq 4\sqrt{N}\left(1 + \frac{|A|\sum_j\tfrac{p_j}{q_j}}{1+\sum_j\frac{p_j}{q_j}}\right) \leq 4(1+|A|)\sqrt{N}
\end{equation*}

Hence, $\frac{1}{N}\norm{P^{-1}(T-V)}_{\infty}\rightarrow 0$ as $N\rightarrow \infty$, and we are done.
\end{proof}

\begin{lemma}\label{lemma:central_density}
    Suppose that $\theta\in (0,1)$, and $X\sim \frac{1}{N}\Bin(N-1,\theta)$, i.e., $XN \sim \Bin(N-1,\theta)$. Then if $N$ is sufficiently large, there exists $b\in \mathbb{R}$ such that 
    \begin{equation}
        \prob{b -\frac{1}{N} \leq X < b} \geq \frac{1}{2\sqrt{N}}
    \end{equation}
    In particular, this holds for $b = (1-\tfrac{1}{N})\theta$.
\end{lemma}
\begin{proof}
    Let $b = \frac{N-1}{N}\theta$.

    Then we have:

    \begin{align*}
        \prob{b -\frac{1}{N} \leq X < b} &= \prob{(N-1)\theta-1 \leq XN < (N-1)\theta}.
    \end{align*}

    Let $r$ be the greatest integer such that $r<(N-1)\theta$. Note that $r\geq 0$ provided $N\geq 2$, and $r$ is the unique integer in $[(N-1)\theta-1,(N-1)\theta)$. Hence,

    \begin{align*}
        \prob{b -\frac{1}{N} \leq X < b} &= \prob{XN=r}\\
        &= \theta^r(1-\theta)^{N-1-r}\binom{N-1}{r}\\
        &\sim \theta^r(1-\theta)^{N-1-r}\frac{(N-1)^{N-1}}{r^r(N-1-r)^{N-1-r}}\sqrt{\frac{N-1}{2\pi(N-1-r)r}} \qquad\qquad\text{(Stirling's approximation)}\\
        &= \theta^r(1-\theta)^{N-1-r}\left(\frac{r}{N-1}\right)^{-r}\left(1-\frac{r}{N-1}\right)^{r-(N-1)}\sqrt{\frac{N-1}{2\pi(N-1-r)r}}
        \end{align*}

    Now, writing $\hat{\theta} \defeq \frac{r}{N-1} \in [\theta-\frac{1}{N-1},\theta)$, the first four factors above are then:

        \begin{equation*}
        \theta^r(1-\theta)^{N-1-r}\hat{\theta}^{-r}(1-\hat{\theta})^{r-(N-1)}
        = \left(\theta/\hat{\theta}\right)^r\left((1-\theta)/(1-\hat{\theta})\right)^{N-1-r}
        = \left(1+\frac{\theta-\hat{\theta}}{\hat{\theta}}\right)^r\left(1-\frac{\theta-\hat{\theta}}{1-\hat{\theta}}\right)^{N-1-r}
    \end{equation*}

Now, taking logarithms, applying the power series for $\log$, and noting that $\theta-\hat{\theta} \in (0,\tfrac{1}{N-1}]$, and $\hat{\theta}\rightarrow \theta \in (0,1)$ as $N\rightarrow \infty$:

\begin{align*}
    \log\left(\left(1+\frac{\theta-\hat{\theta}}{\hat{\theta}}\right)^r\left(1-\frac{\theta-\hat{\theta}}{1-\hat{\theta}}\right)^{N-1-r}\right) &= r\log\left(1+\frac{\theta-\hat{\theta}}{\hat{\theta}}\right) + (N-1-r)\log\left(1-\frac{\theta-\hat{\theta}}{1-\hat{\theta}}\right)\\
    &= r\frac{\theta-\hat{\theta}}{\hat{\theta}} - (N-1-r) \frac{\theta-\hat{\theta}}{1-\hat{\theta}} +O(\tfrac{1}{N})\\
    &= \frac{r(N-1)}{r}(\theta-\hat{\theta}) - \frac{N-1-r}{1-\tfrac{r}{N-1}}(\theta-\hat{\theta})+O(\tfrac{1}{N})\\
    &= (N-1)(\theta-\hat{\theta}) - (N-1)\frac{N-1-r}{N-1-r}(\theta-\hat{\theta})+O(\tfrac{1}{N})\\
    &= O(\tfrac{1}{N})
\end{align*}

Hence, $\theta^r(1-\theta)^{N-1-r}\hat{\theta}^{-r}(1-\hat{\theta})^{r-(N-1)} \rightarrow 1$,
as $N\rightarrow\infty$, by continuity of the exponential function.

Thus, 
\begin{align*}
    \prob{b -\frac{1}{N} \leq X < b} &\sim \sqrt{\frac{N-1}{2\pi(N-1-r)r}}\\
    &\sim \sqrt{\frac{1}{2\pi(N-1)\theta(1-\theta)}}\\
    &\geq \sqrt{\frac{2}{\pi (N-1)}}\\
    &> \frac{1}{2\sqrt{N}}
\end{align*}
Hence, for $N$ sufficiently large, $\prob{b -\frac{1}{N} \leq X < b} \geq \frac{1}{2\sqrt{N}}$.

\end{proof}

\subsection{Proofs for Section \ref{sec:defGGT}}
\eqdiff*
\begin{proof}\label{pf:eqdiff}
    This follows immediately from Lemma \ref{lemma:eq_diff}, noting that there always exists a choice of $\tau$ and $\rho$ satisfying the conditions given in that lemma.
\end{proof}
\begin{lemma}\label{lemma:eq_diff}
    Suppose $F:\Delta(A)\rightarrow\Delta(A)$. Consider $\pi_0\in\Delta(A)$, and let 
    \begin{equation*}
        \Gamma \defeq \max_{a',a: F(a'\mid \pi_0) \neq 0} \frac{-\delta(a' \mid a, \pi_0)}{F(a'\mid \pi_0)}.
    \end{equation*}

    Then $\Gamma$ is well-defined and non-negative whenever $F$ is differentiable at $\pi_0$. Moreover, $\Gamma = 0$ if and only if $\delta(a,\pi_0) = 0$ for all $a$.

    Moreover, given $\rho \in \mathbb{R}$ and $\tau: A \rightarrow \mathbb{R}^{|A|}$, we have:
    \begin{enumerate}[label=(\roman*)]
        \item \begin{equation*}
        F(\pi_0 + \eps(\pi_0' -\pi_0)) = F(\pi_0) + \eps \rho \left(\sum_{a'} \tau(a')\pi_0'(a') - F(\pi_0)\right) + o(\eps)\tag{\ref{eq:rho_tau_char}}
    \end{equation*}
    \item $\rho \geq 0$ \label{cond:rho>=0}
    \item $\tau: A \rightarrow \Delta(A)$ \label{cond:taurange}
    \end{enumerate}
if and only if
    \begin{enumerate}
        \item $F$ is differentiable at $\pi_0$
        \item $\rho \geq \Gamma$
        \item $\tau(a) = F(\pi_0) + \frac{\delta(a,\pi_0)}{\rho}$ whenever $\rho > 0$
        \item $\tau(a) \in \Delta(A)$ when $\rho = 0$.
    \end{enumerate}
\end{lemma}
\begin{proof}

    We will first check the claims regarding $\Gamma$.

    $\Gamma$ is well-defined: For at least one $a'$, $F(a'\mid \pi_0) > 0$, so the maximum is not over the empty set.
    
    $\Gamma \geq 0$: Note that for all $a$, $\sum_{a'} \delta(a'\mid a, \pi_0) = 0$, since otherwise $F(\pi_0+\eps(a-\pi_0))$ would cease to be in $\Delta(A)$ for $\eps$ sufficiently small. Thus, either $\delta(a,\pi_0)=0$ for all $a$, or there exist $a,a'$ with $\delta(a'\mid a,\pi_0) < 0$. In the former case, we immediately obtain $\Gamma = 0$. In the latter case, for $\delta(a'\mid a, \pi_0) < 0$ we cannot have $F(a'\mid \pi_0) = 0$, so we get $\Gamma > 0$.

    From the above, we see that $\Gamma = 0$ if and only if $\delta(a,\pi_0) = 0$ for all $a$.

    Now, $F$ satisfies (\ref{eq:rho_tau_char}) with respect to some $\rho$, $\tau$ at $\pi_0$, if and only if $F$ is differentiable at $\pi_0$ with $\delta(a,\pi_0) \defeq \frac{\partial F}{\partial \pi_0(a)} = \rho(\tau(a)-F(\pi_0))$.

    Then, suppose $F$ satisfies (\ref{eq:rho_tau_char}) with respect to some $\rho$, $\tau$ at $\pi_0$, and that $\rho\geq 0$ and $\tau(a)\in\Delta(A)$ for all $a$. $F$ must then be differentiable at $\pi_0$, giving us condition 1. Condition 4, $\tau(a)\in \Delta(A)$ if $\rho = 0$ is immediate. Then, for $\delta(a,\pi_0) = \rho(\tau(a)-F(\pi_0))$, we need that for $\rho > 0$, $\tau(a) = F(\pi_0) + \frac{\delta(a,\pi_0)}{\rho}$, giving us condition 3. Finally, either $\delta(a,\pi_0) = 0$ for all $a$ or $\rho>0$. If $\delta(a,\pi_0) = 0$ for all $a$, we have $\Gamma= 0$, and so $\rho\geq 0 = \Gamma$ (giving us condition 2). Otherwise, $\rho > 0$, and hence $\tau(a) = F(\pi_0) + \frac{\delta(a,\pi_0)}{\rho} \in \Delta(A)$ for all $a$. Thus, for all $a, a'$:
    \begin{equation*}
        F(a'\mid\pi_0) + \frac{\delta(a'\mid a,\pi_0)}{\rho} \geq 0
    \end{equation*}
    and so, provided $F(a'\mid \pi_0) \neq 0$,
    \begin{equation*}
        \rho \geq -\frac{\delta(a'\mid a,\pi_0)}{F(a'\mid \pi_0)}.
    \end{equation*}
    Since this holds for all $a$ and all $a'$ with $F(a'\mid \pi_0)>0$, we must have $\rho \geq \Gamma$ ($\implies$ condition 2).

    For the converse, suppose that $F$ is differentiable at $\pi_0$, and that $\rho \in \mathbb{R}$ and $\tau: A \rightarrow \mathbb{R}^{|A|}$ satisfy the conditions 2-4. Since $\rho \geq \Gamma \geq 0$, we have \ref{cond:rho>=0}: $\rho \geq 0$. If $\rho = 0$, we have $\tau(a) = F(\pi_0) \in \Delta(A)$ for all $a$ ($\implies$ condition \ref{cond:taurange}).
    
    Otherwise, if $\rho>0$, we get
    \begin{alignat*}{5}
            \sum_{a'}\tau(a' \mid a) &= \sum_{a'}F(a' \mid \pi_0) &&+ \frac{1}{\rho}\sum_{a'}\delta(a' \mid a, \pi_0) &\text{(Conditions 1 and 3)}\\
            &= 1 &&+ \frac{1}{\rho}\frac{\partial}{\partial\pi_0(a)}\sum\limits_{a'}F(a'\mid\pi_0) &\text{\qquad\quad (Def. of $\delta$)}\\
            &= 1 &&+ \frac{1}{\rho}\frac{\partial}{\partial\pi_0(a)}1\\
            &= 1.
    \end{alignat*}

    Now, note that when $F(a'\mid \pi_0) = 0$, we have $\delta(a'\mid a, \pi_0) \geq 0$ for all $a$, and so we have that $\forall a, a'$:

    \begin{equation}\label{eq:rhoFgeq-delta}
        \rho F(a'\mid \pi_0) \geq \Gamma F(a'
        \mid \pi_0) \geq -\delta(a' \mid a, \pi_0).
    \end{equation}
    
    Hence,
    \begin{align*}
        \tau(a'\mid a) &=
        F(a' \mid \pi_0) + \frac{\delta(a' \mid a, \pi_0)}{\rho} &&\\
        &\geq F(a' \mid \pi_0) + \frac{-\rho F(a'\mid\pi_0)}{\rho} &&(\text{Eq. }(\ref{eq:rhoFgeq-delta}))\\
        &= 0, &&
    \end{align*}

    which establishes that $\tau(a)$ defines a probability distribution over $A$, for each $a$, giving us condition \ref{cond:taurange}.

    Finally, it remains to check that $\delta(a,\pi_0) = \rho (\tau(a)-F(\pi_0))$. If $\rho = 0$, we must have $\Gamma = 0$ and hence $\delta(a,\pi_0) = 0$ for all $a$. Thus, we have $\delta(a,\pi_0) = 0 =\rho(\tau(a)-F(\pi_0))$, as required.

    Otherwise, $\rho > 0$ and $\tau(a) = F(\pi_0) + \frac{\delta(a,\pi_0)}{\rho}$, so that $\rho(\tau(a)-F(\pi_0)) = \delta(a,\pi_0)$, as required, concluding our proof.
    
\end{proof}
\subsection{Proofs for Section \ref{sec:main}}
\thmmain*
\begin{proof}\label{pf:thmmain}
First note that $\pi_0$ is CDT+GGT compatible if and only if $\mathbb{E}_{GGT}[u\mid do(a),\pi_0]- \mathbb{E}_{GGT}[u\mid do(\pi_0),\pi_0]\leq 0$ for all $a$. Hence, if $\pi_0$ is CDT+GGT compatible, we obtain immediately from Lemma \ref{lemma:main} that $\frac{\partial}{\partial \pi_0(a)}\E{u\mid\pi_0} \leq 0$.

Suppose then, conversely, that $\frac{\partial}{\partial \pi_0(a)}\E{u\mid\pi_0} \leq 0$.

If $\sum_j\rho_{j}(\pi_0)\E{\#(j)\mid\pi_0} > 0$, the result follows immediately from Lemma \ref{lemma:main}.

Otherwise, if $\sum_j\rho_{j}(\pi_0)\E{\#(j)\mid\pi_0}= 0$, we have that GGT divides credence between states $s$ with
$\E{\#(s)\mid \pi_0}>0$ (Definition \ref{def:GGT}). We must also have $\rho_{i(s)}(\pi_0)\E{\#(s)\mid \pi_0} = 0$ for all $s$. Thus, for all $s$ with positive GGT credence, $\rho_{i(s)}(\pi_0) = 0$, and so $\tau_{i(s)}(a,\pi_0)$ is by definition just $F_{i(s)}(\pi_0)$, and so constant. As a result, $\mathbb{E}_{GGT}[u\mid do(a),\pi_0] - {\mathbb{E}_{GGT}[u\mid do(\pi_0),\pi_0]}=0$, and so we obtain that $\pi_0$ is CDT+GGT compatible.
\end{proof}

\begin{restatable}{lemma}{lemmamain}\label{lemma:main}
Suppose that $(S,S_T,P_0,n,i,A,T,u,\F)$ is a decision problem and that $\F$ is differentiable. Then for any choice of GGT weights $(\rho_j)$,
    \begin{equation*}
    \frac{\partial}{\partial \pi_0(a)}\E{u\mid\pi_0} = \left(\sum_j\rho_{j}(\pi_0)\E{\#(j)\mid\pi_0}\right)\left(\Ex_{GGT}[u\mid do(a),\pi_0]- \Ex_{GGT}[u\mid do(\pi_0),\pi_0]\right)
\end{equation*}
\end{restatable}

\begin{proof}\label{pf:lemmamain}
From Lemma \ref{lemma:ex_ante_deriv}, we have that
\begin{equation}\label{eq:ex_ante_deriv}
    \E{u\mid \vpi + \eps(\vpi'-\vpi)} = \E{u\mid\vpi} + \eps \sum_{s\in S -S_T}\E{\#(s)\mid \vpi}\left(\Ex^{T(\pi_{i(s)}',s)}[u\mid \bm{\pi}] -\Ex^{T(\pi_{i(s)},s)}[u\mid \bm{\pi}]\right) + o(\eps\mid \vpi)
\end{equation}

The rest is essentially chain rule on the simplex. First note that by differentiability on the simplex, we may write $\F(\pi_0 + \eps(\pi_0'-\pi_0)) = \F(\pi_0) + \eps\nabla_{\pi_0'-\pi_0}\F(\pi_0)+\eps H(\eps,\pi_0')$, where $\sup_{\pi_0'}\norm{H(\eps,\pi_0')}_1 \rightarrow 0$ as $\eps \rightarrow 0$. Then by \Cref{cor:ex_ante_Lipschitzish}, we have that, for some constant $\lambda\geq 0$:

\begin{align*}
    |\E{u\mid \F(\pi_0) + \eps\nabla_{\pi_0'-\pi_0}\F(\pi_0)+\eps H(\eps,\pi_0')}- \E{u\mid \F(\pi_0) + \eps\nabla_{\pi_0'-\pi_0}\F(\pi_0)}|
    &\leq \frac{\eps\norm{H(\eps,\pi_0')}_1}{1-\eps\norm{H(\eps,\pi_0')}_1}\lambda = o(\eps\mid\pi_0).
\end{align*}

Consider any set of GGT weights $(\rho_j,\tau_j)$. For some fixed $\pi_0$, writing $\rho_j = \rho_j(\pi_0)$ and $\tau_j(\pi_0') = \tau_j(\pi_0',\pi_0)$, we then have, for each $j$:

\begin{equation*}
    \nabla_{\pi_0'-\pi_0}F_j(\pi_0) = \rho_j(\tau_j(\pi_0')-F_j(\pi_0)).
\end{equation*}

Hence:
\begin{align*}
    \E{u\mid \F(\pi_0+\eps(\pi_0'-\pi_0))} &= \E{u\mid \F(\pi_0) + \eps(\rho_1(\tau_1(\pi_0')-F_1(\pi_0)), \ldots, \rho_n(\tau_n(\pi_0')-F_n(\pi_0))) } + o(\eps\mid \pi_0).
\end{align*}
Write $\bm{v}_j$ for the tuple $(F_1(\pi_0),\ldots F_{j-1}(\pi_0), \tau_j(\pi_0'), F_{j+1}(\pi_0),\ldots, F_n(\pi_0)) \in \Delta(A)^n$. Then, using linearity of the derivative, this is just:
\begin{align*}
    &\E{u\mid \F(\pi_0+\eps(\pi_0'-\pi_0))}\\
    &= \E{u\middle|  \F(\pi_0) + \sum_{j=1}^n \eps\rho_j (\bm{v}_j-\F(\pi_0))} + o(\eps\mid\pi_0)\\
    &= \E{u\mid \F(\pi_0)} + \sum_{j=1}^n(\E{u\mid \F(\pi_0) + \eps\rho_j(\bm{v}_j - \F(\pi_0))}-\E{u\mid \F(\pi_0)}) + o(\eps\mid\pi_0)\\
    &= \E{u\mid \F(\pi_0)} + \eps\sum_{j=1}^n\rho_j\sum_{s\in S-S_T}\E{\#(s)\mid \F(\pi_0)}\left(\Ex^{T((\bm{v}_j)_{i(s)},s)}[u\mid \F(\pi_0)] -\Ex^{T(F_{i(s)}(\pi_0),s)}[u\mid \F(\pi_0)]\right) + o(\eps\mid\pi_0)\\
    &= \E{u\mid \F(\pi_0)} + \eps\sum_{j=1}^n\rho_j\sum_{s\in S-S_T}\E{\#(s)\mid \pi_0}I\{i(s) = j\}\left(\Ex^{T(\tau_{i(s)}(\pi_0'),s)}[u\mid \pi_0] -\Ex^{T(F_{i(s)}(\pi_0),s)}[u\mid \pi_0]\right)+ o(\eps\mid\pi_0)\\
    &= \E{u\mid \pi_0} + \eps\sum_{s\in S-S_T}\rho_{i(s)}\E{\#(s)\mid \pi_0}\left(\Ex^{T(\tau_{i(s)}(\pi_0'),s)}[u\mid \pi_0] -\Ex^{T(F_{i(s)}(\pi_0),s)}[u\mid \pi_0]\right) + o(\eps\mid\pi_0)
\end{align*}

We conclude that $\E{u\mid \pi_0}$ is differentiable with respect to $\pi_0$, with derivative

\begin{align*}
    \frac{\partial}{\partial\pi_0(a)}\E{u\mid\pi_0} &= \lim_{\eps\rightarrow 0}\frac{\E{u\mid \F(\pi_0+\eps(a-\pi_0))}-\E{u\mid\F(\pi_0)}}{\eps}\\
    &= \sum_{s\in S-S_T}\rho_{i(s)}(\pi_0)\E{\#(s)\mid \pi_0}\left(\Ex^{T(\tau_{i(s)}(a,\pi_0),s)}[u\mid \pi_0] -\Ex^{T(F_{i(s)}(\pi_0),s)}[u\mid \pi_0]\right)\\
    &= \left(\sum_j \rho_{j}(\pi_0)\E{\#(j)\mid\pi_0}\right)\sum_{s\in S-S_T}P_{GGT}(s\mid\pi_0)\left(\Ex^{T(\tau_{i(s)}(a,\pi_0),s)}[u\mid \pi_0] -\Ex^{T(F_{i(s)}(\pi_0),s)}[u\mid \pi_0]\right)\\
    &= \left(\sum_j \E{\rho_{j}(\pi_0)\#(j)\mid\pi_0}\right)\left(\Ex_{GGT}[u\mid \Do(a),\pi_0] - \Ex_{GGT}[u\mid \Do(\pi_0),\pi_0]\right)
\end{align*}
where for the last line, we used that $\tau_j(\pi_0,\pi_0) = F_j(\pi_0)$ (Lemma \ref{lemma:tau_pi_0}).

\end{proof}

\begin{notn}
    Given a distribution $P$ over non-terminal states, define the CDT expected utility of the dependant in the current state acting according to $\bm{\pi'}$, and dependants otherwise acting according to $\bm{\pi}$ as:
    \begin{equation*}
        \Ex^{P}[u\mid\Do(\bm{\pi'}),\bm{\pi}] \defeq \E{\Ex^{T(s,\pi_{i(s)})}[u\mid \bm{\pi},s\sim P]} = \sum_s P(s) \sum_{s'} T(s'\mid \pi_{i(s)},s)\Ex^{s'}[u\mid \bm{\pi}].
    \end{equation*}
\end{notn}

\begin{lemma}\label{lemma:ex_ante_deriv}
    The ex ante expected utility is differentiable (in a similar sense to that described in Definition \ref{def:differentiability}) on $\Delta(A)^n$ (i.e., with respect to the policies of the dependants).

    More precisely, let $G(\bm{\pi}) = \E{u\mid \bm{\pi}}$ be the ex ante expected utility. Then for all $\bm{\pi}$ there exists a linear map $DG_{\bm{\pi}}:\Delta(A)\rightarrow \mathbb{R}$ with

    \begin{equation*}
        \lim_{\eps\rightarrow 0} \sup_{\bm{\pi'}\in \Delta(A)^n} \frac{\norm{G(\bm{\pi}+\eps(\bm{\pi'}-\bm{\pi}))-G(\bm{\pi})-\eps DG_{\bm{\pi}}(\bm{\pi'})}}{\eps} = 0.
    \end{equation*}

    In particular,

    \begin{align*}
        DG_{\bm{\pi}}(\bm{\pi'}) &= \sum_{s\in S-S_T} \E{\#(s)\mid \bm{\pi}}\left(\Ex^{T(\pi_{i(s)}',s)}[u\mid \bm{\pi}] -\Ex^{T(\pi_{i(s)},s)}[u\mid \bm{\pi}]\right)\\
        &= \E{\sum_{s\notin S_T} \#(s)\middle| \vpi}\left(\Ex_{GT}[u\mid \Do(\bm{\pi'}),\bm{\pi}] -\Ex_{GT}[u\mid \Do(\bm{\pi}),\bm{\pi}]\right).
    \end{align*}
\end{lemma}
\begin{proof}
    Consider $\E{u\mid (\vpi+\eps(\vpi'-\vpi))}$. Let $\pitwiddle = \pitwiddle(\eps) = \vpi + \eps(\vpi'-\vpi)$.

    Now, let $X_1,X_2,\ldots \simiid$ Bernoulli$(\eps)$. Then, define random variables as follows: Let $S_1\sim P_0$ be the initial state. Let $S_{t+1}\mid S_t,X_t \sim (1-X_t)T(S_t,\pi_{i(S_t)})+ X_tT(S_t,\pi'_{i(S_t)})$. If $S_t$ is a terminal state, take $T(S_t,a)=S_t$ for all $a$ (i.e., make terminal states absorbing). Thus, $S_1,S_2,\ldots$ are distributed as the sequence of states when dependants follow $\pitwiddle$, and $X_t$ is a dummy variable that determines whether the dependant at time $t$ `deviates' and acts according to $\vpi'$ or continues according to $\vpi$.

    Let $L$ be the history length given by the sequence $S_1,S_2,\ldots$ (i.e., the hitting time of the first terminal/absorbing state). Let $D= \sum_{t=1}^{L-1}X_t$ be the total number of deviations.

    Now, let's condition on $D$. We have:

    \begin{align*}
        \E{u\mid \pitwiddle} &= \E{uI\{D=0\}\mid \pitwiddle} + \E{uI\{D=1\}\mid \pitwiddle} + \E{uI\{D\geq 2\}\mid \pitwiddle}.
    \end{align*}

    Now, note that until the first deviation, $S_t$ follows the same distribution regardless of $\vpi'$. Consequently, the distribution of $I\{D=0\}u(S_L)$ is also independent of $\vpi'$, since the distribution of $S_L$ depends on $\vpi'$ only when $X_t =1$ for some $t<L$. Hence, we obtain
    $\E{uI\{D=0\}\mid \pitwiddle} = \E{uI\{D=0\}\mid \vpi}$.

    Next, we argue that the $\E{uI\{D\geq 2\}\mid \pitwiddle}$ term is $O(\eps^2)$. We have $\E{uI\{D\geq 2\}\mid \pitwiddle} \leq \prob{D\geq 2 \mid \pitwiddle}\max |u|$. Let $T_1 = \min\{t:X_t=1\}$ be the time of the first deviation, and $T_2 = \min\{t>T_1:X_t=1\}$ be the time of the second deviation. Then $D\geq 2 \iff T_1,T_2<L$. Thus,

    \begin{equation*}
        \prob{D\geq 2 \mid \pitwiddle} = \prob{T_1<L\mid \pitwiddle}\prob{T_2<L \mid T_1<L,\pitwiddle}.
    \end{equation*}

    Now, note that $(S_t,X_t)_{t=1}^\infty$ forms a Markov chain. Moreover, $T_1,T_2$ and $L$ are stopping times. $T_1+1$ is then also a stopping time. Hence, the Strong Markov Property implies that $(S_{T_1+t},X_{T_1+t})_{t=1}^\infty$ has the same distribution as the original Markov chain with starting state $S_1 \sim T\left(S_{T_1},\vpi'_{i(S_{T_1})}\right)$ (and $X_1$ distributed as normal). For this Markov chain, $T_2$ is then the first deviation. Hence:

    \begin{align*}
        \prob{T_2<L \mid T_1<L,\pitwiddle} &= \sum_{s\in S-S_T} \prob{S_{T_1} = s \mid T_1<L,\pitwiddle}\prob{T_2<L \mid S_{T_1} = s, \pitwiddle}\\
        &= \sum_{s\in S-S_T} \prob{S_{T_1} = s \mid T_1<L,\pitwiddle}\mathbb{P}^{T(s,\pi'_{i(s)})}(T_1<L \mid \pitwiddle)\\
        &\leq \max_s \mathbb{P}^s(T_1<L\mid \pitwiddle).
    \end{align*}

    Meanwhile, since the distribution of $S_{T_1}$ does not depend on $\vpi'$, we have:
    \begin{equation*}
        \mathbb{P}^s(T_1<L\mid \pitwiddle) = \mathbb{P}^s(T_1<L\mid \vpi) = \mathbb{P}^s(D\geq 1\mid \vpi)\leq \Ex^s[D\mid \vpi].
    \end{equation*}
    Applying the law of total expectation, and noting that when $\vpi' = \vpi$, the distribution of $L$ is completely independent of the $(X_t)$:
    \begin{align*}
        \Ex^s[D\mid \vpi] = \Ex^s\left[\Ex^s[D\mid L,\vpi]\mid \vpi\right] = \Ex^s[\eps L\mid \vpi] = \eps \Ex^s[ L\mid \vpi].
    \end{align*}
    And we obtain $\mathbb{P}^s(T_1<L\mid \pitwiddle) \leq \eps \Ex^s[ L\mid \vpi]$. Putting this all together:
    \begin{align*}
    \E{uI\{D\geq 2\}\mid \pitwiddle} &\leq \prob{D\geq 2 \mid \pitwiddle}\max |u|\\
        &\leq \prob{T_1<L\mid \pitwiddle} \max_s \mathbb{P}^s(T_1<L\mid \pitwiddle)\max |u|\\
        &\leq \max_s \mathbb{P}^s(T_1<L\mid \pitwiddle)^2\max |u|\\
        &\leq \eps^2 \max_s\Ex^s[ L\mid \vpi]^2 \max|u|.
    \end{align*}
    Where for all $s$, $\Ex^s[ L\mid \vpi]$ is finite by Lemma \ref{lemma:hist_length_mgf}. This is an $O(\eps^2)$ term not depending on $\vpi'$. We then also have $\E{uI\{D\geq 2\}\mid \vpi} \leq \eps^2 \max_s\Ex^s[ L\mid \vpi]^2 \max|u|$.

    Consequently,
    \begin{equation*}
        \E{u\mid \pitwiddle} - \E{u\mid \vpi} = \E{uI\{D=1\}\mid \pitwiddle} - \E{uI\{D=1\}\mid \vpi} + O(\eps^2\mid \vpi).
    \end{equation*}

    Finally, noting that the distributions of $S_{T_1}$, and $u(S_L)I\{D=0\}$ do not depend on $\vpi'$:

    \begin{align*}
        \E{uI\{D=1\}\mid \pitwiddle} &= \E{\E{uI\{D=1\}\mid S_{T_1},\pitwiddle}\mid \pitwiddle} &&\text{(Law of total expectation)}\\
        &= \sum_{s\in S-S_T} \prob{S_{T_1}= s\mid \pitwiddle}\E{uI\{D=1\}\mid S_{T_1} =s,\pitwiddle}\\
        &= \sum_{s\in S-S_T} \prob{S_{T_1}= s\mid \pitwiddle}\Ex^{T(s,\vpi'_{i(s)})}[uI\{D=0\}\mid \pitwiddle] &&\text{(Strong Markov Property)}\\
        &= \sum_{s\in S-S_T} \prob{S_{T_1}= s\mid \vpi}\Ex^{T(s,\vpi'_{i(s)})}[uI\{D=0\}\mid \vpi]\\
        &= \sum_{s\in S-S_T}\sum_{t=1}^\infty \prob{S_t= s, T_1= t\mid \vpi}\Ex^{T(s,\vpi'_{i(s)})}[uI\{D=0\}\mid \vpi]\\
        &= \sum_{s\in S-S_T}\sum_{t=1}^\infty \eps(1-\eps)^{t-1}\prob{S_t= s\mid \vpi}\Ex^{T(s,\vpi'_{i(s)})}[uI\{D=0\}\mid \vpi]\\
        &= \eps\sum_{s\in S-S_T}\E{\sum_{t=1}^\infty I\{S_t=s\}(1-\eps)^{t-1}\middle|\vpi}\Ex^{T(s,\vpi'_{i(s)})}[uI\{D=0\}\mid \vpi].\\
    \end{align*}

    Now, for any $s$, we have $\Ex^s[uI\{D>0\}\mid \vpi] \leq \prob{T_1<L\mid \vpi}\max|u|$, which we know from previously is at most ${\eps\Ex^s[L\mid \vpi]\max|u|}$. Thus, $\Ex^{T(s,\vpi'_{i(s)})}[uI\{D=0\}\mid \vpi] = \Ex^{T(s,\vpi'_{i(s)})}[u\mid \vpi] + O(\eps)$.

    Meanwhile, for any $s\in S-S_T$, we have $\E{\sum_{t=1}^\infty I\{S_t=s\}(1-\eps)^{t-1}\mid\vpi} \geq \E{\sum_{t=1}^\infty I\{S_t=s\}(1-\eps)^{L-1}\mid \vpi}$, which is in turn equal to $\E{\#(s)(1-\eps)^{L-1}\mid\vpi}$. Now, since $L$ is almost surely finite, we have that $(1-\eps)^{L-1}$ almost surely converges to 1 as $\eps\rightarrow 0$, and so $\#(s)(1-\eps)^{L-1}$ converges almost surely to $\#(s)$. Then, since $\#(s)(1-\eps)^{L-1}\leq \#(s)$, and the latter has finite expectation, we obtain $L^1$ convergence, by the dominated convergence theorem:  $\E{\#(s)(1-\eps)^{L-1}\mid\vpi} \rightarrow \E{\#(s)\mid\vpi}$. Thus, we must also have $\E{\sum_{t=1}^\infty I\{S_t=s\}(1-\eps)^{t-1}\mid\vpi}\rightarrow \E{\#(s)\mid\vpi}$.

    We therefore get
    \begin{equation*}
        \E{uI\{D=1\}\mid \pitwiddle} = o(\eps\mid \vpi) +\eps\sum_{s\in S-S_T}\E{\#(s)\mid\vpi}\Ex^{T(s,\vpi'_{i(s)})}[u\mid \vpi].
    \end{equation*}

    Putting everything together:

    \begin{align*}
        \E{u\mid \pitwiddle} - \E{u\mid \vpi} &= o(\eps\mid \vpi) + \eps\sum_{s\in S-S_T}\E{\#(s)\mid\vpi}\left(\Ex^{T(s,\vpi'_{i(s)})}[u\mid \vpi] - \Ex^{T(s,\vpi_{i(s)})}[u\mid \vpi]\right)\\
        &= o(\eps\mid \vpi) + \eps\sum_{s\in S-S_T}\E{\#(s)\mid\vpi}\left(\Ex^s[u\mid \Do(\pi_{i(s)}'),\bm{\pi}] -\Ex^s[u\mid \Do(\pi_{i(s)}),\bm{\pi}]\right).
    \end{align*}
    Noting that the $O(\eps)$ term is indeed linear in $\vpi'$, we are then done.
\end{proof}

\begin{cor}\label{cor:ex_ante_Lipschitzish}
    For any decision problem, there exists $\lambda \geq 0$ such that $\forall \vpi,\vpi'$ with $\norm{\vpi-\vpi'}_1<1$, we have
    \begin{equation*}
        |\E{u\mid \vpi} - \E{u\mid \vpi'}|\leq \lambda \frac{\norm{\vpi-\vpi'}_1}{1-\norm{\vpi-\vpi'}_1}.
    \end{equation*}
\end{cor}
\begin{proof}
    We will use \Cref{lemma:ex_ante_deriv}. Again write $G(\vpi)\defeq \E{u\mid \vpi}$. First note that by \Cref{lemma:hist_length_mgf}, we then have that for all $\vpi,\vpi'$, $|DG_{\vpi}(\vpi')|$ is bounded, say by $\lambda/2$. Then, applying the mean value theorem to the directional derivative, we obtain that for any $\vpi$,$\vpi'$, we have for some $\theta\in (0,1)$:

    \begin{align*}
        G(\vpi+\eps(\vpi'-\vpi)) -G(\vpi) &= \eps \nabla_{\vpi'-\vpi}G(\vpi+\eps\theta(\vpi'-\vpi))\\
        &=\frac{\eps}{1-\eps\theta}\nabla_{\vpi'-(\vpi+\eps\theta(\vpi'-\vpi))}G(\vpi+\eps\theta(\vpi'-\vpi))\\
        &= \frac{\eps}{1-\eps\theta}DG_{\vpi+\eps\theta(\vpi'-\vpi)}(\vpi').
    \end{align*}

    Hence, for all $\vpi',\vpi\in \Delta(A)^n$, $|G(\vpi+\eps(\vpi'-\vpi)) -G(\vpi)| \leq \frac{\eps}{1-\eps}\lambda/2$.

    Now, fix $\vpi,\vpi'\in \Delta(A)^n$ and assume $\norm{\vpi-\vpi'}_1<1$. For each $i,j$, let $v_{i,j} = \min((\pi_i')_j,(\pi_i)_j)$, and $v_i = (v_{i,1},\ldots,v_{i,|A|})$. Note that $\norm{v_i}_1 = \sum_j \min((\pi_i')_j,(\pi_i)_j) \geq \sum_j (\pi_i)_j - |(\pi_i')_j-(\pi_i)_j| \geq 1 - \norm{\vpi'-\vpi}_\infty.$ Then, define $\pitwiddle_i = v_i/\norm{v_i}_1$. Since $\norm{v_i}_1\geq 1 - \norm{\vpi'-\vpi}_\infty$, we may therefore write $\vpi = (1 - \norm{\vpi'-\vpi}_\infty)\pitwiddle + \norm{\vpi'-\vpi}_\infty\bm{\gamma}$ and $\vpi = (1 - \norm{\vpi'-\vpi}_\infty)\pitwiddle + \norm{\vpi'-\vpi}_\infty\bm{\gamma'}$ for some $\bm{\gamma},\bm{\gamma'} \in \Delta(A)^n$.

    Putting everything together:
    \begin{align*}
        |G(\vpi)-G(\vpi')| &\leq |G(\vpi)-G(\pitwiddle)| + |G(\vpi')-G(\pitwiddle)|\\
        &\leq \frac{\norm{\vpi'-\vpi}_\infty}{1-\norm{\vpi'-\vpi}_\infty}\lambda\\
        & \leq \frac{\norm{\vpi'-\vpi}_1}{1-\norm{\vpi'-\vpi}_1}\lambda.
    \end{align*}
\end{proof}

\begin{lemma}\label{lemma:tau_pi_0}
    Let $\tau_j$ and $\rho_j$ be a GGT transformation function and weight with respect to $F_j$ (i.e. $F_j$, $\rho_j$, $\tau_j$ satisfy equation (\ref{eq:rho_tau_char_j}) for all $\pi_0$). Then we have $\sum_{a}\tau_j(a, \pi_0)\pi_0(a) = F_j(\pi_0)$.
    \end{lemma}
    \begin{proof}

    If $\rho_j = 0$, we have $\sum_{a}\tau_j(a, \pi_0)\pi_0(a)= \sum_{a}F_j(\pi_0)\pi_0(a) = F(\pi_0)$ as required. Otherwise, we have, applying Lemma \ref{lemma:equiv_wts_transf} and using differentiability of $F_j$:

    \begin{align*}
        \sum_{a} \tau_j(a, \pi_0)\pi_0(a) &= \sum_{a}\pi_0(a) F_j(\pi_0) + \sum_{a}\pi_0(a)\frac{\delta_j(a, \pi_0)}{\rho_j(\pi_0)}\\
        &= F_j(\pi_0) + \frac{1}{\rho_j(\pi_0)}\sum_a \pi_0(a)\lim_{\eps\rightarrow 0}\frac{F_j(\pi_0+\eps(a-\pi_0))-F_j(\pi_0)}{\eps}\\
        &= F_j(\pi_0) + \frac{1}{\rho_j(\pi_0)}\lim_{\eps\rightarrow 0}\frac{F_j(\pi_0+\eps(\pi_0-\pi_0))-F_j(\pi_0)}{\eps}\\
        &= F_j(\pi_0).
    \end{align*}
    \end{proof}

\subsection{Proofs for Section \ref{sec:relating_to_sims}}
\localsamplemodel*
\begin{proof}\label{pf:localsamplemodel}
    From Lemma \ref{lemma:g_equiv}, we immediately get that, for all $\pi_0$ and $i$:

    \begin{align*}
    \forall a: F_i(\pi_0+\eps(a-\pi_0)) &= F_i(\pi_0)+ \eps N_i\left(\frac{1}{N_i}\sum_{k=1}^{N_i} \E{g(A_1,\ldots,A_{N_i} \mid A_{-k}\simiid \pi_0, A_k = a} - F(\pi_0)\right) + o(\eps).
    \end{align*}
    
    Then, we have that $F_i$ is differentiable on the simplex for all $i$, and since $\frac{1}{N_i}\sum_k \E{g_i(A_1,\ldots, A_{N_i})\mid A_{-k} \simiid \pi_0, A_k = a}$ is a valid probability distribution, we see that we may choose GGT weights $\rho_i = N_i$ and transformation functions $\tau_i(a, \pi_0) = \sum_{k= 1}^{N_i} \frac{1}{N_i}\E{g_i(A_1,\ldots, A_{N_i})\mid A_{-k} \simiid \pi_0, A_k = a, \pi_0}$. From the definition, the GGT credences and transformation functions are then identical to the LSGT credences and transformation functions.

    In the case of GSGT beliefs, everything is the same except the $g_i$s and $N_i$s don't depend on the policy $\pi_0$. In this case, GSGT beliefs and LSGT beliefs are identical.
\end{proof}

\Nrho*
\begin{proof}\label{pf:Nrho}
    Suppose that $A = \{0,1\}$ so that we may express policies $\pi_0$ as numbers $p\in [0,1]$.
    
    I.e., write $F(p)$ for $F(1\mid p)$, and let
    \begin{equation*}
        F(p) = \frac{16p^4}{1+ 16p^4}.
    \end{equation*}
    Then $F$ is everywhere differentiable with
    \begin{equation*}
        F'(p) = \frac{64p^3}{1+ 16p^4} - \frac{64p^3 \times 16p^4}{(1+ 16p^4)^2} = \frac{64p^3}{(1+16p^4)^2}.
    \end{equation*}
    Hence, recalling that $\delta(1\mid a, p) \defeq \lim_{\eps\rightarrow 0} \frac{F(p+\eps(a-p))}{\eps}$ we have $\delta(1\mid1,p) = (1-p)F'(p)$ and $\delta(1\mid0,p) = -p F'(p)$. Note that   $\delta(0\mid a,p)$ is then just $-\delta(1\mid a,p)$. Then, with $\Gamma$ is defined as in \Cref{lemma:equiv_wts_transf}, i.e. $\Gamma(p) \defeq \max_{a,a':F(a'\mid p)> 0}\frac{-\delta(a'\mid a,p)}{F(a'\mid p)}$, we have, for $p>0$:

    \begin{equation*}
        \Gamma(p) = \max_{a,a'} \frac{-\delta(a'\mid a,p)}{F(a'\mid p)}
        = \max\left(\frac{-\delta(1\mid 0,p)}{F(p)},\frac{-\delta(0\mid 1,p)}{1-F(p)}\right) = \max\left(\frac{pF'(p)}{F(p)},\frac{(1-p)F'(p)}{1-F(p)}\right).
    \end{equation*}

    Now, for $p=\frac{1}{2}$, we have $F(p)=\frac{1}{2}$ and $F'(p)=2$. Thus, we obtain $\Gamma(p) = \max(\frac{2/2}{1/2}, \frac{2/2}{1/2}) = 2$. Then, we may take $\rho(p) = \Gamma(p)$, by \Cref{lemma:equiv_wts_transf}.

    Now, suppose that we could write $F$ at $\tfrac{1}{2}$ in the desired form for some $N$ with $\E{N} \leq \rho = 2$. I.e.,
    \begin{equation*}
        F(\tfrac{1}{2} + \eps(a-\tfrac{1}{2})) = \sum_{m=1}^\infty \prob{N=m} \E{g_m(A_1,\ldots A_m) \mid A_{1:m}\simiid \tfrac{1}{2} + \eps(a-\tfrac{1}{2})} + o(\eps).
    \end{equation*}
    WLOG assume $g$ is invariant under permutations of its arguments. Then, applying \Cref{lemma:g_equiv}, and observing that we must have $F(\tfrac{1}{2}) = \E{g_N(A_1,\ldots, A_N)\mid A_{1:N} \simiid \tfrac{1}{2}}$, we obtain:

    \begin{align*}
        F(\tfrac{1}{2} + \eps(a-\tfrac{1}{2})) &= \sum_{m=1}^\infty \prob{N=m}\eps m\left(\E{g_m(a,A_{2:m}) \mid A_{2:m}\simiid \tfrac{1}{2}} - \E{g_m(A_{1:m}) \mid A_{1:m}\simiid \tfrac{1}{2}}\right) \\
        &\qquad+\sum_{m=0}^\infty \prob{N=m}\E{g_m(A_{1:m}) \mid A_{1:m}\simiid \tfrac{1}{2}} + o(\eps)\\
        &= F(\tfrac{1}{2}) + o(\eps)+ \sum_{m=1}^\infty \prob{N=m}\eps m\left(\E{g_m(a,A_{2:m}) \mid A_{2:m}\simiid \tfrac{1}{2}} - \E{g_m(A_{1:m}) \mid A_{1:m}\simiid \tfrac{1}{2}}\right)
    \end{align*}
    so that
    \begin{align*}
        \delta(1\mid a,p) &= \sum_{m=1}^\infty \prob{N=m} m\left(\E{g_m(a,A_2,\ldots A_m) \mid A_{2:m}\simiid \tfrac{1}{2}} - \E{g_m(A_1,A_2,\ldots A_m) \mid A_{1:m}\simiid \tfrac{1}{2}}\right).
    \end{align*}
    Then, setting $a=1$, we get $\delta(1\mid 1,p) = 1$, and so:
    \begin{align*}
        1 &= \sum_{m=1}^\infty m\prob{N=m} \left(\E{g_m(1,A_2,\ldots A_m) \mid A_{2:m}\simiid \tfrac{1}{2}} - \E{g_m(A_1,A_2,\ldots A_m) \mid A_{1:m}\simiid \tfrac{1}{2}}\right)\\
        &= \sum_{m=1}^\infty \tfrac{1}{2}m\prob{N=m} \left(\E{g_m(1,A_2,\ldots A_m) \mid A_{2:m}\simiid \tfrac{1}{2}} - \E{g_m(0,A_2,\ldots A_m) \mid A_{2:m}\simiid \tfrac{1}{2}}\right)\\
        & \leq \sum_{m=1}^\infty \tfrac{1}{2}m\prob{N=m}= \tfrac{1}{2}\E{N}\leq \rho/2 = 1.
    \end{align*}

    We have equality only when for all $m$ with $\prob{N=m}>0$:
    \begin{equation*}
        \E{g_m(1,A_2,\ldots A_m) \mid A_{2:m}\simiid \tfrac{1}{2}} - \E{g_m(0,A_2,\ldots A_m) \mid A_{2:m}\simiid \tfrac{1}{2}} = 1.
    \end{equation*}
    I.e., for $a= 0, 1$:
    \begin{equation*}
        \E{g_m(a,A_2,\ldots A_m) \mid A_{2:m}\simiid \tfrac{1}{2}} = a.
    \end{equation*}
    But this can only hold if $m = 1$, since, otherwise we would need $\E{g_m(0,1,A_3\ldots A_m) \mid A_{3:m}\simiid \tfrac{1}{2}}$ to be both 0 and 1.
    
    However, if $N$ is 1 with probability 1, $\E{N} = 1$, and so we would then have $1 \leq \tfrac{1}{2}\E{N} = \tfrac{1}{2}$, a contradiction.
\end{proof}

\subsection{Proofs for Section \ref{sec:limitations}}
\GGTfaithful*
\begin{proof}\label{pf:GGTfaithful}
    That GGT is not fanciful follows immediately from the definition.
    Now, note that when $F_j(\pi_0) = \pi_0$, the derivative of $F_j$ is not zero anywhere, and so we must have $\rho_j > 0$. Thus, when $\E{\#(j)\mid \pi_0} > 0$ we have $P_{GGT}(j\mid \pi_0)> 0$, from the definition of GGT. Then, we must have $\tau_j(a\mid \pi_0) = \pi_0 + \frac{1}{\rho_j}(a-\pi_0)$ (and $\rho_j \geq 1$), and so $\tau_j(a\mid\pi_0)-\tau_j(a'\mid \pi_0) = \frac{1}{\rho_j}(a-a')$, and we are done.
\end{proof}
\faithfulimposs*
\begin{proof}\label{pf:faithful_imposs}

    Consider the following decision problem. An agent is given a bottle of wine. The day before, Newcomb's Demon predicted whether the agent would drink the wine with positive probability. If the Demon predicted the agent would do so, it poisoned the wine. The agent prefers drinking unpoisoned wine to not drinking wine (utility 1 vs 0), but strongly prefers not to be poisoned (utility -100). We may formalise this decision problem as follows:
    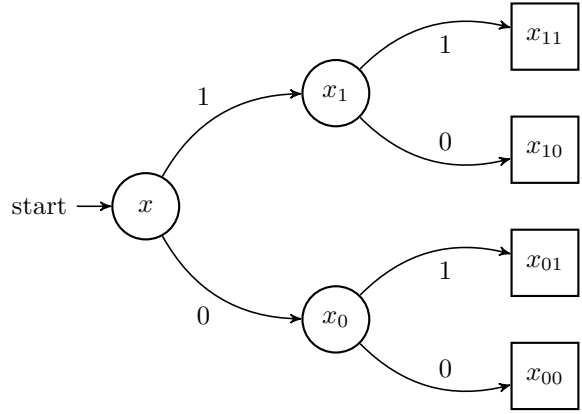
\begin{figure}[h]
    \centering
	\begin{tikzpicture}[->, >=stealth', auto, semithick, node distance= 0.75cm and 2.75cm, on grid,
	terminal/.style ={},
	accepting/.style = {shape = rectangle},
	every state/.style = {fill=none,draw=black,thick,text=black},
	every text node part/.style={align=center}]
	\node[state,initial]    (s) {$x$};
	\node[state]    (s1)[above right = 1.5cm and 2.5cm of s] {$x_1$};
	\node[state]    (s0)[below right = 1.5cm and 2.5cm of s]   {$x_0$};
    \node[state,accepting]    (s11)[above right = of s1]{$x_{11}$};
    \node[state,accepting]    (s10)[below right = of s1]{$x_{10}$};
    \node[state,accepting]    (s01)[above right = of s0]{$x_{01}$};
    \node[state,accepting]    (s00)[below right = of s0]{$x_{00}$};
	\path
	(s) edge[bend left, above left]	node{$1$}	(s1)
	(s) edge[bend right, below left]		node{$0$}	(s0)

	(s1) edge[bend left, below right] node{$1$} (s11)
	(s1) edge[bend right] node{$0$} (s10)
	(s0) edge[bend left, below right] node{$1$} (s01)
	(s0) edge[bend right] node{$0$} (s00);
	\end{tikzpicture}
	\caption{A graph of the decision problem for the proof of \Cref{prop:faithful_imposs}. The nodes show the states and the edges show the transitions (all deterministic).} 	\label{fig:graph}
\end{figure}
\begin{itemize}[nolistsep]
    \item We have non-terminal states $S-S_T = \{x,x_0,x_1\}$, and terminal states $S_T = \{x_{00},x_{01},x_{10},x_{11}\}$. We may partition the non-terminal states as $S_1 = \{x\}$ and $S_2 = \{x_0,x_1\}$.
    \item The initial distribution is deterministic, $P_0(x)=1$.
    \item The index function is given by $i(x) = 2$, $i(x_1) = i(x_0) =1$ (that is, the predictor is at state $x$, and the copy of the agent is at state $x_0$ or $x_1$).
    \item The action set is $A=\{0,1\}$, with 1 corresponding to drinking the wine (or poisoning it, for the predictor), and 0 to not drinking/poisoning it.
    \item The transition function is deterministic with $T(x,a) = x_a$ and $T(x_a,a') = x_{aa'}$.
    \item The dependence functions are, written as a function from $[0,1]$ to $[0,1]$ for convenience, $F_1(p) = p$ (for the copy of the agent) and $F_2(p) = I[p> 0]$ (for the predictor).
    \item The utility function is $u(x_{11}) = -100$ (drinking the poisoned wine), $u(x_{10}) = 0$ (not drinking the poisoned wine), $u(x_{01}) = 1$ (drinking unpoisoned wine) and $u(x_{00}) = 0$ (not drinking unpoisoned wine).
\end{itemize}

    Now, the \textit{ex ante} optimal policy is to never drink the wine $p = 0$, since this is the unique policy that doesn't result in a utility of $-100$ (rather in a utility of $0$). Suppose then that the agent expects to follow policy $p = 0$. If the agent is at state $x$, it will be neutral between all actions, regardless of the transformation function, since its copy will later drink the wine and receive utility 0 whatever the predictor does.

    Meanwhile, since $X$ is non-fanciful and faithful, it assigns positive credence to state $x_0$, and zero to state $x_1$. Then the overall CDT+$X$ utility of drinking the wine compared to not drinking the wine is $P_X(x_0)\left(\Ex^{x_0}_{X}[{u\mid \Do(1), p=0}] - \Ex_{X}^{x_0}[{u\mid \Do(0), p=0}]\right) = P_X(x_0)(\tau(1,0)_1-\tau(0,0)_1) > 0$, since $X$ is faithful. Hence, $p = 0$ is not a CDT+$X$ policy, so we are done.
\end{proof}

\pragmaticimposs*
\begin{proof}\label{pf:pragmaticimposs}
    Let $D_n$ be the following simple decision problem:

    \begin{figure}[h]
    \centering
	\begin{tikzpicture}[->, >=stealth', auto, semithick, node distance= 0.75cm and 2.75cm, on grid,
	terminal/.style ={},
	accepting/.style = {shape = rectangle},
	every state/.style = {fill=none,draw=black,thick,text=black},
	every text node part/.style={align=center}]
	\node[state,initial]    (s) {$s$};
	\node[state]    (s1)[above right = 1cm and 2cm of s] {$s_1$};
	\node[state]    (s0)[below right = 1cm and 2cm of s]   {$s_0$};
    \node[state,accepting]    (s11)[above right =0.75cm and 2cm of s1]{$2$};
    \node[state,accepting]    (s10)[below right =1cm and 2cm of s1]{$1$};
    \node[state,accepting]    (s00)[below right =0.75cm and 2cm of s0]{$0$};
	\path
	(s) edge[bend left, above left]	node{$1$}	(s1)
	(s) edge[bend right, below left]		node{$0$}	(s0)

	(s1) edge[bend left] node{$1$} (s11)
	(s1) edge[bend right, out = -30, in =195] node{$0$} (s10)
	(s0) edge[bend left, out = 30, in =-195] node[below]{$1$} (s10)
	(s0) edge[bend right] node{$0$} (s00);
	\end{tikzpicture}
	\caption{A graph of $D_n$ for the proof of \Cref{prop:pragmaticimposs}. The nodes show the states and the edges show the transitions (all deterministic).} 	\label{fig:graphprag}
\end{figure}
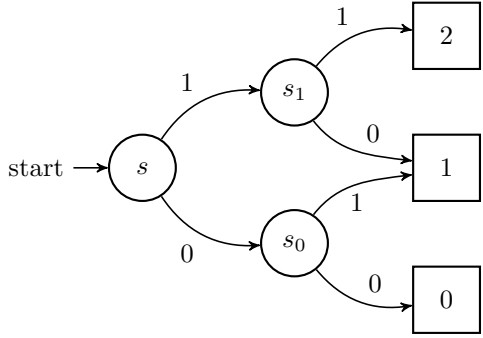
\begin{itemize}
    \item We have non-terminal states $S-S_T = \{s,s_0,s_1\}$, and terminal states $S_T = \{0,1,2\}$. We may partition the non-terminal states as $S_1 = \{s\}$ and $S_2 = \{s_0,s_1\}$.
    \item The initial distribution is deterministic, $P_0(s)=1$.
    \item The index function is 1 everywhere.
    \item The action set is $A=\{0,1\}$.
    \item The transition function is deterministic with $T(s,a) = s_a$ and $T(s_i,a) = i+a$.
    \item The sole dependence function is, written as a function from $[0,1]$ to $[0,1]$ for convenience, $F(p) = \frac{\floor{np}}{n}$. (See \Cref{fig:graph_dep_2}.)
\end{itemize}
    \begin{figure}[h]
\centering
\begin{tikzpicture}
\begin{axis}[
title={$F(p)=\frac{\floor{np}}{n}$},
width = 3cm,
height = 3cm,
 xlabel={$p$},
ylabel={$F(p)$},
xmin=0, xmax=1,
ymin=0, ymax=1,
ytick={1/5,2/5,3/5,4/5,1},
yticklabels ={$\sfrac{1}{5}$,$\sfrac{2}{5}$,$\sfrac{3}{5}$,$\sfrac{4}{5}$,$1$},
xtick={0,1/5,2/5,3/5,4/5,1},
xticklabels ={0,$\tfrac{1}{5}$,$\tfrac{2}{5}$,$\tfrac{3}{5}$,$\tfrac{4}{5}$,$1$},
enlargelimits=false,
scale only axis=true
]
\addplot[color=red,domain=0:1,samples = 1000]{floor(5*x)/5};

\end{axis}
\end{tikzpicture}
\caption{The dependence function for $D_n$.}\label{fig:graph_dep_2}
\end{figure}
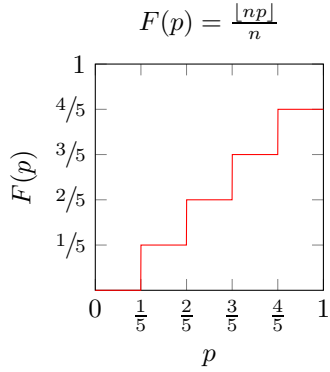
Note that $\norm{F-id}_\infty = \frac{1}{n} \rightarrow 0$ as $n\rightarrow \infty$, as required. We will not specify fully the utility function yet, but instead set $u(0) = 0$, $u(1) = 1$ and $u(2)=u_2$.

Now, the \textit{ex ante} expected utility is given by

\begin{equation*}
    \E{u\mid p} = 2F(p)(1-F(p)) + F(p)^2u_2 = F(p)^2(u_2-2) + 2F(p).
\end{equation*}

Let $X$ be a generalised theory of self-locating beliefs such that the \textit{ex ante} optimal policy in $D_n$ is always a CDT+$X$ policy.

We shall begin by finding, for each $p\in (0,1)$, a range of values for $u_2$ for which $p$ is the \textit{ex ante} optimal policy, and hence for which $p$ is a CDT+$X$ policy. Then, for $p\in (0,1)$, by varying $u_2$ within the range that $p$ is optimal, we will deduce that $\tau(a,p)$ must be constant in $a$. From this, the result will follow.

Note that, for any $u_2$, $F(p)^2(u_2-1)+2F(p)$ has at most one stationary point as a function of $F(p)$. Thus, for $p\in [\tfrac{k}{n},\tfrac{k+1}{n})$, where $k \in \{0,1,\ldots, n-1\}$, to be \textit{ex ante} optimal, it suffices that $\E{u\mid p}\geq \E{u\mid q}$, for $q=\tfrac{k+1}{n}$ and, if $k>0$, also for $q = \tfrac{k-1}{n}$.

Turning first to $k=0$, $p \in [0,\tfrac{1}{n})$ and $F(p) = 0$, we have then that $p$ is \textit{ex ante} optimal when
\begin{equation*}
    \frac{1}{n^2}(u_2-2)+\frac{2}{n} \leq 0\iff 2n + u_2 - 2 \leq 0\iff u_2 \leq -2(n-1).
\end{equation*}

Meanwhile, for $k\in \{1,\ldots, n-1\}$, $p\in [\tfrac{k}{n},\tfrac{k+1}{n})$ is then optimal when

\begin{align*}
    &\frac{k^2}{n^2}(u_2-2) + \frac{2k}{n} \geq \frac{(k-1)^2}{n^2}(u_2-2) + \frac{2(k-1)}{n} \qquad&&\text{and}\qquad \frac{k^2}{n^2}(u_2-2) + \frac{2k}{n} \geq \frac{(k+1)^2}{n^2}(u_2-2) + \frac{2(k+1)}{n}\\
    \iff &\frac{2k-1}{n^2}(u_2-2) +\frac{2}{n} \geq 0 \qquad&&\text{and}\qquad \frac{2k+1}{n^2}(u_2-2) + \frac{2}{n} \leq 0\\
    \iff &(2k-1)u_2 - 2(2k-1) \geq -2n \qquad&&\text{and}\qquad (2k+1)u_2 - 2(2k+1) \leq -2n\\
    \iff &u_2 \geq 2-\frac{2n}{2k+1} \qquad&&\text{and}\qquad u_2 \leq 2- \frac{2n}{2k+1}\\
    \iff &u_2 \in \left[2-\frac{n}{k-\tfrac{1}{2}},2-\frac{n}{k+\tfrac{1}{2}}\right].
\end{align*}

Now, assume that given policy $p \in [\tfrac{k}{n},\tfrac{k+1}{n})$, $X$ assigns overall credence $q(p)$ to the dependant at the other state choosing 1. I.e., $q(p) = P_X(s_1\mid p) + P_X(s,1\mid p)$. Let $\tau(a,p)\in [0,1]$ be the transformation function at $p$. Then the CDT+$X$ utility of taking action 1 rather than 0, given that the agent believes it will follow $p$ is given by:

\begin{align*}
    \Ex_X[u\mid \Do(p'),p] - \Ex_X[u\mid \Do(p),p] &= (\tau(1,p)-\tau(0,p))\left(q(p)(u_2-1)+(1-q(p))\right)\\
    &= (\tau(1,p)-\tau(0,p))\left(q(p)u_2 + 1-2q(p)\right).
\end{align*}
For $p\in (0,1)$, thist must be zero, for all $u_2 \in \left[2-\frac{n}{k-\tfrac{1}{2}},2-\frac{n}{k+\tfrac{1}{2}}\right]$. Now, note that $q(p)u_2+1 -2q(p)$ is zero only when $q(p)\neq 0$ and $u_2= 2 -\tfrac{1}{q(p)}$. Thus, for any $q(p)$, we may choose $u_2$ in the required interval such that $q(p)u_2+1 -2q(p)\neq 0$. We deduce that we must have $\tau(1,p)=\tau(0,p)$. Hence, for any assignment of utilities, all mixed policies must be CDT+$X$ policies.
\end{proof}

\end{document}